\documentclass[twoside,11pt]{article}

\usepackage[preprint]{jmlr2e}

\usepackage{cite}
\usepackage{amsmath,amsfonts}
\usepackage{xcolor}
\usepackage{cleveref}
\usepackage[ruled,vlined]{algorithm2e}
\usepackage[export]{adjustbox}
\usepackage[font=small,format=hang]{caption}
\usepackage[font=small]{subcaption}
\usepackage{comment}
\usepackage{sidecap}
\usepackage{tikz}
\usepackage{float}
\usetikzlibrary{decorations.pathreplacing}
\usetikzlibrary{fadings}

 \newtheorem{assumption}{Assumption}
 \DeclareMathOperator*{\argmax}{arg\,max}
 \DeclareMathOperator*{\argmin}{arg\,min}

\jmlrheading{1}{2021}{1-48}{10/30}{N/A}{meila00a}{Martin Figura, Yixuan Lin, Ji Liu, Vijay Gupta}

\ShortHeadings{Resilient Consensus-based Multi-agent Reinforcement Learning }{Figura, Lin, Liu, and Gupta}
\firstpageno{1}

\begin{document}

\title{Resilient Consensus-based Multi-agent Reinforcement Learning with Function Approximation}

\author{\name Martin Figura\email mfigura@nd.edu \\
       \addr Department of Electrical Engineering\\
       University of Notre Dame\\
       Notre Dame, IN 46556, USA
       \AND
       \name Yixuan Lin \email yixuan.lin.1@stonybrook.edu \\
       \addr Department of Applied Mathematics and Statistics \\ Stony Brook University\\
       Stony Brook, NY 11794, USA
       \AND
       \name Ji Liu \email ji.liu@stonybrook.edu\\
       \addr Department of Electrical and Computer Engineering \\ Stony Brook University\\
       Stony Brook, NY 11794, USA
       \AND
       \name Vijay Gupta\email vgupta2@nd.edu \\
       \addr Department of Electrical Engineering\\
       University of Notre Dame\\
       Notre Dame, IN 46556, USA}

\editor{Tong Zhang, Martin Jaggi, Pradeep Ravikumar, Honglak Lee, Laurent Orseau}

\maketitle

\begin{abstract}
Adversarial attacks during training can strongly influence the performance of multi-agent reinforcement learning algorithms. It is, thus, highly desirable to augment existing algorithms such that the impact of adversarial attacks on cooperative networks is eliminated, or at least bounded. In this work, we consider a fully decentralized network, where each agent receives a local reward and observes the global state and action. We propose a resilient consensus-based actor-critic algorithm, whereby each agent estimates the team-average reward and value function, and communicates the associated parameter vectors to its immediate neighbors. We show that in the presence of Byzantine agents, whose estimation and communication strategies are completely arbitrary, the estimates of the cooperative agents converge to a bounded consensus value with probability one, provided that there are at most $H$ Byzantine agents in the neighborhood of each cooperative agent and the network is $(2H+1)$-robust. Furthermore, we prove that the policy of the cooperative agents converges with probability one to a bounded neighborhood around a local maximizer of their team-average objective function under the assumption that the policies of the adversarial agents asymptotically become stationary.
\end{abstract}


\section{Introduction}
In multi-agent reinforcement learning (MARL), multiple agents interact with each other and a common environment to learn policies that maximize their objective functions. The definition of the objective function for each agent determines whether the agents are cooperative, competitive, or mixed \citep{zhang2019survey}. To this date, most notable success stories of applied MARL have occurred in the competitive setting, particularly in games where each agent wishes to optimize its own objective function~\citep{vinyals2019}. More recently, cooperative MARL, in which agents wish to maximize a team objective function and whose implementation is arguably more challenging due to the required team effort in learning and information sharing, has emerged as an exciting method to solve dynamic programming approximately for teams of agents with aligned objectives. Cooperative MARL has numerous potential applications in cyber-physical systems \citep{xu2021}, robotics \citep{yang2004}, traffic networks \citep{kuyer2008}, swarm systems \citep{huettenrauch2019}, or economics \citep{charpentier2021}.\par
In cooperative MARL, the agents are generally assumed to be independent decision-makers whose control policies are reinforced through their actions in the environment. The cooperative agents use the individual or local reward data obtained from their actions, the state of the environment, and communicated data from the other agents to achieve a team objective that is usually assumed to be optimizing the sum of the individual objective functions of all the agents. There is a long line of literature that assumes that the agents first participate in centralized training and then execute their decentralized policy at test time \citep{foerster2018,lowe2017,rashid2020}. While these methods show great promise, they assume that the agents share their local rewards, which they may wish to keep private in certain applications. Recently, this assumption was removed by establishing methods for completely decentralized learning where both training and execution at test time are decentralized, such that the agents receive only local rewards and communicate local information about the team performance (e.g., local rewards or parameters of the estimated team-average action-value function) to their neighbors according to a directed graph \citep{zhang2018,qu2020,lin2019,suttle2020}. The difference between reinforcement learning paradigms with centralized and decentralized training is depicted schematically in Figure~\ref{fig:dis_vs_dec}.\par
In this paper, we focus on decentralized learning using consensus-based actor-critic (AC) MARL methods \citep{zhang2018,qu2019} that have been shown to scale well with the size of the multi-agent Markov decision processes (MMDP) and to ensure sufficient exploration through the implementation of stochastic policies. AC methods are particularly efficient in problems with large action spaces, where deep Q-learning may become impractical. Moreover, consensus-based AC MARL methods are attractive in the setting where individual agents observe the full state of the environment and actions taken by fellow agents. These algorithms are based on parameter sharing, i.e., the agents locally update parameters of the team-average value function surrogate using their local reward signal and communicate the updated parameters to their immediate neighbors. This idea, which originated in decentralized Q-learning \citep{kar2013}, was adopted in two consensus-based AC MARL algorithms with linear function approximation proposed in \citep{zhang2018} and the consensus-based AC MARL algorithm with nonlinear function approximation proposed in \citep{qu2019}.\par

\begin{figure}[t]
\centering
\includegraphics[width=0.8\linewidth]{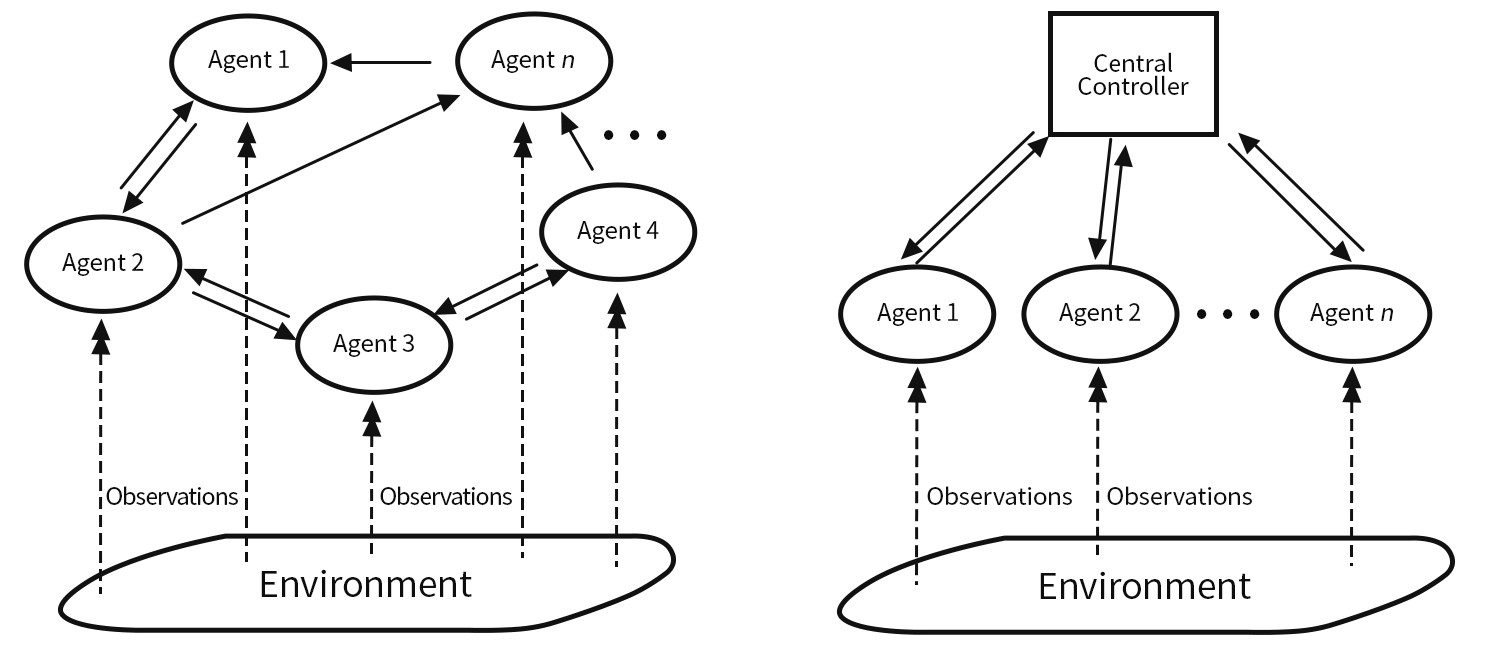}
\caption{MARL training paradigms with centralized v/s decentralized training. Decentralized training (left) is executed locally by individual agents that communicate to their neighbors according to a directed graph, whereas centralized training (right) takes place exclusively at a centralized coordinator. The communication topology on the right is also followed in paradigms such as distributed or federated machine learning, where a centralized coordinator can assign tasks to the agents. Our paper focuses on the decentralized paradigm depicted on the left.}
\label{fig:dis_vs_dec}
\end{figure}

While there is much to be admired about the state-of-the-art consensus-based AC MARL algorithms, there are some question marks about their resilience. In \citep{figura2021}, it was shown that the consensus-based AC MARL algorithm in \citep{zhang2018} is susceptible to a simple adversarial attack. Specifically, a single self-interested agent can mislead all the other (cooperative) agents to learn policies that maximize the objective function of the self-interested agent, even though these policies may be arbitrarily poor for the team objective. The fragility of the algorithm motivates our present work, where the goal is to understand the learning performance in the presence of adversarial agents and provide an answer to the question
\begin{quote}
\textit{Can we design a consensus-based AC MARL algorithm with parametric function approximation for decentralized learning that is provably resilient to adversarial attacks, in the sense that the cooperative agents learn optimal policies in an environment influenced by the adversarial agents?}
\end{quote}
It is important to note that the adversarial agents that we consider impact the other agents both due to the information they communicate to them as well as through implementation of control policies that affect the evolution of the state of the environment. To achieve resilience against adversarial attacks on control policies is difficult in our setting as we do not assume that the agents are aware of the control policies of one another. Our goal is to design a resilient algorithm that leads the cooperative agents to learn near-optimal policies in an environment that is affected by the adversarial agents. This is still a unique challenge because the adversarial agents can model attacks on communication channels that seek to degrade the network performance, which has been considered in cyber-physical system safety \citep{dibaji2019survey}. Resilience to adversarial attacks is a highly desirable property if MARL has to be used in safety-critical systems \citep{koutsoukos2017}.

We would also like to note that in the context of MARL, attacks may occur either during the training process or during execution. In this paper, we assume that the attacks occur during training, and hence the goal of the resilient cooperative agents is to learn policies that at least roughly maximize the objective function of the cooperative agents in the presence of adversaries.
\subsection{Relevant Work}
Any form of multi-agent learning essentially requires aggregation of information from multiple agents. Resilient multi-agent learning algorithms are ones in which the effects of any adversarial data injection are eliminated or at least attenuated. Adversarial agents with the capability to transmit arbitrary information to other agents and implement any control policies as we consider are usually termed as Byzantine agents. Unfortunately, resilient multi-agent learning algorithms that {\em eliminate} the effects of Byzantine attacks entirely quickly run into the curse of dimensionality. For instance, even in the simple context of agents trying to learn the mean of a static vector value that they hold when some agents are Byzantine -- the so-called {\em Byzantine consensus vector} problem, the number of reliable machines must be proportional to the product of the number of Byzantine machines and the dimension of the shared vectors \citep{vaidya2014}.\par
One field that has studied the presence of Byzantine adversarial agents in the context of learning is distributed machine learning (ML) with master-worker networks. In this setting, where the communication topology is similar to the right subfigure in Figure~\ref{fig:dis_vs_dec}, each agent applies stochastic gradient descent (SGD) to update shared function parameters and a central machine aggregates parameters received from multiple worker agents. If some agents can be adversarial, one can extrapolate the algorithms from Byzantine vector consensus; however, they quickly run into a similar curse of dimensionality. This has led to the development of alternative algorithms for resilient distributed ML. Here, we can point to algorithms based on the concept of the geometric median \citep{xie2018,chen2017}. To reduce the computational complexity, the method of median of means was proposed in \citep{tu2021} and the method of entry-wise median was introduced in \citep{alistarh2018}. Another recently proposed approach, the Krum method, is based on medoids \citep{blanchard2017}. It is important to note that all these methods merely bound the deviation of the aggregated parameter vector from the desired one.\par
In consensus-based MARL, the agents communicate data to their neighbors according to a given graph to facilitate estimation of the team-average value function. A popular method to attenuate the effect of adversarial data injected by Byzantine agents is to apply the element-wise trimmed-mean \citep{xie2021,lin2020,wu2021}, where the agents discard $H$ largest and $H$ smallest parameter values received from the neighbors. The hyperparameter $H$ denotes the maximum number of Byzantine agents allowed in the network. According to \citep{xie2021}, a cooperative network that applies this resilient method in decentralized $Q$-learning  attains a robust estimate of the team-average action-value function, and thus the policies of the cooperative agents converge to a neighborhood around the team-optimal policy. Matters become significantly more complicated in large MMDPs, where agents employ function approximation and execute consensus updates in the parameter space. A resilient consensus-based AC MARL algorithm with linear function approximation was proposed in \citep{lin2020}; however, this method involves a centralized coordinator that receives parameter vectors from all agents and provides them with an element-wise trimmed mean of each parameter. This approach is analogous to distributed ML due to the assumed master-worker setting, but is not applicable in decentralized MARL. The approach was later extended to fully decentralized MARL in \citep{wu2021}, where the element-wise trimmed mean is computed on a local level. While the element-wise truncation methods guarantee boundedness of the estimated parameters, the Byzantine agents can still design attacks that manipulate individual parameters within bounded intervals. If properly designed, these attacks lead to the overestimation of selected features, which may compound large errors in the approximated functions. Consequently, the cooperative agents are prevented from learning better policies.

\subsection{Contributions}
In this work, we introduce a novel resilient projection-based consensus method for decentralized AC MARL. Our main contribution is the resilient projection-based consensus AC algorithm with linear function approximation (Algorithm~\ref{alg:2}), whereby the cooperative agents estimate the critic and team-average reward function that are crucial in the approximation of the true policy gradient. The algorithm includes two important steps that jointly faciliate a high degree of resilience in the critic and team-average reward function. In the first step, the received parameters are projected into the feature vectors that are the same for all agents since the agents train linear models using the same basis functions. The projection effectively maps the received parameter vectors into a 1-D subspace and enables the cooperative agents to estimate the estimation errors applied in the local SGD updates by their neighbors. In the second step, the cooperative agents perform resilient aggregation in the 1-D space of the estimated neighbors' estimation errors and apply the aggregated estimation error in the SGD update, which ensures diffusion of local data across the network. One similarity between our resilient projection-based consensus method and the method based on trimmed means as proposed in \citep{wu2021} is in the rule for resilient aggregation; our method includes trimming of extreme values of the estimated  estimation errors at the neighbors and subsequent averaging of the remaining values. The most notable contrast between the methods is in the dimension of the aggregated data as the resilient aggregation using our projection-based method is performed over scalars, whereas the method based on trimmed mean aggregates parameter vectors that are of the same dimension as the feature vectors. Thus, with Algorithm~\ref{alg:2}, the Byzantine agents are essentially forced to propose a scalar value to the cooperative agents despite transmitting a parameter vector. Intuitively, it is significantly more difficult for the Byzantine agents to stage a successful attack that degrades network performance when the cooperative agents apply the projection-based consensus method than when they apply the resilient consensus method based on trimmed mean. We provide a convergence analysis of Algorithm~\ref{alg:2}, which guarantees convergence to the neighborhood of a locally optimal policy under reasonable assumptions on the policies of the Byzantine agents.\par
Our second contribution is the integration of the resilient projection-based consensus method in the AC algorithm with nonlinear function approximation (Algorithm~\ref{alg:3}). We also make a side contribution regarding non-resilient consensus AC algorithms. Specifically, we introduce the projection-based consensus AC algorithm (Algorithm~1), which is a special case of Algorithm~\ref{alg:2} with no trimming applied in the consensus updates. We make assumptions that are analogous to \citep{zhang2018} but different from assumptions made for Algorithm~\ref{alg:2}. We prove that the agents find locally optimal policies under Algorithm~\ref{alg:2} despite applying consensus updates over scalars. These findings are consistent with the convergence results for the consensus-based AC algorithms in \citep{zhang2018}.

\subsection{Paper Organization}
In Section~\ref{sec:background}, we specify the MMDP and the objective functions of the cooperative agents, and elaborate on functions approximated in the consensus-based AC MARL algorithms with a view towards resilient consensus. The three consensus-based  AC MARL algorithms are presented in Section~\ref{sec:algorithms}. We provide a convergence analysis for Algorithm~1 and Algorithm~2 in Section~\ref{sec:convergence} and demonstrate the efficacy of Algorithm~3 in Section~\ref{sec:simulation}. For completeness, some technical results on stochastic approximation can be found in the appendix.

\section{Background}\label{sec:background}
This section formulates the decentralized MARL problem with cooperative and adversarial agents. We begin with the definition of an underlying MMDP, which is analogous to \citep{zhang2018}. The definition assumes that the agents are independent decision-makers that receive private rewards, which depend on the global state which is observable. The state transition probability is also a function of the global state and action, which renders the MMDP setting quite general. We highlight that the cooperative and Byzantine agents have conflicting objectives, which poses a difficulty in the learning process. Since the Byzantine agents are present in the network, we emphasize the importance of resilient approximation on behalf of the cooperative agents.
\subsection{Networked Markov Decision Process}\label{Sec MMDP}
We consider an MMDP given as a tuple  $(\mathcal{S},\{\mathcal{A}^i\}_{i\in\mathcal{N}},\mathcal{P},\{\mathcal{R}^i\}_{i\in\mathcal{N}},\mathcal{G})$, where $\mathcal{N}=\{1,\dots,N\}$ is the set of all agents, $\mathcal{S} $ is a set of (global) states, $\mathcal{P}$ is a set of transitional probabilities, $\gamma\in[0,1)$ is a discount factor, $\mathcal{G}$ represents a set of communication graphs, and $\mathcal{A}^i$ and $\mathcal{R}^i$ are a set of actions and rewards of agent $i$, respectively. The communication graph active at time $t$ is denoted by $\mathcal{G}_{t}$. With a small abuse of notation, we let $\mathcal{G}_t=(\mathcal{N},\mathcal{E}_t)$ so that the set of vertices is also denoted by $\mathcal{N}$, with each vertex $i$ being associated with agent $i$, and a set of directed edges $\mathcal{E}_t\subseteq\mathcal{N}\times\mathcal{N}$. Furthermore, we define sets $\mathcal{N}_{in,t}^i$ and $\mathcal{N}_{out,t}^i$ that include all agents that transmit data to and receive data from agent~$i$ at time~$t$, respectively. The global state is denoted by $s\in\mathcal{S}$. The global action is obtained by stacking the actions of all the agents and is denoted by $a$. We will use $s^\prime$ to denote the global state at the future step. All variables with the superscript $i$ pertain to agent $i$. We let $r^i(s,a):\mathcal{S}\times\mathcal{A}\rightarrow\mathcal{R}^i\subset\mathbb{R}$ denote the local reward of subsystem $i$, $p(s^\prime|s,a):\mathcal{S}\times\mathcal{S}\times\mathcal{A}\rightarrow\mathcal{P}\subset\mathbb{R}$ the joint transitional probability, and $\pi^i(a^i|s):\mathcal{S}\times\mathcal{A}^i\rightarrow(0,1)$ the policy of subsystem~$i$. The global policy is given as $\pi(a|s)=\prod_{i\in\mathcal{N}} \pi^i(a^i|s)$. If needed, we emphasize the dependence of a signal on time by using subscript~$t$, i.e., $r_{t+1}^i(s_t,a_t)$. If the dependence is clear from the context, we drop the subscript to reduce notational clutter. The rewards remain private and each agent generally receives a different reward, i.e., $r^i\neq r^j$ for $i,j\in\mathcal{N},\,i\neq j$. We assume that every agent observes the global state $s$ and action $a$ at each step in training. We define the average individual reward under global policy $\pi(a|s)$ as $r_\pi^i(s)=\sum_a\pi(a|s)r^i(s,a)$, the average individual reward under global policy $\pi(a|s)$ at all states $s\in\mathcal{S}$ as $R_\pi^i=[r_\pi^i(s),s\in\mathcal{S}]^T\in\mathbb{R}^{|\mathcal{S}|}$, and the average individual reward at all state-action pairs $(s,a)$ as $R^i=[r^i(s,a),s\in\mathcal{S},a\in\mathcal{A}]^T\in\mathbb{R}^{|\mathcal{S}|\cdot|\mathcal{A}|}$. The distributions of states and state-action pairs visited by the agents under a fixed policy $\pi(a|s)$ are denoted as $d_\pi(s)$ and $d_\pi^\prime(s,a)$, respectively. 

\subsection{Objective Functions}
In this subsection, we define the learning objective of the agents in the MMDP. We divide the agents into a set of cooperative agents and a set of Byzantine agents, which we denote by $\mathcal{N}^+$ and $\mathcal{N}^-$, respectively. The definition of a Byzantine agent is stipulated as follows.
\begin{definition}\label{def:byzantine}
A Byzantine agent is one that communicates arbitrary and generally distinct information to each of its neighbors in the set $\mathcal{N}_{out,t}^i$ and enacts an arbitrary policy $\pi^i(a^i|s)$. 
\end{definition}
We note that the membership or cardinality of the sets $\mathcal{N}^+$ and $\mathcal{N}^-$ is not known. In other words, we do not know a priori whether an agent is cooperative or Byzantine. We let $\pi^+(a^+|s)=\prod_{i\in\mathcal{N}^+}\pi^i(a^i|s)$ denote the aggregated policy of the cooperative agents, where $a^+$ is the aggregated action of the cooperative agents. Similarly, we define the aggregated policy of the Byzantine agents as  $\pi^-(a^-|s)=\prod_{i\in\mathcal{N}^-}\pi^i(a^i|s)$, where $a^-$ represents the aggregated action of the Byzantine agents. Each cooperative agent~$i$, $i\in\mathcal{N}^+$, is associated with an objective function:
\begin{align*}
J^i(\pi)=J^i(\pi^+,\pi^-)=\mathbb{E}_{\pi,d_\pi}\big[\sum_{t=0}^\infty \gamma^tr_{t+1}^i(s_t,a_t)\big],
\end{align*}
for a discount factor $\gamma\in(0,1).$ The goal of the cooperative agents is to find a policy $\pi^+(a^+|s)$ that maximizes a team-average objective function $J^+(\pi)=\frac{1}{N^+}\sum_{i=\in\mathcal{N^+}}J^i(\pi)$. In other words, the cooperative agents solve the following well-defined optimization problem:
\begin{align}\label{local opt problem}
\pi^+_*=\argmax_{\pi^+} J^+(\pi^+,\pi^-),
\end{align}
in which the cooperative agents optimize their policy $\pi^+$ in an environment that is jointly affected by the policies $\pi^-$ of the Byzantine agents. It is important to note that the cooperative agents search for a policy that is optimal when the MMDP evolution is affected by the Byzantine agents rather than a policy that is optimal when the Byzantine agents are absent from the problem. In addition to the policy optimization at the cooperative agents, the Byzantine agents seek to maximize an arbitrary and potentially unknown objective function $J^-(\pi^-,\pi^+)$ that may, in general, not be aligned with $J^+(\pi^+,\pi^-)$. The policy $\pi^-$ is unknown to the cooperative agents. Furthermore, we assume that the Byzantine agents cannot be identified in the training process. Under this assumption, traditional consensus-based AC MARL algorithms, such as the ones proposed in~\citep{zhang2018,qu2019}, are fragile in the sense that a network of many cooperative agents employing a consensus-based AC MARL algorithm ends up maximizing an arbitrary objective function proposed by a single Byzantine agent, $J^-(\pi^-,\pi^+)$, instead of the desired team-average objective function $J^+(\pi^+,\pi^-)$ \citep{figura2021}. In this work, resilient learning is equivalent to solving (at least approximately) the optimization problem in \eqref{local opt problem} despite the Byzantine agents' best efforts to convince the cooperative agents that the true team-average objective function is not $J^+(\pi^+,\pi^-)$. We propose two resilient consensus-based AC MARL algorithms that ensure a resilient approximation of the gradient of the objective function, $\nabla_{\pi^+}J^+(\pi^+,\pi^-)$, which leads to the policy improvement of the cooperative agents despite the presence of the Byzantine agents. The multi-agent policy gradient $\nabla_{\pi^+}J^+(\pi^+,\pi^-)$ is discussed in detail in the next subsection.

\subsection{Multi-agent Policy Gradient}
Since AC algorithms are gradient-based optimization methods, we first establish a general framework to evaluate the gradient of the objective function, $\nabla_{\pi^+}J^+(\pi^+,\pi^-)$. We recall the well-known policy gradient theorem \citep{sutton2018book}, according to which the gradient of the objective function $J^+(\pi^+,\pi^-)$ is given as follows
\begin{align*}
\nabla_{\pi} J^+(\pi^+,\pi^-)=\mathbb{E}_{\pi,d_\pi}\big[Q_{\pi^+}(s,a)\nabla_\pi\log\pi(a|s)\big],
\end{align*}
where $Q_{\pi^+}(s,a)=\mathbb{E}_\pi\big[\big(\frac{1}{N^+}\sum_{i\in\mathcal{N}^+}\sum_{t=0}^\infty \gamma^t r_{t+1}^i(s_t,a_t)\big)\big|s_0=s,a_0=a\big]$. Using the fact that $\log\pi(a|s)=\log\big(\prod_{i\in\mathcal{N}}\pi^i(a^i|s)\big)=\sum_{i\in\mathcal{N}}\log\pi^i(a^i|s)$, the policy gradient can be expressed as a sum of gradients with respect to the local policies, i.e., $\nabla_{\pi^+} J^+(\pi^+,\pi^-)=\sum_{i\in\mathcal{N}^+}\nabla_{\pi^i} J^+(\pi^+,\pi^-)$, where
\begin{align*}
\nabla_{\pi^i} J^+(\pi^+,\pi^-)=\mathbb{E}_{\pi,d_\pi}\big[Q_{\pi^+}(s,a)\nabla_{\pi^i}\log\pi^i(a^i|s)\big].
\end{align*}
It is important to note that methods that use the policy gradient in the presented form are known to suffer from high variance since the magnitude of the gradient steps depends on the sampled rewards. Baseline policy gradient methods take a step further by employing a baseline that reduces the variance of the policy gradients. We define a critic $V_{\pi^+}(s)=\mathbb{E}_{\pi}\big[Q_{\pi^+}(s,a)\big]$, which serves as the baseline. Then, the baseline policy gradient reads as
\begin{align*}
\nabla_{\pi^i} J^+(\pi^+,\pi^-)=\mathbb{E}_{\pi,d_\pi}\big[\big(Q_{\pi^+}(s,a)-V_{\pi^+}(s)\big)\nabla_{\pi^i}\log\pi^i(a^i|s)\big].
\end{align*}
In AC methods, the agents employ TD learning, which is a bootstrapping method, whereby the agents sample the advantage function $Q_{\pi^+}(s,a)-V_{\pi^+}(s)$ online, and thus accelerate the learning process. In this work, we consider the TD(0) method, where the advantage function is sampled as follows
\begin{align}
Q_{\pi^+}(s,a)-V_{\pi^+}(s)\sim\frac{1}{N^+} \sum_{i\in\mathcal{N}^+}r^i(s,a)+\gamma V_{\pi^+}(s^\prime)-V_{\pi^+}(s),
\end{align}
where $\frac{1}{N^+} \sum_{i\in\mathcal{N}^+}r^i(s,a)+\gamma V_{\pi^+}(s^\prime)-V_{\pi^+}(s)$ is the team-average TD error. The distributed AC policy gradient consists of two components: the team-average advantage function $Q_{\pi^+}(s,a)-V_{\pi^+}(s)$ and the term $\nabla_{\pi^i}\log\pi^i(a_t^i|s_t)$. Whereas the latter can be evaluated locally by each agent, the team-average advantage function cannot be sampled directly in decentralized networks because the agents neither observe the team-average rewards nor have access to the centralized critic. However, as shown by \citep{zhang2018}, there exists a solution method that applies approximation of the critic and team-average reward function and communication between agents that enables the agents to sample the team-average advantage function $Q_{\pi^+}(s,a)-V_{\pi^+}(s)$. We explore the method, consensus-based AC, in the next subsection.

\subsection{Consensus-based Actor-critic Method}
In this subsection, we present the consensus-based AC method. This method is particularly useful in the setting of the MMDP with large state and action spaces that we introduced in Section~\ref{Sec MMDP}. The method exploits the idea that the team-average advantage function $Q_{\pi^+}(s,a)-V_{\pi^+}(s)$ can be estimated by every agent, even though it cannot be directly sampled. We establish an estimator of the team-average objective function by taking inspiration from Algorithm~2 in \citep{zhang2018}, which employs approximations of the team-average reward function and critic. We make the following assumption about these approximations.
\begin{assumption}\label{as:linear_approx}
For each cooperative agent $i$, the critic $V_\pi(s)$ and team-average reward function $\bar{r}(s,a)=\frac{1}{N^+}\sum_{i\in\mathcal{N}^+}r^i(s,a)$ are approximated by linear models, i.e., $V (s;v^i) = \phi(s)^Tv^i$ and $\bar{r}(s,a;\lambda^i)=f(s,a)^T\lambda^i$, where $\phi(s) = [\phi_1(s),\dots,\phi_L(s)]\in\mathbb{R}^L$ and $f(s,a)=[f_1(s,a),\dots,f_M(s,a)]\in\mathbb{R}^M$ are the features associated with $s$ and $(s,a)$, respectively. The feature vectors $\phi(s)$ and $f(s,a)$ are uniformly bounded for any $s\in\mathcal{S}$, $a\in\mathcal{A}$. Furthermore, we let the feature matrix $\Phi\in\mathbb{R}^{|\mathcal{S}|\times L}$ have $[\phi_l(s), s\in\mathcal{S}]^T$ as its $l$-th column for any $l\in [L]$, and the feature matrix $F\in\mathbb{R}^{|\mathcal{S}|\cdot|\mathcal{A}|\times M}$ have $[f_m(s,a), s\in\mathcal{S},a\in\mathcal{A}]^T$ as its $m$-th column for any $m\in[M]$. Both $\Phi$ and $F$ have full column rank.
\end{assumption}
The goal of the cooperative network in the consensus-based AC MARL algorithm is to solve a constrained distributed optimization problem:
\begin{align}
v_\pi&=\argmin\limits_{v^i}\mathbb{E}_{d_\pi}\bigg\{\mathbb{E}_{\pi,p}\bigg(\frac{1}{N^+}\sum\limits_{i\in\mathcal{N}^+}r^i(s,a)+\gamma V(s^\prime;v^i)-V(s;v^i)\bigg)^2\bigg\}\quad \text{s.t.}\quad v^i=v^j\label{opt problem v},\\
\lambda_\pi&=\argmin\limits_{\lambda^i}\mathbb{E}_{d_\pi}\bigg\{\mathbb{E}_{\pi}\bigg(\frac{1}{N^+}\sum\limits_{i\in\mathcal{N}^+}r^i(s,a)-\bar{r}(s,a;\lambda)\bigg)^2\bigg\}\quad \text{s.t.}\quad \lambda^i=\lambda^j\label{opt problem lambda},
\end{align}
where the equality constraint applies to $i,j\in\mathcal{N}^+$. Letting $\delta_{v,t}^i=r_{t+1}^i(s_t,a_t)+\gamma V(s_{t+1};v_t^i)-V(s_t;v_t^i)$ and $\delta_{\lambda,t}^i=r_{t+1}^i(s_t,a_t)-\bar{r}_{t+1}(s_t,a_t;\lambda_t^i)$ denote the local estimation errors, the distributed optimization problem can be solved as follows \citep{zhang2018}:
\begin{align}
\tilde{v}_t^i&=v_t^i+\alpha_{v,t}\cdot\delta_{v,t}^i\cdot\nabla_vV(s_t;v_t^i), &v_{t+1}^i=\sum_{j\in\mathcal{N}} c_t(i,j)\tilde v_t^j\label{v_update2}, \\
\tilde\lambda_{t}^i&=\lambda_t^i+\alpha_{\lambda,t}\cdot\delta_{\lambda,t}^i\cdot\nabla_{\lambda}\bar{r}_{t+1}(s_t,a_t;\lambda_t^i),&\lambda_{t+1}^i=\sum_{j\in\mathcal{N}}c_t(i,j)\tilde\lambda_t^j\label{lambda_update2}.
\end{align}
The critic and team-average reward function parameters are updated locally based on the most recent local reward and observation of the global state and action before being transmitted to immediate neighbors on the communication graph.\par
The success of the distributed stochastic approximation in \eqref{v_update2} and \eqref{lambda_update2} hinges on the assumption that all agents are reliable, i.e., $\mathcal{N}^-=\emptyset$. It is easy to see that in the presence of Byzantine agents, some cooperative agents may include the parameter values transmitted by the Byzantine agents in the consensus updates. Even if a single agent deviates from the proposed updates, the distributed stochastic approximation can yield arbitrarily poor results, e.g., a single adversary can drive the cooperative network to maximize only the objective function $J^-(\pi^-,\pi^+)$ \citep{figura2021}. To provide resilience against adversarial attacks in consensus-based MARL, the method of trimmed means was proposed in \citep{wu2021}. In this method, assuming that there are at most $H$ Byzantine agents in the network, the cooperative agents rank the values received from their neighbors for each parameter and do not consider the $H$ largest and $H$ smallest values in the consensus update. While the method is easy to implement, it may lead to loss of performance as illustrated by the following simple example.
\begin{example}\label{ex:1}
Suppose that under a policy $\pi(a|s)$, the transition to states $s=0$ and $s=1$ occurs with probability $p$ and $1-p$, respectively. There are three cooperative agents, $\mathcal{N}^+=\{1,2,3\}$, that receive a private reward $r^i(s)=i-4s$ and wish to estimate the team-average reward $\bar{r}(s)=\frac{1}{N^+}\sum_{i\in\mathcal{N}^+}r^i(s)=2-4s$ in the presence of a Byzantine agent. We consider a linear model $\bar{r}(s;\lambda^i)=\begin{bmatrix} 1 & s \end{bmatrix}\lambda^i$. The Byzantine agent transmits constant values $\tilde\lambda^i=\begin{bmatrix}5 & -10\end{bmatrix}^T$ to the fully connected network of agents. The cooperative agents apply updates as in \eqref{lambda_update2}, except that the consensus updates include trimming of the maximum and minimum values. The simulation results in Figure~\ref{fig:ex1} show that the method of trimmed means overestimates the team-average reward $\bar{r}(s)$, i.e., $\bar{r}(s=0;\lambda^i)>\max_{i\in\mathcal{N}^+}\{r^i(s=0)\}$. This simple example illustrates that the correctness of the linear approximation under the method of trimmed means is already lost under this simple model of a Byzantine agent.
\end{example}
\begin{figure}[t]
\begin{subfigure}{.49\textwidth}
  \centering
\includegraphics[width=1\linewidth]{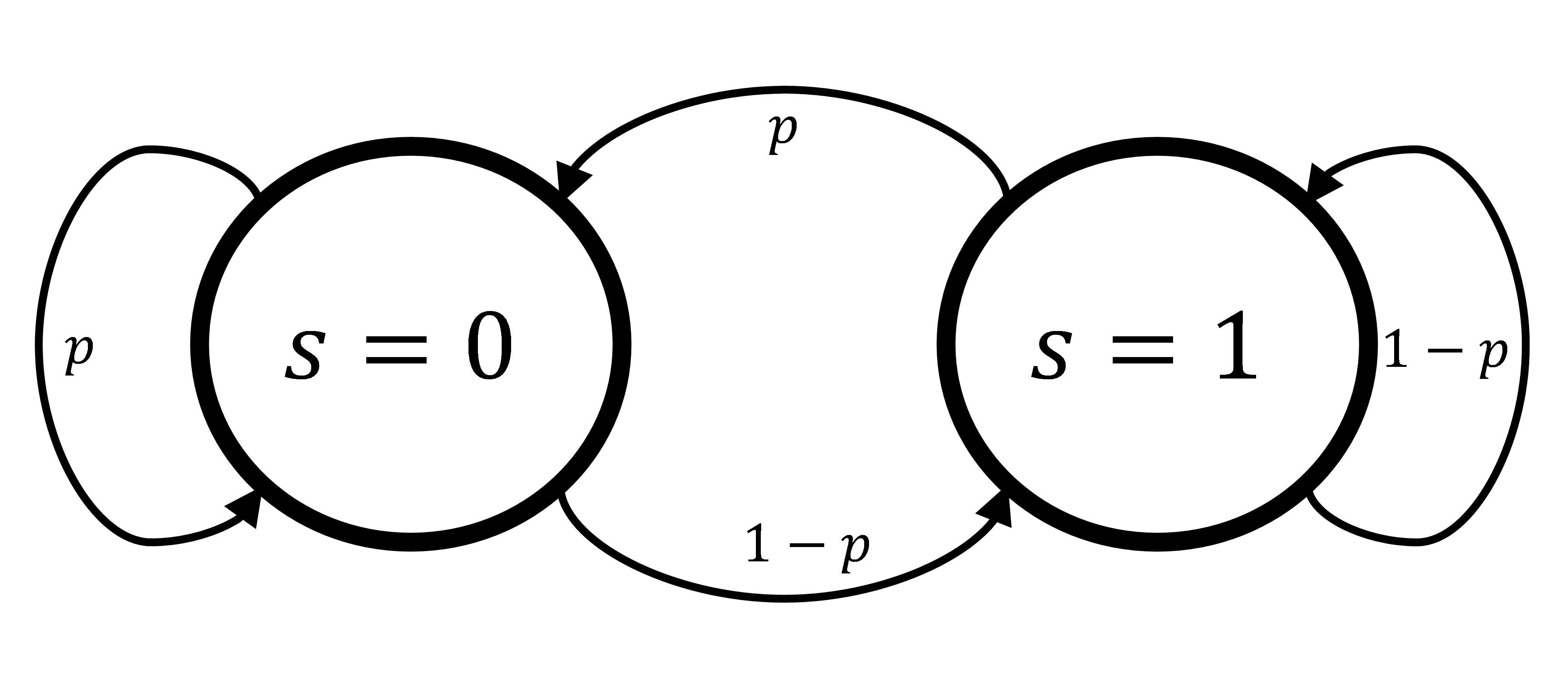}
\end{subfigure}
\begin{subfigure}{.49\textwidth}
  \centering
\includegraphics[width=1\linewidth]{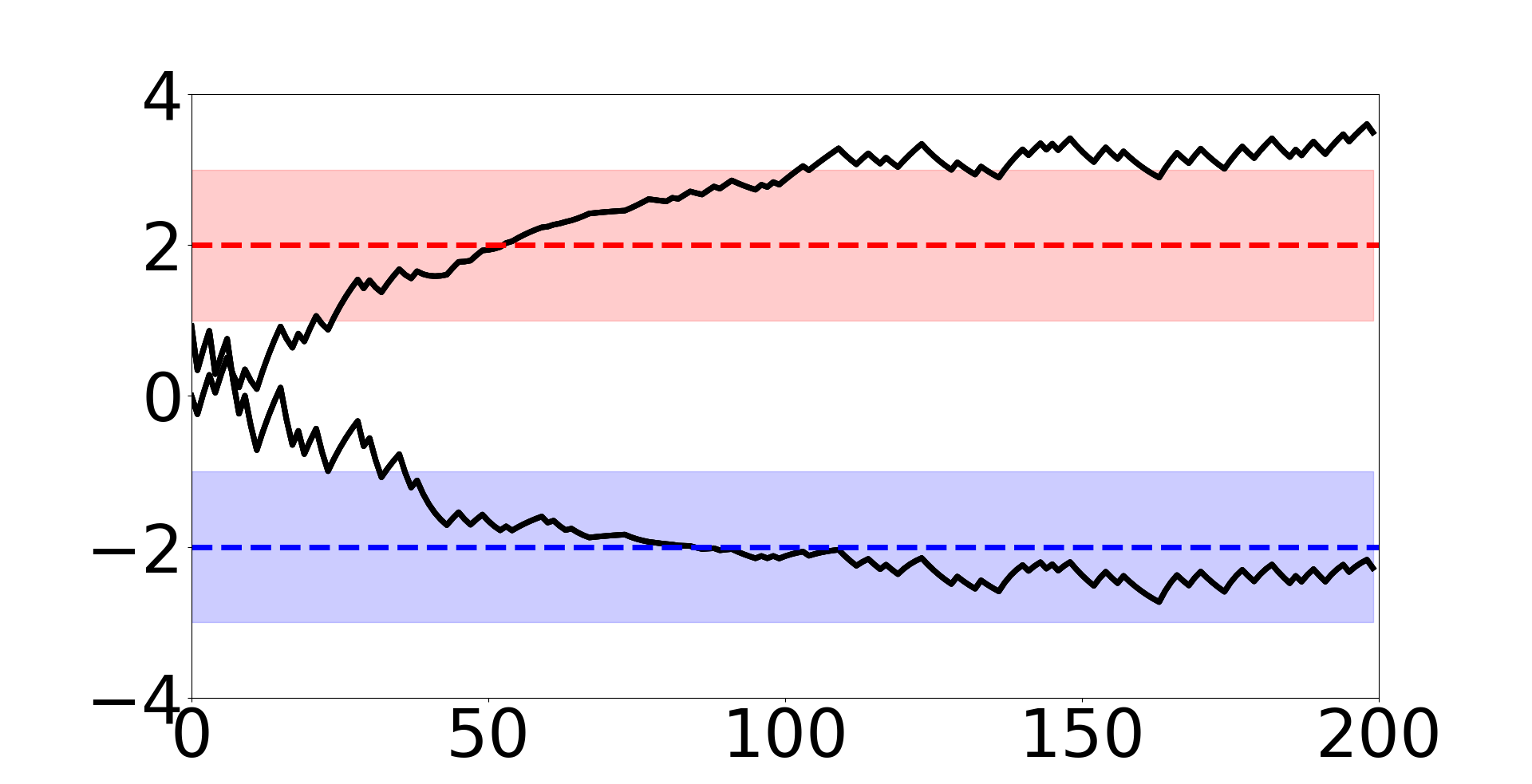}
\end{subfigure}
\caption{Team-average reward function estimation in Example~\ref{ex:1} for $p=0.5$. The estimated values $\bar{r}(s;\lambda^i)$ for $s\in\{0,1\}$ are shown in black color. The true team-average reward for $s\in\{0,1\}$ is depicted by the red and the blue dashed line, respectively. The shaded regions correspond to the convex hull of $r^i(s)$ for $i\in\mathcal{N}^+$. We selected the step size $\alpha_{\lambda,t}=0.05$ in the local updates.}
\label{fig:ex1}
\end{figure}
Motivated by Example~\ref{ex:1}, our goal is to design a resilient consensus method that is suitable for function approximation, and thus can be integrated in MARL algorithms. In the next subsection, we introduce three consensus-based AC MARL algorithms that employ a novel projection-based consensus method that minimizes the impact of Byzantine attacks.

\section{Algorithms}\label{sec:algorithms}
In this section, we present our resilient consensus-based AC MARL algorithms. In the first part, we introduce a novel projection-based consensus method that allows agents to conservatively perform consensus updates. The method is incorporated into the projection-based consensus AC MARL algorithm with linear approximation (Algorithm~1). In the second part, we present the resilient projection-based consensus AC MARL algorithm with linear approximation (Algorithm~2) that further includes a truncation step that provides resilience in the team-average estimation. In the last part, we introduce the deep resilient projection-based consensus AC MARL algorithm (Algorithm~3) that employs deep neural networks to approximate the actor, critic, and team-average reward function.
 
\subsection{Projection-based Consensus Actor-critic with Linear Approximation}
In Section~2.2 and 2.3, we defined the team-average objective function $J^+(\pi^+,\pi^-)$, presented the distributed AC policy gradient, and stressed the need to evaluate the team-average TD error. In Section~2.4, we defined the stochastic and consensus updates of the vanilla consensus-based AC MARL algorithm with function approximation, whereby the agents simultaneously optimize the critic and team-average reward function in order to facilitate the evaluation of the team-average TD error without the knowledge of other agents' rewards. In this subsection, we present a novel projection-based consensus method that also facilitates the evaluation of the team-average TD error by each agent. As its name suggests, the method includes a scalar projection of the received neighbors' parameters. In the following paragraphs, we provide a detailed explanation of the method.\par
Throughout the paper, we use shorthand $f_t$ and $\phi_t$ to denote the feature vectors $f(s_t,a_t)$ and $\phi(s_t)$, respectively. By Assumption~\ref{as:linear_approx}, we can rewrite the critic and team-average reward stochastic updates from \eqref{v_update2} and \eqref{lambda_update2} as follows:
\begin{align}
\tilde{v}_t^i&=v_t^i+\alpha_{v,t}\big(r_{t+1}^i+\gamma \phi_{t+1}^Tv_t^i-\phi_t^Tv_t^i\big)\phi_t,\qquad
\tilde\lambda_{t}^i=\lambda_t^i+\alpha_{\lambda,t}\big(r_{t+1}^i-f_t^T\lambda_t^i\big)f_t,\label{stochastic updates}
\end{align}
where the feature vectors $\phi_t$ and $f_t$ are the same for all agents   because they consider the same basis functions in the critic and team-average reward approximation, respectively. It is easy to see that a single update is performed in the subspace spanned by the feature vectors and its magnitude and direction in this subspace is governed by the step size and estimation error. We recall that the agents receive neither the rewards nor the estimation errors of their neighbors. However, they can exploit the knowledge of the common feature vectors $\phi_t$ and $f_t$ to estimate the estimation error of their neighbors using scalar projection as follows
\begin{align}
r_{t+1}^j+\gamma \phi_{t+1}^Tv_t^j-\phi_t^Tv_t^j\approx\frac{\phi_t^T(\tilde{v}_t^j-v_t^i)}{\alpha_{v,t}\Vert\phi_t\Vert^2},\qquad r_{t+1}^j-f_t^T\lambda_t^j\approx\frac{f^T(\tilde\lambda_t^j-\lambda_t^i)}{\alpha_{\lambda,t}\Vert f_t\Vert^2}.\label{estimated_errors}
\end{align}
We note that the signals on the RHS are available to agent $i$, which makes the projection-based approximation feasible. In Lemma~\ref{lem:error estimation}, we prove that the approximation technique in~\eqref{estimated_errors} becomes exact once the agents reach consensus on the parameter values.
\begin{lemma}\label{lem:error estimation}
Suppose that agent $i$ reaches consensus on the critic and team-average reward function parameters with its neighbors, i.e., $x_t^i=x_t^j$ for $x\in\{v,\lambda\}$ and all $j\in\mathcal{N}_{in,t}^i$. Then, the agent can exactly evaluate estimation errors $r_{t+1}^j-f_t^T\lambda_t^j$ and $r_{t+1}^j+\gamma \phi_{t+1}^Tv_t^j-\phi_t^Tv_t^j$.
\end{lemma}
\begin{proof}
We manipulate the neighbor updates in \eqref{stochastic updates} and apply scalar projection into their respective feature vectors $\phi_t$ and $f_t$. We obtain $r_{t+1}^j+\gamma \phi_{t+1}^Tv_t^j-\phi_t^Tv_t^j=\frac{\phi_t^T(\tilde{v}_t^j-v_t^i)}{\alpha_{v,t}\Vert\phi_t\Vert^2}$ and $r_{t+1}^j-f_t^T\lambda_t^j=\frac{f_t^T(\tilde\lambda_t^j-\lambda_t^i)}{\alpha_{\lambda,t}\Vert f_t\Vert^2}$, where we used the fact that $x_t^i=x_t^j$ for $x\in\{v,\lambda\}$. Therefore, agent $i$ evaluates the estimation errors exactly.
\end{proof}
Lemma~\ref{lem:error estimation} validates our intuition that the estimation of the neighbors' estimation errors proposed in \eqref{estimated_errors} is meaningful. As we noted before, an important part of decentralized cooperative learning is the diffusion of local information across the network, which is facilitated by communicating data to neighbors and aggregating data at each agent. In contrast to the vector aggregation via consensus updates in \eqref{v_update2} and \eqref{lambda_update2}, the aggregation in the projection-based consensus method is done via consensus updates over the estimated neighbors' estimation errors that take scalar values. The method is incorporated in the parameter updates of the projection-based consensus AC algorithm (Algorithm~\ref{alg:1}) that have the following structure
\begin{align}
\tilde{v}_t^i&=v_t^i+\alpha_{v,t}\cdot\delta_{v,t}^i\cdot\nabla_vV(s_t;v_t^i),
&\tilde\lambda_{t}^i&=\lambda_t^i+\alpha_{v,t}\cdot\delta_{\lambda,t}^i\cdot\nabla_{\lambda}\bar{r}_{t+1}(\lambda_t^i), \nonumber\\
\epsilon_{v,t}^{ij}&=\frac{\phi_t^T(\tilde{v}_t^j-v_t^i)}{\alpha_{v,t}\Vert\phi_t\Vert^2},
&\epsilon_{\lambda,t}^{ij}&=\frac{f^T(\tilde\lambda_t^j-\lambda_t^i)}{\alpha_{v,t}\Vert f_t\Vert^2},\nonumber\\
 \epsilon_{v,t}^i&=\sum_{j\in\mathcal{N}_{in,t}^i} c_t(i,j)\cdot\epsilon_{v,t}^{ij},
&\epsilon_{\lambda,t}^i&=\sum_{j\in\mathcal{N}_{in,t}^i} c_t(i,j)\cdot\epsilon_{\lambda,t}^{ij},\nonumber\\
v_{t+1}^i&=v_t^i+\alpha_{v,t}\cdot\epsilon_{v,t}^i\cdot\nabla_vV(s_t;v_t^i),
&\lambda_{t+1}^i&=\lambda_t^i+\alpha_{v,t}\cdot\epsilon_{\lambda,t}^i\cdot\nabla_{\lambda}\bar{r}_{t+1}(\lambda_t^i).\label{alg1_updates}
\end{align}

\begin{figure}[t]
\begin{subfigure}{.49\textwidth}
  \centering
\includegraphics[width=1\linewidth]{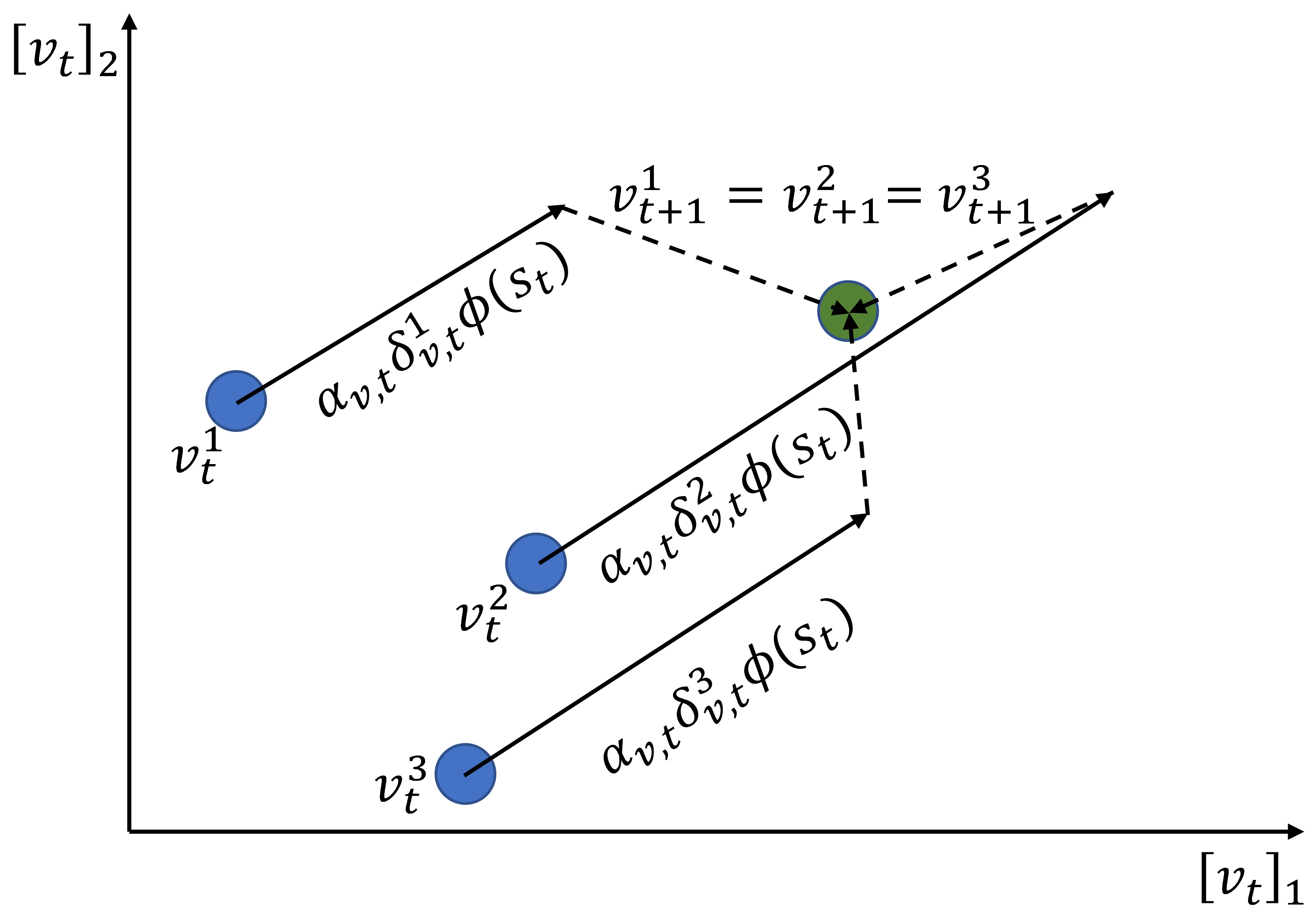}
  \caption{Consensus-based AC updates}
  \label{fig:proj_free_consensus}
\end{subfigure}
\begin{subfigure}{.49\textwidth}
  \centering
\includegraphics[width=1\linewidth]{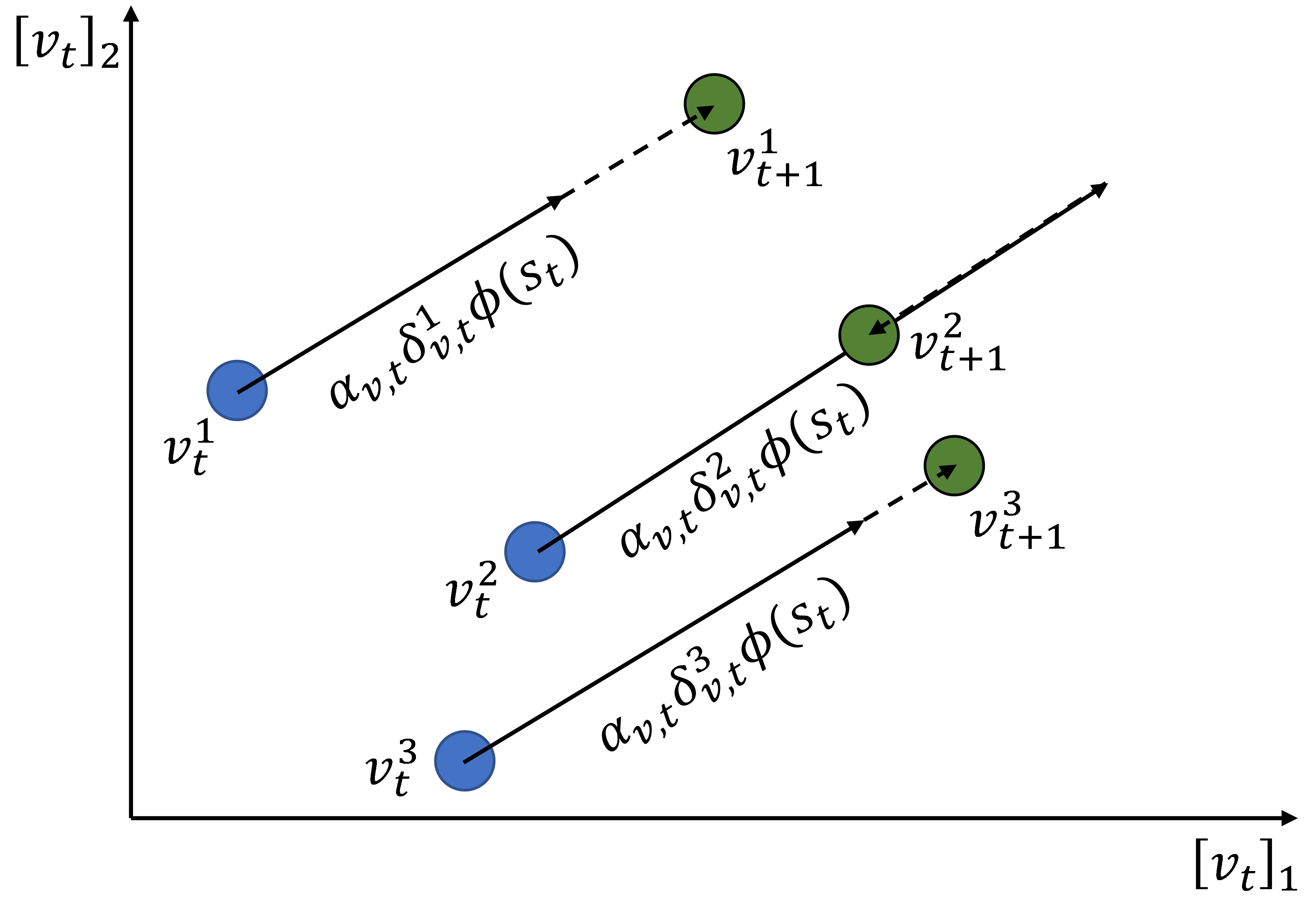}
  \caption{Projection-based consensus AC updates}
  \label{fig:proj_based_consensus}
\end{subfigure}
\caption{Critic updates of the consensus-based AC and the projection-based consensus AC algorithm in a 2-D parameter space. In this example, we consider that the agents communicate on a complete graph and compute a simple average of the received critic parameter values. The stochastic updates are depicted by solid lines while the consensus updates are represented by dashed lines. The updated parameter values are shown in green color.}
\label{fig:consensus_comparison}
\end{figure}
\noindent
The agents perform a stochastic update using their local reward $r_{t+1}^i$ and exchange the updated parameters $\tilde{v}_t^i$ and $\tilde\lambda_t^i$ over the communication graph. Then, they estimate the average estimation errors $\epsilon_{v,t}^i$ and $\epsilon_{\lambda,t}^i$ through the projection-based consensus update and apply the average estimation errors in the parameter updates that yield new values $v_{t+1}^i$ and $\lambda_{t+1}^i$. In Fig.~\ref{fig:consensus_comparison}, we provide a comparison between the consensus-based AC updates that are given in \eqref{v_update2} and \eqref{lambda_update2}, and the projection-based consensus AC updates that are given in \eqref{alg1_updates}. We note that the projection-based consensus AC algorithm performs more conservative updates compared to the consensus-based AC algorithm. The pseudo-code for the projection-based consensus AC algorithm is presented in Algorithm~\ref{alg:1}. In the following lines, we make some standard assumptions that apply to the algorithm.

\begin{assumption}\label{as:policy_coop}
The policy of cooperative agent $i$, $\pi^i(a^i|s;\theta^i)$, is approximated by a generalized linear model. It is stochastic, i.e., $\pi^i(a^i|s;\theta^i)>0$ for any $i\in\mathcal{N}^+$, $\theta^i\in\Theta^i$, $s\in\mathcal{S}$, $a^i\in\mathcal{A}^i$, and continuously differentiable in $\theta^i$. We define $p_\pi(s^\prime|s)=\sum_{a\in\mathcal{A}}p(s^\prime|s,a)\pi(a|s;\theta)$ and let $P_\pi=\begin{bmatrix} p_\pi(s^\prime|s),s^\prime\in\mathcal{S},s\in\mathcal{S}\end{bmatrix}\in\mathbb{R}^{|\mathcal{S}|\times|\mathcal{S}|}$ denote the state transition matrix of the Markov chain $\{s_t\}_{t\geq 0}$ induced by policy $\pi(a|s;\theta)$. The Markov chain $\{s_t\}_{t\geq 0}$ is irreducible and aperiodic under any $\pi(a|s;\theta)$.
\end{assumption}

\begin{assumption}\label{as:reward_bound}
The reward $r^i(s,a)$ is uniformly bounded for any $i\in\mathcal{N}^+$.
\end{assumption}

\begin{algorithm}[t]
	\SetAlgoLined
	 \textbf{Initialize} $s_0,\{\alpha_{v,t}\}_{t\geq 0},,\{\alpha_{\lambda,t}\}_{t\geq 0},\{\alpha_{\theta,t}\}_{t\geq 0}, t\leftarrow 0, \theta_0^i,\lambda^i_0,\tilde\lambda^i_0,v_0^i,\tilde{v}^i_0$, $\forall i\in\mathcal{N}$\\
	 \textbf{Take} action $a_0\sim\pi(a_0|s_0;\theta_0)$\;
	 \textbf{Repeat until convergence} \\
 	 \For{$i\in\mathcal{N}$}{
  		\textbf{Observe} state $s_{t+1}$, action $a_t$, and reward $r_{t+1}^i$\\
  		\textbf{Update actor}\\
  		$\delta_t^i\leftarrow \bar{r}_{t+1}(\lambda_t^i)+\gamma V(s_{t+1};v_t^i)-V(s_t;v_t^i)$\\
  		$\psi_t^i\leftarrow\nabla_{\theta^i}\log\pi^i(a_t^i|s_t^i;\theta_t^i)$\\
  		$\theta_{t+1}^i\leftarrow\theta_t^i+\alpha_{\theta,t}\delta_t^i\psi_t^i$\\
  		\textbf{Update critic and team reward function}\\
  		$\tilde v_{t}^i\leftarrow v_t^i+\alpha_{v,t}\big(r_{t+1}^i+\gamma V(s_{t+1};v_t^i)-V(s_t;v_t^i)\big)\nabla_v V(s_t;v_t^i)$\\
  		$\tilde\lambda_{t}^i\leftarrow \lambda_t^i+\alpha_{\lambda,t}\big(r_{t+1}^i-\bar{r}_{t+1}(\lambda_t^i)\big)\nabla_{\lambda}\bar{r}_{t+1}(\lambda_t^i)$\\
  		\textbf{Send} $\tilde\lambda_t^i,\tilde{v}_t^i$ to $j\in\mathcal{N}_{out,t}^i$\\
  		}
  		\For{$i\in\mathcal{N}$}{
  		\textbf{Receive} $\tilde\lambda_t^j,\tilde{v}_t^j$ from $j\in\mathcal{N}_{in,t}^i$\\
  		\textbf{Projection-based consensus step}\\
  		$\epsilon_{v,t}^{ij}\leftarrow\frac{\phi_t^T(\tilde{v}_t^j-v_t^i)}{\alpha_{v,t}\Vert\phi_t\Vert^2}$, $\epsilon_{\lambda,t}^{ij}\leftarrow\frac{f^T(\tilde\lambda_t^j-\lambda_t^i)}{\alpha_{\lambda,t}\Vert f_t\Vert^2}$ for $j\in\mathcal{N}_{in,t}^i$\\
  		$ \epsilon_{v,t}^i\leftarrow\sum_{j\in\mathcal{N}_{in,t}^i} c_t(i,j)\cdot\epsilon_{v,t}^{ij}$, 
  		$\epsilon_{\lambda,t}^i\leftarrow\sum_{j\in\mathcal{N}_{in,t}^i} c_t(i,j)\cdot\epsilon_{\lambda,t}^{ij}$\\
		$v_{t+1}^i\leftarrow v_t^i+\alpha_{v,t}\cdot\epsilon_{v,t}^i\cdot\nabla_vV(s_t;v_t^i)$\\
		$\lambda_{t+1}^i\leftarrow\lambda_t^i+\alpha_{\lambda,t}\cdot\epsilon_{\lambda,t}^i\cdot\nabla_{\lambda}\bar{r}_{t+1}(s_t,a_t;\lambda_t^i)$\\
  		\textbf{Take} action $a_{t+1}^i\sim\pi^i(a_{t+1}^i|s_{t+1};\theta_{t+1}^i)$\;
  		 }
  	\textbf{Update} iteration counter $t\leftarrow t+1$
 \caption{Projection-based consensus actor-critic with linear approximation}
 	\label{alg:1}
\end{algorithm}

\begin{assumption}\label{as:step_size}
The step sizes $\alpha_{x,t}$, $x\in\{v,\lambda,\theta\}$, are positive and satisfy $\sum_t\alpha_{x,t}=\infty$, $\sum_t \alpha^2_{x,t}<\infty$, $\alpha_{\theta,t}=o(\alpha_{v,t}+\alpha_{\lambda,t})$, and $\lim_{t\rightarrow\infty}\alpha_{x,t+1}\alpha_{x,t}^{-1}=1$.
\end{assumption}

\begin{assumption}\label{as:policy_updates}
The update of the actor parameters $\theta_t^i$, $i\in\mathcal{N}^+$, includes a projection operator $\Psi_{\Theta^i}:\mathbb{R}^{m_i}\rightarrow\Theta^i\subset\mathbb{R}^{m_i}$, where $\Theta^i$ is a compact set defined as a hyperrectangle. 
\end{assumption}

\begin{assumption}\label{as:consensus_matrix}
The sequence of stochastic matrices $\{C_t\}_{t\geq 0}\in\mathbb{R}^{N\times N}$ is conditionally independent of all signals given any state-action pair $(s,a)$. The consensus matrix $C_t$ adheres to the communication graph $\mathcal{G}_t$, i.e., we have $[C_t]_{ij}=c_t(i,j)\geq\nu$ for some $\nu>0$ and every $(i,j)\in\mathcal{G}_t$. The mean communication graph $\mathbb{E}(\mathcal{G}_t|s_t,a_t)$ is connected and the mean consensus matrix is column stochastic, i.e., $\mathbf{1}^T\mathbb{E}(C_t|s_t,a_t)=\mathbf{1}^T$.
\end{assumption}
The assumptions above are analogous to \citep{zhang2018}. We provide several comments to motivate their use in the analysis.
\begin{remark}
By Assumption~\ref{as:step_size}, we can analyze convergence of Algorithm~\ref{alg:1} on separate timescales because the updates of the actor are slower than the updates of the critic and team-average reward function. This is a reasonable assumption as it is a common practice to perform multiple updates in the critic and team-average reward function (policy evaluation) before an actor update (policy improvement). The inclusion of the projection operator in the actor updates, as presented in Assumption~\ref{as:policy_updates}, is considered in the convergence analysis of RL algorithms but it is typically omitted in the implementation. Assumption~\ref{as:consensus_matrix} states that the consensus matrix has a well-defined mean value and the updates are balanced over the network on average for all visited state-action pairs. The nonzero consensus weights are assumed to be lower-bounded by a positive constant to ensure that the agents' values contract to a consensus value.
\hfill $\Box$
\end{remark}
The projection-based consensus AC algorithm provides a backbone for the resilient consensus AC algorithms in the following subsections. From here on, we will drop Assumption~\ref{as:consensus_matrix} and focus on the design of a resilient projection-based consensus update, whereby the agents can withstand Byzantine attacks on the critic parameters $v$ and team-average reward function parameters $\lambda$.

\subsection{Resilient Projection-based Consensus Actor-critic with Linear Approximation}
In the previous section, we introduced the projection-based consensus method to minimize the effect of disturbances that enter the consensus updates in the adversary-free setting. While the method does not provide resilience guarantees, it does present a convenient framework that can be further augmented to provide resilience to Byzantine attacks. We focus on this functionality and define a resilient projection-based consensus AC algorithm in this subsection.\par
To design a defense mechanism against Byzantine attacks presented in Definition~\ref{def:byzantine}, we adopt the basic idea of resilient consensus from W-MSR algorithms, a popular class of resilient consensus algorithms that apply a weighted trimmed mean. The acronym W-MSR stands for Weighted Mean-Subsequence-Reduced which describes a strategy for consensus updates whereby each agent \textit{reduces} scalar values received from its neighbors and, \textit{subsequently}, computes a \textit{weighted mean} of the remaining values. By eliminating the most extreme values at every step, the final agreed value among agents is guaranteed to lie within a convex hull of non-faulty agents if the network is sufficiently robust \citep{leblanc2013}.\par
We apply the W-MSR consensus method over the estimated neighbors' estimation errors. Each agent forms lists of sorted values $\{\epsilon_{v,t}^{ij}\}_{j\in\mathcal{N}_{in,t}^i}$ and $\{\epsilon_{\lambda,t}^{ij}\}_{j\in\mathcal{N}_{in,t}^i}$, and removes $H$ largest values and $H$ smallest values from each set, except for values that are smaller and larger than the value of the agent, respectively. Figure~\ref{fig:ex2} demonstrates that the estimation under the resilient projection-based consensus method performs better than under the method of trimmed means in the presence of an adversary at least in the simple problem introduced in Example~\ref{ex:1}.
\begin{figure}[t]
  \centering
\includegraphics[width=0.5\linewidth]{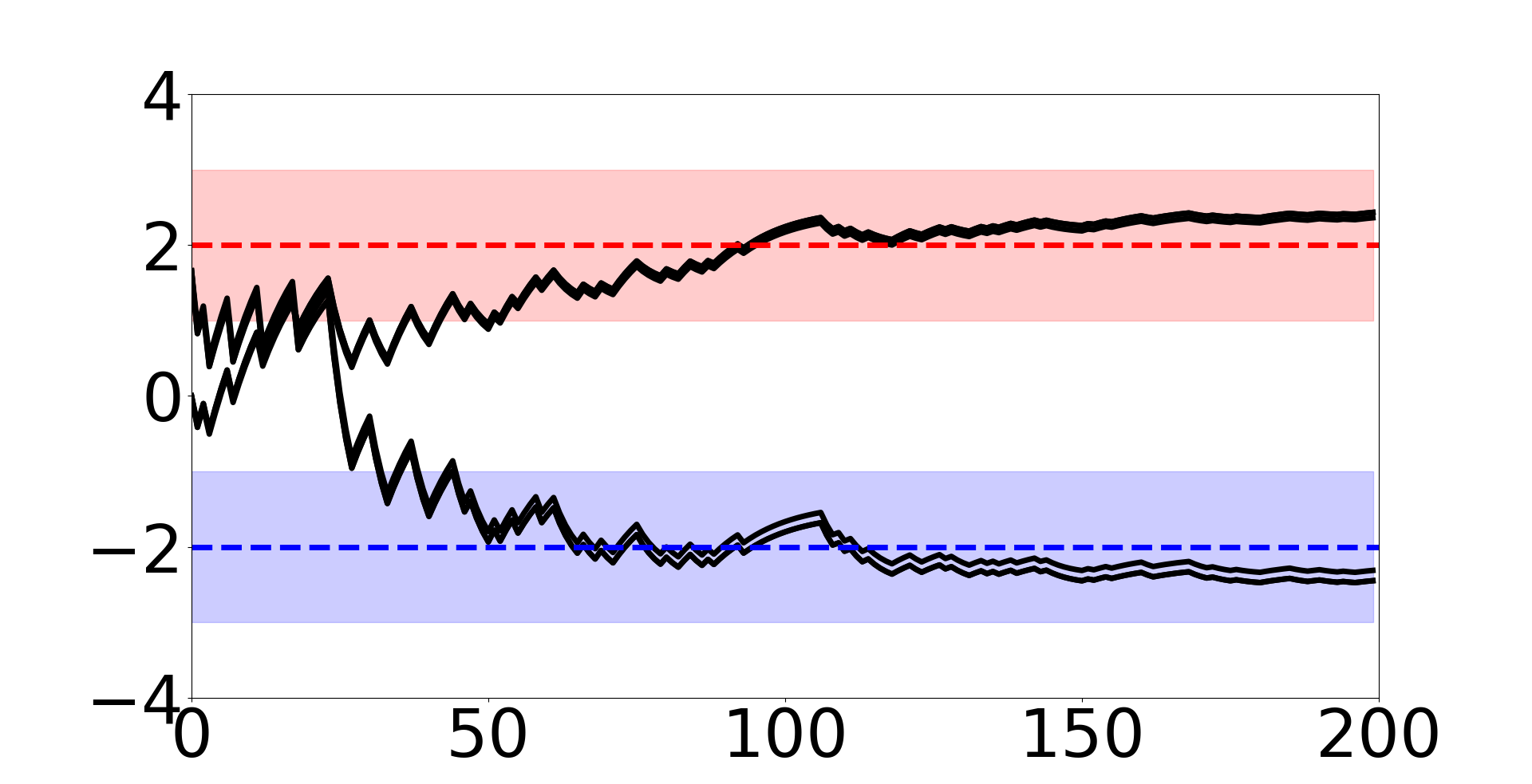}
\caption{Team-average reward function estimation in Example~\ref{ex:1} with the resilient projection-based consensus method. The estimated values $\bar{r}(s;\lambda^i)$ for $s\in\{0,1\}$ are shown in black color. The true team-average reward for $s\in\{0,1\}$ is depicted by the red and the blue dashed line, respectively. The shaded regions correspond to the convex hull of $r^i(s)$ for $i\in\mathcal{N}^+$.}
\label{fig:ex2}
\end{figure}
Note that the resilient projection-based consensus method does not suffer overestimation of the approximated functions because the Byzantine agents can no longer directly manipulate individual parameters in $v_t^j$ and $\lambda_t^j$. This limitation of the influence of Byzantine agents is due to the projection of $\tilde{v}_t^j$ and $\tilde\lambda_t^j$. We note that the resilient projection-based consensus method does not require \textit{network robustness} to scale with the dimension of transmitted data, i.e., we do not have to strengthen graph connectivity and increase the percentage of cooperative agents in the network to ensure resilience. In the following lines, we present the definition of \textit{reachable sets} and \textit{network robustness} and two associated lemmas.

\begin{definition}[$\zeta$-reachable sets \citep{leblanc2013}]
Given a directed graph $\mathcal{G}_t$ and a nonempty subset of nodes $\mathcal{Z}\subset\mathcal{N}$, we say that $\mathcal{Z}$ is an $\zeta$-reachable set if there exists $i\in\mathcal{Z}$ such that $|\mathcal{N}_t^i\backslash\mathcal{Z}|\geq \zeta$, where $\zeta\in\mathbb{Z}_{\geq 0}$.
\end{definition}

\begin{definition}[$\zeta$-robustness \citep{leblanc2013}]
A directed graph $\mathcal{G}_t$ on $|\mathcal{N}|$ nodes ($|\mathcal{N}| \geq 2$) is $\zeta$-robust, with $\zeta\in\mathbb{Z}_{\geq 0}$, if for every pair of nonempty, disjoint subsets of $\mathcal{N}$, at least one of the subsets is $\zeta$-reachable.
\end{definition}

\begin{lemma}[Network robustness after edge removal \citep{leblanc2013}]\label{lem:network_robustness}
Given a $\zeta$-robust directed graph $\mathcal{G}_t$, we let $\mathcal{G}_t^\prime$ denote the directed graph produced by removing up to $k$ received values. It follows that $\mathcal{G}_t^\prime$ is $(\zeta-k)$-robust.
\end{lemma}
By the definition of \textit{network robustness}, we can split the agents into any two groups. For a network that applies trimming (edge removal), there is always at least one agent that receives values from $\zeta-k$ agents in the other group. This leads to the notion of connectivity between agents, which is stipulated in the following lemma.
\begin{lemma}[Connectivity of robust graphs \citep{leblanc2013}]\label{lem:connectivity}
Suppose $\mathcal{G}_t$ is a $\zeta$-robust directed graph, with $0\leq\zeta\leq|\mathcal{N}|/2$. Then, $\mathcal{G}_t$ is at least $\zeta$-connected.
\end{lemma}

\begin{algorithm}[t]
	\SetAlgoLined
	 \textbf{Initialize} $s_0,\{\alpha_{v,t}\}_{t\geq 0},\{\alpha_{\lambda,t}\}_{t\geq 0}\{\alpha_{\theta,t}\}_{t\geq 0}, t\leftarrow 0, \theta_0^i,\lambda^i_0,\tilde\lambda^i_0,v_0^i,\tilde{v}^i_0$, $H$, $\forall i\in\mathcal{N}^+$\\
	 \textbf{Take} action $a_0\sim\pi(a_0|s_0;\theta_0)$\;
	 \textbf{Repeat until convergence} \\
 	 \For{$i\in\mathcal{N}^+$}{
  		\textbf{Observe} state $s_{t+1}$, action $a_t$, and reward $r_{t+1}^i$\\
  		\textbf{Update actor}\\
  		$\delta_t^i\leftarrow \bar{r}_{t+1}(\lambda_t^i)+\gamma V(s_{t+1};v_t^i)-V(s_t;v_t^i)$\\
  		$\psi_t^i\leftarrow\nabla_{\theta^i}\log\pi^i(a_t^i|s_t^i;\theta_t^i)$\\
  		$\theta_{t+1}^i\leftarrow\theta_t^i+\alpha_{\theta,t}\delta_t^i\psi_t^i$\\
  		\textbf{Update critic and team reward function}\\
  		$\tilde v_{t}^i\leftarrow v_t^i+\alpha_{v,t}\big(r_{t+1}^i+\gamma V(s_{t+1};v_t^i)-V(s_t;v_t^i)\big)\nabla_v V(s_t;v_t^i)$\\
  		$\tilde\lambda_{t}^i\leftarrow \lambda_t^i+\alpha_{\lambda,t}\big(r_{t+1}^i-\bar{r}_{t+1}(\lambda_t^i)\big)\nabla_{\lambda}\bar{r}_{t+1}(\lambda_t^i)$\\
  		\textbf{Send} $\tilde\lambda_t^i,\tilde{v}_t^i$ to $j\in\mathcal{N}_{out,t}^i$\\
  		}
  		\For{$i\in\mathcal{N}^+$}{
  		\textbf{Receive} $\tilde\lambda_t^j,\tilde{v}_t^j$ from $j\in\mathcal{N}_{in,t}^i$\\
  		\textbf{Resilient projection-based consensus step}\\
  		$\epsilon_{v,t}^{ij}\leftarrow\frac{\phi_t^T(\tilde{v}_t^j-v_t^i)}{\alpha_{v,t}\Vert\phi_t\Vert^2}$ for $j\in\mathcal{N}_{in,t}^i$, 
  		$\epsilon_{\lambda,t}^{ij}\leftarrow\frac{f_t^T(\tilde\lambda_t^j-\lambda_t^i)}{\alpha_{\lambda,t}\Vert f_t\Vert^2}$ for $j\in\mathcal{N}_{in,t}^i$\\
  		$\mathcal{N}_{v,t}^i\leftarrow$ remove $H$ smallest values that are smaller than and $H$ largest values that are larger than $\epsilon_{v,t}^{ii}$ from the set $\{\epsilon_{v,t}^{ij}\}_{j\in\mathcal{N}_{in,t}^i}$, return the remaining indices\\
  		  		$\mathcal{N}_{\lambda,t}^i\leftarrow$ remove $H$ smallest values that are smaller than and $H$ largest values that are larger than $\epsilon_{\lambda,t}^{ii}$ from the set $\{\epsilon_{\lambda,t}^{ij}\}_{j\in\mathcal{N}_{in,t}^i}$, return the remaining indices\\
  		$ \epsilon_{v,t}^i\leftarrow\sum_{j\in\mathcal{N}_{v,t}^i} c_{v,t}(i,j)\cdot\epsilon_{v,t}^{ij}$,  		$\epsilon_{\lambda,t}^i\leftarrow\sum_{j\in\mathcal{N}_{\lambda,t}^i} c_{\lambda,t}(i,j)\cdot\epsilon_{\lambda,t}^{ij}$\\
		$v_{t+1}^i\leftarrow v_t^i+\alpha_{v,t}\cdot\epsilon_{v,t}^i\cdot\nabla_vV(s_t;v_t^i)$\\
		$\lambda_{t+1}^i\leftarrow\lambda_t^i+\alpha_{\lambda,t}\cdot\epsilon_{\lambda,t}^i\cdot\nabla_{\lambda}\bar{r}_{t+1}(s_t,a_t;\lambda_t^i)$\\
  		\textbf{Take} action $a_{t+1}^i\sim\pi^i(a_{t+1}^i|s_{t+1};\theta_{t+1}^i)$\;
  		 }
  	\textbf{Update} iteration counter $t\leftarrow t+1$
 \caption{Resilient projection-based consensus actor-critic with linear approximation}
 	\label{alg:2}
\end{algorithm}
\noindent
The statements provided in Lemma~\ref{lem:network_robustness} and Lemma~\ref{lem:connectivity} establish that the network of agents remains connected despite each agent reducing the set of neighbors whose parameters are averaged in its local consensus updates. Moreover, the trimming approach ensures that the weighted means $\epsilon_{v,t}^{i}$ and $\epsilon_{\lambda,t}^{i}$ are bounded by the minimum and maximum values in the set of cooperative neighbors of agent $i$ as long as the following two assumptions hold.
\begin{assumption}\label{as:byzantine_agents}
There are at most $H$ Byzantine agents in the network.
\end{assumption}
\begin{assumption}\label{as:network_robustness}
The network is $(2H+1)$-robust.
\end{assumption}
These assumptions are standard in the consensus literature. We note that the resilient projection-based method is efficient even with a relatively high proportion of Byzantine agents in the network, i.e., it does not run into the curse of dimensionality unlike the exact vector aggregation in \citep{vaidya2014}.  In the following lines, we present assumptions on the consensus updates in the resilient projection-based consensus AC algorithm with linear function approximation and the policies of the Byzantine agents.
\begin{assumption}\label{as:consensus_matrix2}
For $x\in\{v,\lambda\}$, the consensus matrix $C_{x,t}$, is a row stochastic matrix that satisfies $[C_{x,t}]_{ij}=c_{x,t}(i,j)$. For each $j\in\mathcal{N}_{x,t}^i$, where $\mathcal{N}_{x,t}^i$ is the reduced set of nodes transmitting to agent $i$, the consensus weights satisfy $c_{x,t}(i,j)\geq\nu$ for some $\nu>0$.
\end{assumption}
\begin{assumption}\label{as:policy_adv}
The policy of every Byzantine agent converges to a stationary policy, i.e., $\lim_{t\rightarrow\infty}\pi_t^i\rightarrow\pi^i_*$ for $i\in\mathcal{N}^-$. Furthermore, $\pi_{t+1}^i-\pi_t^i=o(\alpha_{v,t}+\alpha_{\lambda,t})$.
\end{assumption}
\begin{remark}
We note that the consensus weights $c_{\lambda,t}(i,j)$ and $c_{v,t}(i,j)$ are distinct since the updates of $\lambda_t^i$ and $v_t^i$ are not the same in general. We will show in the convergence analysis in Section~\ref{sec:convergence} that the consensus updates of the cooperative agents can be expressed exclusively in terms of their values under Assumption~\ref{as:byzantine_agents}-\ref{as:consensus_matrix2}. Assumption~\ref{as:policy_adv} allows us to separate the parameter updates into two timescales, since the changes in the adversarial agents' policies are slower than the critic and team-average reward function updates. Moreover, the assumed stationarity of the Byzantine agents' policies in the limit allows us to prove that the cooperative agents maximize their team-average objective. In practice, the Byzantine agents can arbitrarily change their policies, which induces non-stationarity in the underlying MMDP. The important takeaway is that it does not alter the cooperative nature of the cooperative agents despite a potential non-convergence of policies to a stationary point of the cooperative-team-average objective function or its neighborhood. We do not assume a stationary behavior of the adversarial agents in the empirical analysis presented in Section~\ref{sec:simulation}.
\end{remark}
The pseudo-code for the resilient projection-based consensus AC algorithm with linear function approximation is given in Algorithm~\ref{alg:2}. In the next subsection, we extend the resilient projection-based AC algorithm to the setting with nonlinear function approximation of the actor, critic, and team-average reward.

\subsection{Deep Resilient Projection-based Consensus Actor-critic}
In this subsection, we extend the functionality demonstrated in Algorithm~2 to nonlinear function approximation.  We design the resilient projection-based AC algorithm, where the policy, value function, and team-average reward function are approximated by deep neural networks. Nonlinear function approximation allows a further degree of freedom in the team-average estimation as the agents learn the associated features adaptively. We let $\lambda_t^i=\begin{bmatrix} (\lambda_{hid,t}^i)^T & (\lambda_{out,t}^i)^T  \end{bmatrix}^T$ where $\lambda_{hid,t}^i$ and $\lambda_{out,t}^i$ denote stacked vectors with all hidden layer and output layer parameters of the team-reward function of agent $i$, respectively. On a similar note, we let $v_t^i=\begin{bmatrix} (v_{hid,t}^i)^T & v_{out,t}^i)^T  \end{bmatrix}^T$. The pseudo-code for the deep resilient projection-based AC MARL algorithm is presented in Algorithm~\ref{alg:3}.

\begin{algorithm}[h]
	\SetAlgoLined
	 \textbf{Initialize} $s_0,\{\alpha_{v,t}\}_{t\geq 0},,\{\alpha_{\lambda,t}\}_{t\geq 0},\{\alpha_{\theta,t}\}_{t\geq 0}, t\leftarrow 0, \theta_0^i,\lambda^i_0,\tilde\lambda^i_0,v_0^i,\tilde{v}^i_0$, $H$, $\forall i\in\mathcal{N}^+$\\
	 \textbf{Take} action $a_0\sim\pi(a_0|s_0;\theta_0)$\;
	 \textbf{Repeat until convergence} \\
 	 \For{$i\in\mathcal{N}^+$}{
  		\textbf{Observe} state $s_{t+1}$, action $a_t$, and reward $r_{t+1}^i$\\
  		\textbf{Update actor}\\
  		$\delta_t^i\leftarrow \bar{r}_{t+1}(\lambda_t^i)+\gamma V(s_{t+1};v_t^i)-V(s_t;v_t^i)$\\
  		$\psi_t^i\leftarrow\nabla_{\theta^i}\log\pi^i(a_t^i|s_t^i;\theta_t^i)$\\
  		$\theta_{t+1}^i\leftarrow\theta_t^i+\alpha_{\theta,t}\delta_t^i\psi_t^i$\\
  		\textbf{Update critic and team reward function}\\
  		$\tilde v_{t}^i\leftarrow v_{t}^i+\alpha_{v,t}\big(r_{t+1}^i+\gamma V(s_{t+1};v_t^i)-V(s_t;v_t^i)\big)\nabla_v V(s_t;v_t^i)$\\
  		$\tilde\lambda_{t}^i\leftarrow \lambda_{t}^i+\alpha_{\lambda,t}\big(r_{t+1}^i-\bar{r}_{t+1}(\lambda_t^i)\big)\nabla_\lambda\bar{r}_{t+1}(\lambda_t^i)$\\
  		\textbf{Send} $\tilde\lambda_{t}^i,\tilde{v}_{t}^i$ to $j\in\mathcal{N}_{out,t}^i$\\
  		}
  		\For{$i\in\mathcal{N}^+$}{
  		\textbf{Receive} $\tilde\lambda_{t}^j,\tilde{v}_{t}^j$ from $j\in\mathcal{N}_{in,t}^i$\\
  		\textbf{Resilient consensus step in the hidden layers (trimmed mean)}\\
		  		$\mathcal{N}_{v,t}^i\leftarrow$ remove $H$ smallest values and $H$ largest values from $\{\tilde{v}_{hid,t}^j\}_{j\in\mathcal{N}_{in,t}^i}$, return the indices of accepted agents\\
		  		$\mathcal{N}_{\lambda,t}^i\leftarrow$ remove $H$ smallest values and $H$ largest values from $\{\lambda_{hid,t}^j\}_{j\in\mathcal{N}_{in,t}^i}$, return the indices of accepted agents\\
		  		$v_{hid,t+1}^i\leftarrow\sum_{j\in\mathcal{N}_{v,t}^i} c_{v,t}^\prime(i,j)\cdot \tilde{v}_{hid}^j$, $\lambda_{hid,t+1}^i\leftarrow\sum_{j\in\mathcal{N}_{\lambda,t}^i} c_{\lambda,t}^\prime(i,j)\cdot\tilde\lambda_{hid,t}^j$\\
  		\textbf{Resilient projection-based consensus step in the output layer}\\
  		$\epsilon_{v,t}^{ij}\leftarrow\frac{V(s_t;\tilde{v}_t^j)-V(s_t;v_t^i)}{\alpha_{v,t}\Vert\nabla_{v_{out}}V(s_t;v_t^i)\Vert^2}$, $\epsilon_{\lambda,t}^{ij}\leftarrow\frac{\bar{r}(s_t,a_t;\tilde\lambda_t^j)-\bar{r}(s_t,a_t;\lambda_t^i)}{\alpha_{\lambda,t}\Vert \nabla_{\lambda_{out}}\bar{r}(s_t,a_t;\lambda_t^i)\Vert^2}$ for $j\in\mathcal{N}_{in,t}^i$\\
  		$\mathcal{N}_{v,t}^i\leftarrow$ remove $H$ smallest values that are smaller than and $H$ largest values that are larger than $\epsilon_{v,t}^{ii}$ from the set $\{\epsilon_{v,t}^{ij}\}_{j\in\mathcal{N}_{in,t}^i}$, return the remaining indices\\
  		  		$\mathcal{N}_{\lambda,t}^i\leftarrow$ remove $H$ smallest values that are smaller than and $H$ largest values that are larger than $\epsilon_{\lambda,t}^{ii}$ from the set $\{\epsilon_{\lambda,t}^{ij}\}_{j\in\mathcal{N}_{in,t}^i}$, return the remaining indices\\
  		$ \epsilon_{v,t}^i\leftarrow\sum_{j\in\mathcal{N}_{v,t}^i} c_{v,t}(i,j)\cdot\epsilon_{v,t}^{ij}$, $\epsilon_{\lambda,t}^i\leftarrow\sum_{j\in\mathcal{N}_{\lambda,t}^i} c_{\lambda,t}(i,j)\cdot\epsilon_{\lambda,t}^{ij}$\\
    	$v_{out,t+1}^i\leftarrow v_{out,t}^i+\alpha_{v,t}\cdot\epsilon_{v,t}^i\cdot\nabla_{v_{out}}V(s_t;v_{out,t}^i,v_{hid,t+1}^i)$\\	$\lambda_{t+1}^i\leftarrow\lambda_t^i+\alpha_{\lambda,t}\cdot\epsilon_{\lambda,t}^i\cdot\nabla_{\lambda_{out}}\bar{r}_{t+1}(s_t,a_t;\lambda_{out,t}^i,\lambda_{hid,t+1}^i)$\\
  		\textbf{Take} action $a_{t+1}^i\sim\pi^i(a_{t+1}^i|s_{t+1};\theta_{t+1}^i)$\\
  		 }
  	\textbf{Update} iteration counter $t\leftarrow t+1$

 \caption{Deep resilient projection-based consensus AC}
 	\label{alg:3}
\end{algorithm}

\section{Convergence Results}\label{sec:convergence}

In this section, we provide a convergence analysis for Algorithm~\ref{alg:1} and Algorithm~\ref{alg:2}. In the analysis, we treat the critic and team-average reward function updates, and the actor updates separately. We show that the critic and team-average reward function converge on the faster timescale and the actor converges on the slower timescale. The convergence analysis is inspired by the study of stochastic approximation in \citep{kushner2003book} and \citep{borkar2009book} and the seminal work on consensus AC algorithms with function approximation in \citep{zhang2018}, where the asymptotic behavior of stochastic updates can be conveniently expressed in terms of ordinary differential equations (ODEs) or differential inclusions. We note that our analysis is more concise than in \citep{zhang2018} as we associate the proof of boundedness of the algorithm iterates with existing theorems in \citep{borkar2009book}. Furthermore, the convergence results of Algorithm~1 are essentially the same as in \citep{zhang2018} despite a significant difference in the consensus updates due to the projection step. The analysis of Algorithm~\ref{alg:2} involves different assumptions on the network topology and consensus updates, and hence the convergence results are weaker to those obtained for Algorithm~\ref{alg:1}.\par
Before we proceed to prove the convergence under Algorithm~\ref{alg:1} and \ref{alg:2}, we introduce definitions that are used in both proofs. We let $v_t=\begin{bmatrix} (v_t^1)^T & \dots & (v_t^N)^T\end{bmatrix}^T$ and $\lambda_t=\begin{bmatrix} (\lambda_t^1)^T & \dots & (\lambda_t^N)^T\end{bmatrix}^T$. Furthermore, we let $D_\pi^{s}=diag([d_\pi(s),s\in\mathcal{S}])\in\mathbb{R}^{|\mathcal{S}|\times|\mathcal{S}|}$ and $D_\pi^{s,a}=diag([d_\pi^\prime(s,a),s\in\mathcal{S},a\in\mathcal{A}])\in\mathbb{R}^{(|\mathcal{S}|\cdot|\mathcal{A}|)\times(|\mathcal{S}|\cdot|\mathcal{A}|)}$ denote matrices with a stationary distribution of states and state-action pairs, respectively.
\begin{definition}[Team-average]
The averaging operator $\left<\cdot\right>:\mathbb{R}^{NK}\rightarrow\mathbb{R}^K$ is defined such that $\left<x\right>=\frac{1}{N}(\mathbf{1}^T\otimes I)x=\frac{1}{N}\sum_{i\in\mathcal{N}}x^i$, where $\otimes$ denotes the Kronecker product.
\end{definition}
\begin{definition}[Projection into consensus subspace]
	The operator $\mathcal{J}=(\frac{1}{N}\mathbf{11}^T)\otimes I$ is defined such that $\mathcal{J}x=\mathbf{1}\otimes \left<x\right>$.
\end{definition}
\begin{definition}[Projection into disagreement subspace]
	The operator $\mathcal{J}_{\perp}=I-\mathcal{J}$ is defined such that $x_\perp=\mathcal{J}_{\perp} x=x-\mathbf{1}\otimes \left<x\right>$.
\end{definition}
\begin{definition}[Projection into gradient subspace]
   The projection matrices are given as $\Gamma_{v,t}=\frac{\phi_t\phi_t^T}{\Vert\phi_t\Vert^2}$ and $\Gamma_{\lambda,t}=\frac{f_tf_t^T}{\Vert f_t\Vert^2}$.
\end{definition}
\begin{definition}[Projection into orthogonal subspace]
The orthogonal projection matrices are given as $\hat\Gamma_{v,t}=I-\Gamma_{v,t}$ and $\hat\Gamma_{\lambda,t}=I-\Gamma_{\lambda,t}$.
\end{definition}

\subsection{Algorithm~\ref{alg:1}} 
In the analysis of Algorithm~\ref{alg:1}, we are going to prove convergence to unique asymptotically stable equilibria of the critic and team-average reward estimate under fixed policy evaluation and convergence of the actor to the local maximizer of the approximated team-average objective function.\par
We let $A_{v,t}^\prime=\phi_t(\gamma\phi_{t+1}-\phi_t)^T$, $b_{v,t}^i=\phi_tr_{t+1}^i$, $A_{\lambda,t}^\prime=-f_tf_t^T$, and $b_{\lambda,t}^i=f_tr_{t+1}^i$. Furthermore, we define $A_{x,t}=I\otimes A_{x,t}^\prime$ and $b_{x,t}=\begin{bmatrix} (b_{x,t}^1)^T & \dots & (b_{x,t}^N)^T\end{bmatrix}^T$ for $x\in\{v,\lambda\}$. The critic and team-average reward function updates under Algorithm~\ref{alg:1} are compactly written as
\begin{align}
v_{t+1}&=v_t+(C_t\otimes\Gamma_{v,t})\big(v_t+\alpha_{v,t}(A_{v,t}v_t+b_{v,t})\big)-(I\otimes\Gamma_{v,t})v_t\label{alg1_v}\\
\lambda_{t+1}&=\lambda_t+(C_t\otimes\Gamma_{\lambda,t})\big(\lambda_t+\alpha_{\lambda,t}(A_{\lambda,t}\lambda_t+b_{\lambda,t})\big)-(I\otimes\Gamma_{\lambda,t})\lambda_t.\label{alg1_lam}
\end{align}
Our first goal in the analysis is to describe how the consensus updates affect the asymptotic convergence of the parameters $v_t$ and $\lambda_t$. We want to establish that the agents reach consensus on the parameter values in the limit. In Lemma~\ref{lem:spectral_radius}, we analyze the spectral radius of the mean consensus update in the disagreement subspace, which facilitates contraction in the disagreement subspace stated in Lemma~\ref{lem:v_consensus} and \ref{lem:lam_consensus}. We let $\mathcal{F}_t^{x}=\sigma(x_0,Y_{t-1},\xi_\tau,\tau\leq t)$ denote a filtration of a random variable $x\in\{v,\lambda\}$, where $Y_t$ are the incremental changes in parameters $v$ and $\lambda$ due to \eqref{alg1_v} or \eqref{alg1_lam}, and $\xi_\tau=(r_\tau,s_\tau,a_\tau,C_{\tau-1})$ is a collection of random signals. The filtration captures the evolution of the random variable based on its initial value $x_0$ and updates $Y_t$ that occur along the trajectory of the Markov chain $\xi_\tau$.

\begin{lemma}\label{lem:spectral_radius}
Under Assumption~\ref{as:consensus_matrix}, the spectral radius $\rho_t\big(\mathbb{E}(C_t^T(I-\mathbf{11}^T/N)C_t|\mathcal{F}_t^x)\big)<1$, where $x\in\{v,\lambda\}$.
\end{lemma}

\begin{lemma}\label{lem:v_consensus}
Suppose there are no adversarial agents in the network. Under Assumption~\ref{as:linear_approx}-\ref{as:step_size} and \ref{as:consensus_matrix}, the agents reach consensus on the critic parameters with probability one, i.e., $\lim_{t\rightarrow\infty}v_{\perp,t}=0$. Furthermore, the term $\mathbb{E}(\Vert(I\otimes\Gamma_{v,t})\alpha_{v,t}^{-1}v_{\perp,t}\Vert^2|\mathcal{F}_t^v)$ is uniformly bounded.
\end{lemma}

\begin{lemma}\label{lem:lam_consensus}
Suppose there are no adversarial agents in the network. Under Assumption~\ref{as:linear_approx}-\ref{as:step_size} and \ref{as:consensus_matrix}, the agents reach consensus on the team-average reward function parameters with probability one, i.e., $\lim_{t\rightarrow\infty}\lambda_{\perp,t}=0$. Furthermore, the term $\mathbb{E}(\Vert(I\otimes\Gamma_{\lambda,t})\alpha_{\lambda,t}^{-1}\lambda_{\perp,t}\Vert^2|\mathcal{F}_t^\lambda)$ is uniformly bounded.
\end{lemma}
\noindent An intermediate result of Lemma~\ref{lem:v_consensus} and \ref{lem:lam_consensus} states that under a sufficiently small step size, the disagreement vector scaled by the step size is contractive and subject to a bounded ``input" disturbance that is due to the heterogeneous rewards observed by individual agents. Therefore, if the trajectories of  $\alpha_{v,t}^{-1}v_{\perp,t}$ and $\alpha_{\lambda,t}^{-1}\lambda_{\perp,t}$ happen to escape a compact set for $t>t_0$, the trajectories exponentially converge back to the set, which implies the boundedness in the disagreement space. This property is extremely useful in the analysis of the team-average parameter trajectories, since the disagreement vectors $v_{\perp,t}$ and $\lambda_{\perp,t}$ represent a ``bias" that enters the team-average stochastic updates. We are ready to state important theorems regarding convergence of the team-average parameter trajectories.

\begin{theorem}\label{thm:v_team_convergence}
Suppose there are no adversarial agents in the network. Under Assumption~\ref{as:linear_approx}-\ref{as:step_size}~and~\ref{as:consensus_matrix}, the critic parameters satisfy $\sup_t\Vert v_t\Vert<\infty$ with probability one. Furthermore, they asymptotically converge with probability one, i.e., $\lim_t v_t^i=v_\pi$ for $i\in\mathcal{N}$. The limit $v_\pi$ is a unique solution to $\Phi^TD_\pi^s\big(\frac{1}{N}\sum_{i\in\mathcal{N}}R_\pi^i+\gamma P_\theta\Phi v_\pi-\Phi v_\pi\big)=0$.
\end{theorem}
\noindent{\bf Proof Sketch.\;}
We express the team-average critic parameter updates in terms of the mean update conditioned on a filtration of random variables $\mathcal{F}_t^v$  and ``noise" due to sampling. We verify conditions in \ref{ap:unc} and leverage ergodicity of the Markov chain $\xi_t$ to establish convergence properties of the team-average critic parameter $\left< v_t\right>$.

\begin{theorem}\label{thm:lam_team_convergence}
Suppose there are no adversarial agents in the network. Under Assumption~\ref{as:linear_approx}-\ref{as:step_size}~and~\ref{as:consensus_matrix}, the team-average reward function parameters satisfy $\sup_t\Vert \lambda_t\Vert<\infty$ with probability one. Furthermore, they converge with probability one, i.e., $\lim_t\lambda_t^i=\lambda_\pi$ for $i\in\mathcal{N}$. The limit $\lambda_\pi$ is a unique solution to $F^TD_\pi^{s,a}\big(\frac{1}{N}\sum_{i\in\mathcal{N}}R^i-F\lambda_\pi\big)=0$.
\end{theorem}
\noindent We note that the convergence point $\lambda_\pi$ corresponds to the minimum mean squared error estimate of the true team-average reward weighted over the distribution of state-action pairs $(s,a)$, $d_\pi^\prime(s,a)$, and the convergence point $v_\pi$ corresponds to the minimum mean squared projected Bellman error estimate of the true team-average critic weighted over the state distribution $d_\pi(s)$. Theorem~\ref{thm:v_team_convergence} and \ref{thm:lam_team_convergence} serve evidence that the cooperative agents end up estimating the team-average advantage function as the team-average information is reflected in the limits $v_\pi$ and $\lambda_\pi$. In the following lines, we use the findings from Theorem~\ref{thm:v_team_convergence} and \ref{thm:lam_team_convergence} to establish convergence of the actor. This represents the most important result in the convergence analysis of Algorithm~\ref{alg:1} as it determines that the agents approximately maximize the objective function defined in \eqref{local opt problem}. We are ready to state the main theorem.

\begin{theorem}\label{thm:actor}
Suppose there are no adversarial agents in the network. Under Assumption~\ref{as:linear_approx}-\ref{as:consensus_matrix}, the policy parameter $\theta_t^i$, $i\in\mathcal{N}$, converges with probability one to a set of locally asymptotically stable equilibria of the ODE
\begin{align*}
\dot\theta^i=\Psi^i_\Theta\big[\mathbb{E}_{\pi,d_\pi,p}\big\{\big(\bar{r}(s,a;\lambda_\pi)+\gamma V(s^\prime;v_\pi)-V(s;v_\pi)\big)\nabla\log\pi^i(a^i|s;\theta^i)\big\}\big],
\end{align*}
where the parameters $\lambda_\pi$ and $v_\pi$ are the GAS equilibria under policy $\pi(a|s;\theta)$.
\end{theorem}
\noindent{\bf Proof Sketch.\;} 
We express the local actor parameter updates in terms of the mean update over the stationary distribution of the states and actions, ``noise'' due to the currently visited state-action pair, and ``bias" due to the convergence of the critic and team-average reward function parameters on the faster timescale. We apply analytical results from Appendix~\ref{ap:con} to establish asymptotic convergence of the actor updates.\par
The convergence analysis of Algorithm~\ref{alg:1} demonstrates that the cooperative agents approximately maximize the true team-average objective function $J^+(\pi^+)$.  The distance between the approximate maximizer and the true maximizer depends on the bias due to TD learning and on the accuracy of the surrogate models used to approximate the team-average objective function $J^+(\pi^+)$, namely the team-average reward function $\bar{r}(s,a;\lambda^i)$ and the critic $V(s;v^i)$. In the next subsection, we provide the convergence analysis of Algorithm~\ref{alg:2}. 

\subsection{Algorithm 2}
In this subsection, we analyze the convergence of the actor, critic, and team-average reward updates under Algorithm~\ref{alg:2}. While most of the analytical tools carry over from the previous subsection, it is important to note some major differences between the two algorithms and their impact on the convergence results. Whether or not adversarial agents are present in the network, Algorithm~\ref{alg:2} includes a resilient consensus step that is no longer conditionally independent of the parameter values $\tilde{v}_t^i$ and $\tilde\lambda_t^i$ received from other agents. Therefore, it is expected that the parameters $v^i$ and $\lambda^i$ converge at best to compact sets as opposed to the unique GAS equilibria $v_\pi$ and $\lambda_\pi$ in the case of Algorithm~\ref{alg:1}. Most importantly, the different convergence properties of $v^i$ and $\lambda^i$ under Algorithm~\ref{alg:2} impact the actor convergence as a consequence. To distinguish between the parameters of cooperative and Byzantine agents, we make slight changes in the notation. Without loss of generality, we assume that the agents' indices are ordered such that $\mathcal{N}^+=\{1,\dots,N^+\}$ and $\mathcal{N}^-=\{N-N^-+1,\dots,N\}$. We use superscript $+$ and $-$ to denote signals of all cooperative and Byzantine agents, respectively. For example, the cooperative agents' rewards are given as $r_{t+1}^+=[(r_{t+1}^1)^T\,\dots\,(r_{t+1}^{{N}^+})^T]^T$.  Highlighting the fact that each cooperative agent that employs Algorithm~\ref{alg:2} truncates up to $H$ values greater and smaller than its own value, we have the following lemma on the consensus updates.
\begin{lemma}[{\citet[prop.~5.1]{sundaram2018}}]\label{lem:equivalent_consensus}
Under Assumption~\ref{as:byzantine_agents}-\ref{as:consensus_matrix2},\\ the resilient consensus update for each $i\in\mathcal{N}^+$ with weights $c_{x,t}(i,j)$ is mathematically equivalent to $\epsilon_{x,t}^i=\sum_{j\in\mathcal{N}_{in,t}^i\cap\mathcal{N}^+}c_{x,t}^+(i,j)\epsilon_{x,t}^{ij}$, where $c_{x,t}^+(i,j)$ are consensus weights that satisfy $\sum_{j\in\mathcal{N}_{in,t}^i\cap\mathcal{N}^+}c_{x,t}^+(i,j)=1$. Moreover, it holds that $c_{x,t}^+(i,i)\geq\nu$ and $c_{x,t}^+(i,j)\geq\nu/2$ for some $\nu>0$ and $j\in\mathcal{N}_{in,t}^i\cap\mathcal{N}^+$.
\end{lemma}
\noindent Lemma~\ref{lem:equivalent_consensus} ensures that the consensus updates of the cooperative agents can be expressed purely in terms of the cooperative agents' values. In addition to the connectivity of the subgraph of cooperative agents that was already stated in Lemma~\ref{lem:connectivity}, Lemma~\ref{lem:equivalent_consensus} ensures that there exist equivalent lower-bounded consensus weights $c_{x,t}^+(i,j)$ associated with the edges of the subgraph. The equivalence of consensus updates is very useful in our analysis as we want to formulate the critic and team-reward parameter updates of the cooperative agents in terms of their own values. We let $C_{x,t}$ and $C_{x,t}^+$, $x\in\{v,\lambda\}$, denote the original and equivalent consensus matrices and define $A_{x,t}^+=(I\otimes A_{x,t}^\prime)$ and $b_{x,t}^+=[(b_{x,t}^1)^T\,\dots\, b_{x,t}^{N^+})^T]^T$. We apply Lemma~\ref{lem:equivalent_consensus} to write the cooperative agents' updates in Algorithm~\ref{alg:2} as follows
\begin{align}
v_{t+1}^+&=v_t^++(C_{v,t}^+\otimes\Gamma_{v,t})\big(v_t^++\alpha_{v,t}(A_{v,t}^+v_t^++b_{v,t}^+)\big)-(I\otimes\Gamma_{v,t})v_t^+\label{alg2_v}\\
\lambda_{t+1}^+&=\lambda_t^++(C_{\lambda,t}^+\otimes\Gamma_{\lambda,t})\big(\lambda_t^++\alpha_{\lambda,t}(A_{\lambda,t}^+\lambda_t^++b_{\lambda,t}^+)\big)-(I\otimes\Gamma_{\lambda,t})\lambda_t^+.\label{alg2_lam}
\end{align}
We are familiar with these updates from the analysis of Algorithm~\ref{alg:1}. However, it is important to note that while the updates retain the form under Assumption~\ref{as:byzantine_agents}-\ref{as:consensus_matrix2}, the consensus matrices $C_{v,t}^+$ and $C_{\lambda,t}^+$ are influenced by the Byzantine agents and are not unique in general. In contrast to the analysis of Algorithm~\ref{alg:1} that assumed mean connectivity of the communication graph, here we assume that the cooperative agents remain connected despite applying the truncation step for $t>0$. Therefore, we drop the mean representation and analyze the parameter updates in the disagreement subspace at every step. We present three lemmas that show contraction of the parameter updates in the disagreement subspace.

\begin{lemma}\label{lem:spectral_radius2}
Under Assumption~\ref{as:byzantine_agents}-\ref{as:consensus_matrix2}, the spectral radius $\rho_t^+\big(C_{x,t}^{+T}(I-\mathbf{11}^T/N^+)C_{x,t}^+\big)<1$, where $x\in\{v,\lambda\}$.
\end{lemma}

\begin{lemma}\label{lem:v_consensus2}
Under Assumption~\ref{as:linear_approx}-\ref{as:step_size} and \ref{as:byzantine_agents}-\ref{as:consensus_matrix2}, the agents reach consensus on the critic parameters with probability one, i.e., $\lim_{t\rightarrow\infty}v_{\perp,t}^+=0$. Furthermore, the term $\Vert(I\otimes\Gamma_{v,t})\alpha_{v,t}^{-1}v_{\perp,t}^+\Vert^2$ is uniformly bounded.
\end{lemma}

\begin{lemma}\label{lem:lam_consensus2}
Under Assumption~\ref{as:linear_approx}-\ref{as:step_size} and \ref{as:byzantine_agents}-\ref{as:consensus_matrix2}, the agents reach consensus on the team-average reward function parameters with probability one, i.e., $\lim_{t\rightarrow\infty}\lambda_{\perp,t}^+=0$. Furthermore, the term $\Vert(I\otimes\Gamma_{\lambda,t})\alpha_{\lambda,t}^{-1}\lambda_{\perp,t}^+\Vert^2$ is uniformly bounded.
\end{lemma}
\noindent
Lemma~\ref{lem:v_consensus2} and \ref{lem:lam_consensus2} ensure that the updates in the disagreement subspace become contractive in finite time and are subject only to a bounded disturbance that originates in the homogeneous rewards. Therefore, the trajectories of $v_{\perp,t}^+$ and $\lambda_{\perp,t}^+$ remain in a compact set for $t>0$. This property is crucial in the convergence theorem for the team-average values that we state in the following lines. Under Assumption~\ref{as:policy_adv}, we can still apply the two-timescale principle since the actor updates are slower than the updates of the critic and team-average reward function for all agents. Hence, we consider fixed policies in the convergence theorems for the team-average critic $\left<v_t\right>$ and team-average reward function parameters $\left<\lambda_t\right>$.

\begin{theorem}\label{thm:v_team_convergence2}
Under Assumption~\ref{as:linear_approx}-\ref{as:policy_adv} and \ref{as:byzantine_agents}-\ref{as:consensus_matrix2}, the critic parameters are uniformly bounded and converge with probability one to a bounded neighborhood around a fixed point $v_\pi^+$ that satisfies $$\Phi^TD_\pi^s\bigg(\frac{1}{N^+}\sum_{i\in\mathcal{N^+}}R_\pi^i+\gamma P_\pi\Phi v_\pi^+-\Phi v_\pi^+\bigg)=0.$$
The limiting sequence of the team-average update, $\left<v^+\right>$, satisfies $\Vert\Phi^TD_\pi^s(\gamma P_\pi-I)\Phi(\left<v^+\right>-v_\pi^+)\Vert\leq\Vert\Delta_v\Vert$, where
\begin{align*}
\Vert\Delta_v\Vert&\leq\lim_{t,m}\sup_{\xi_t}\bigg\Vert\frac{1}{m}\sum_{k=t}^{t+m-1}\frac{1}{N^+}\bigg((\mathbf{1}^TC_{v,k}^+r_{k+1}^+\otimes \phi_k)\nonumber\\
&\quad+(\mathbf{1}^TC_{v,k}^+\otimes\frac{\phi_k\phi_k^T}{\Vert \phi_k\Vert^2})\alpha_{v,k}^{-1}v_{\perp,k}^+\bigg)- \frac{1}{N^+}\sum_{i\in\mathcal{N^+}}\Phi^TD_\pi^{s} R_\pi^i\bigg\Vert.
\end{align*}
\end{theorem}
\noindent{\bf Proof Sketch.\;} Similarly to the proof of Theorem~\ref{thm:v_team_convergence}, we decompose the team-average critic parameter updates into the mean update conditioned on the filtration of random variables and ``noise" due to sampling. Even though the Markov chain is not ergodic in this case, we characterize the asymptotic properties by verifying conditions in Appendix~\ref{ap:unc}.

\begin{theorem}\label{thm:lam_team_convergence2}
Under Assumption~\ref{as:linear_approx}-\ref{as:step_size} and \ref{as:byzantine_agents}-\ref{as:consensus_matrix2}, the team-average reward function parameters are uniformly bounded and converge with probability one to a bounded neighborhood around a fixed point $\lambda_\pi^+$ that satisfies $$F^TD_\pi^{s,a}\bigg(\frac{1}{N^+}\sum_{i\in\mathcal{N}^+}R^i-F\lambda_\pi^+\bigg)=0.$$
The limiting sequence of the team-average update, $\left<\lambda^+\right>$, satisfies $\Vert F^TD_\pi^{s,a}F(\left<\lambda^+\right>-\lambda_\pi^+)\Vert\leq\Vert\Delta_{\lambda}\Vert$, where
\begin{align*}
\Vert\Delta_\lambda\Vert&=\lim_{t,m\rightarrow\infty}\sup_{\xi_t}\bigg\Vert\frac{1}{m}\sum_{k=t}^{t+m-1}\frac{1}{N^+}\bigg((\mathbf{1}^TC_{\lambda,k}^+r_{k+1}^+\otimes f_k)\nonumber\\
&\qquad+(\mathbf{1}^TC_{\lambda,k}^+\otimes\frac{f_kf_k^T}{\Vert f_k\Vert^2})\alpha_{\lambda,t}^{-1}\lambda_{\perp,k}^+\bigg) - \frac{1}{N^+}\sum_{i\in\mathcal{N^+}}F^TD_\pi^{s,a}R^i\bigg\Vert.
\end{align*}
\end{theorem}
\noindent
The convergence results of the critic and team-average reward function can be interpreted as follows. The team-average values converge to a neighborhood of the desired optima. There are several factors that determine the size of the neighborhood. One, the equivalent consensus matrices $C_{v,t}^+$ and $C_{\lambda,t}^+$ are not column stochastic; hence, some local rewards and discrepancies between agents parameters are weighted more heavily than others. Two, the neighborhood grows with the discrepancy in agents' rewards. In the special case when the cooperative agents receive the same reward, i.e., $r_{t+1}^i=r_{t+1}^j$ for $i,j\in\mathcal{N}^+$ and $t\geq 0$, the neighborhood shrinks to a singleton because the first term in the summation cancels out with the negative term, and the second term in the summation goes to zero almost surely because the disagreement vector contracts exponentially after some finite time $t_0$ (this can be verified in Lemma~\ref{lem:v_consensus2}). This is a very interesting result since the Byzantine agents cannot hurt the distributed optimization process by attacking the communication channels in this case. In the following lines, we introduce the main theorem regarding the actor convergence.

\begin{theorem}\label{thm:actor2}
Under Assumption~\ref{as:linear_approx}-\ref{as:policy_updates} and \ref{as:byzantine_agents}-\ref{as:policy_adv}, the policy parameter $\theta_t^i$, $i\in\mathcal{N}^+$, converges with probability one to a neighborhood of a locally asymptotically stable equilibrium of the ODE
$$\dot\theta^i=\Psi^i_\Theta\big[\mathbb{E}_{\pi,d_\pi,p}\big\{\big(\bar{r}(s,a;\lambda_\pi^+)+\gamma V(s^\prime;v_\pi^+)-V(s;v_\pi^+)\big)\nabla\log\pi^i(a^i|s;\theta^i)\big\}\big].$$
\end{theorem}
\noindent{ \bf Proof Sketch. \;} We split the actor updates into the average updates and ``noise" that is due to the currently visited state of the Markov chain $(r_{t}^+,s_t,a_t,C_{v,t-1}^+,C_{\lambda,t-1}^+)$. Since the Markov chain is not ergodic, we have to account for all its possible realizations to describe the convergence properties of the actor updates.

Due to the ever-present disturbance in the approximation of the critic and team-average reward function, the actor converges at best to the neighborhood of a local maximum of the approximated cooperative team's objective function $J^+(\theta^+,\pi^-)=\mathbb{E}_{\pi,d_\pi,p}\big[\bar{r}(s,a;\lambda_\pi^+)+\gamma V(s^\prime;v_\pi^+)\big]$ when the Byzantine agents' policies $\pi^-$ are stationary. Otherwise, the cooperative agents search for a locally optimal policy in a non-stationary environment induced by the Byzantine agents. Although we do not present any strategy for detection of adversarial agents in this paper, the agents can potentially identify an adversary through monitoring each other's policies. Therefore, the assumed stationarity should not be seen as an obstacle in the analysis. In the next section, we show great empirical performance of Algorithm~\ref{alg:2} in the cooperative navigation task despite the Byzantine agents actively changing their policies.

\section{Simulation Results}\label{sec:simulation}
In this section, we provide an empirical analysis of the deep resilient projection-based consensus AC MARL algorithm (Algorithm~\ref{alg:3}). We consider a multi-agent grid world environment, where the agents learn to solve the cooperative navigation task. The environment is of size $6\times 6$ and includes five agents that communicate their parameter values on a directed graph. We refer the reader to Figure~\ref{fig:cooperative_navigation} for illustration. The goal of each agent is to find the shortest path to a fixed random desired position and avoid collisions with other agents. We let $s^i$ denote the 2D positional coordinates of agent~$i$ and $d^i$ the 2D positional coordinates of its desired position. Each agent chooses from a set of five actions that correspond to the cardinal direction of its next state transition and staying put in the same state. The dynamics are deterministic and the agents are coupled through the reward function $r_{t+1}^i(s_t,a_t)=-\Vert s_{t+1}^i-d^i\Vert_1-\delta(min_{j\in\mathcal{N}}\Vert s_{t+1}^i-s_{t+1}^j\Vert_1)$, where $\delta(x)$ is the Dirac delta function. The first term penalizes the distance from the desired position and the second term facilitates a penalty for collision. Each agent in the network uses neural networks with two hidden layers, 30 hidden units in each layer, and leaky ReLU activation function to approximate the actor, critic, and team-average reward approximation.\par
We train the network for 10000 episodes that last 20 steps. The team-average reward function and the critic are evaluated under a fixed policy $\pi(a|s)$ after every 100 episodes. The agents perform batch stochastic updates and resilient projection-based consensus updates over 
20 epochs, which are followed by a single actor update. We select the discount factor $\gamma=0.9$ and the learning rates $\alpha_{\theta,t}=0.002$ and $\alpha_{v,t}=\alpha_{\lambda,t}=0.01$. In the simulations, we consider four scenarios with different adversarial attacks and compare the performance of the deep resilient projection-based consensus AC algorithm under the trimming hyperparameter $H=0$ and $H=1$. The results presented in Figure~\ref{fig:sim_results} correspond to an average outcome of five simulations of each scenario that are initialized with different random seeds and the same simulation parameters.\par

\begin{figure}[t]
\begin{subfigure}{.49\textwidth}
\centering
\includegraphics[width=1\linewidth]{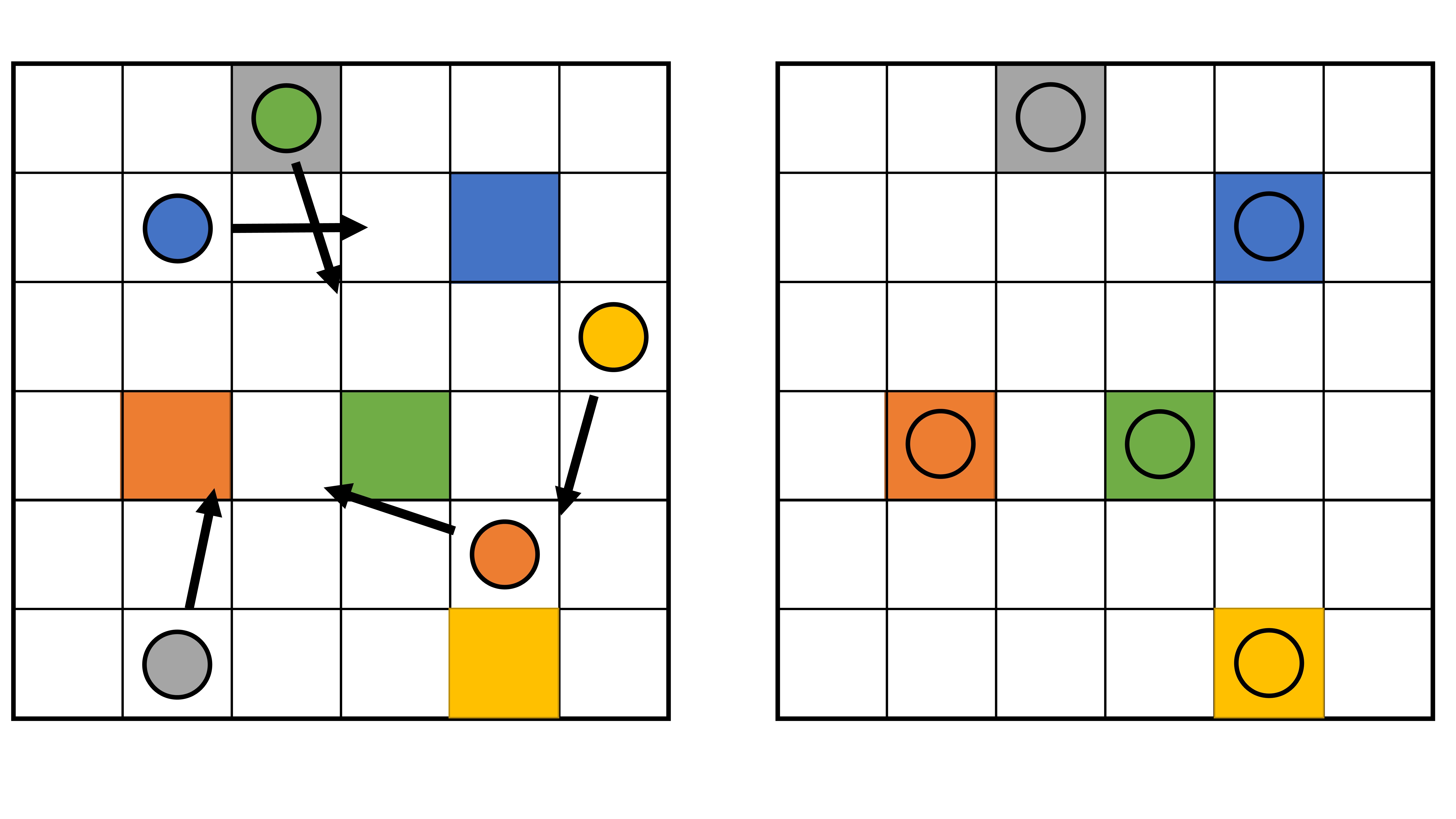}
\end{subfigure}
\begin{subfigure}{.49\textwidth}
\centering
\includegraphics[width=0.65\linewidth]{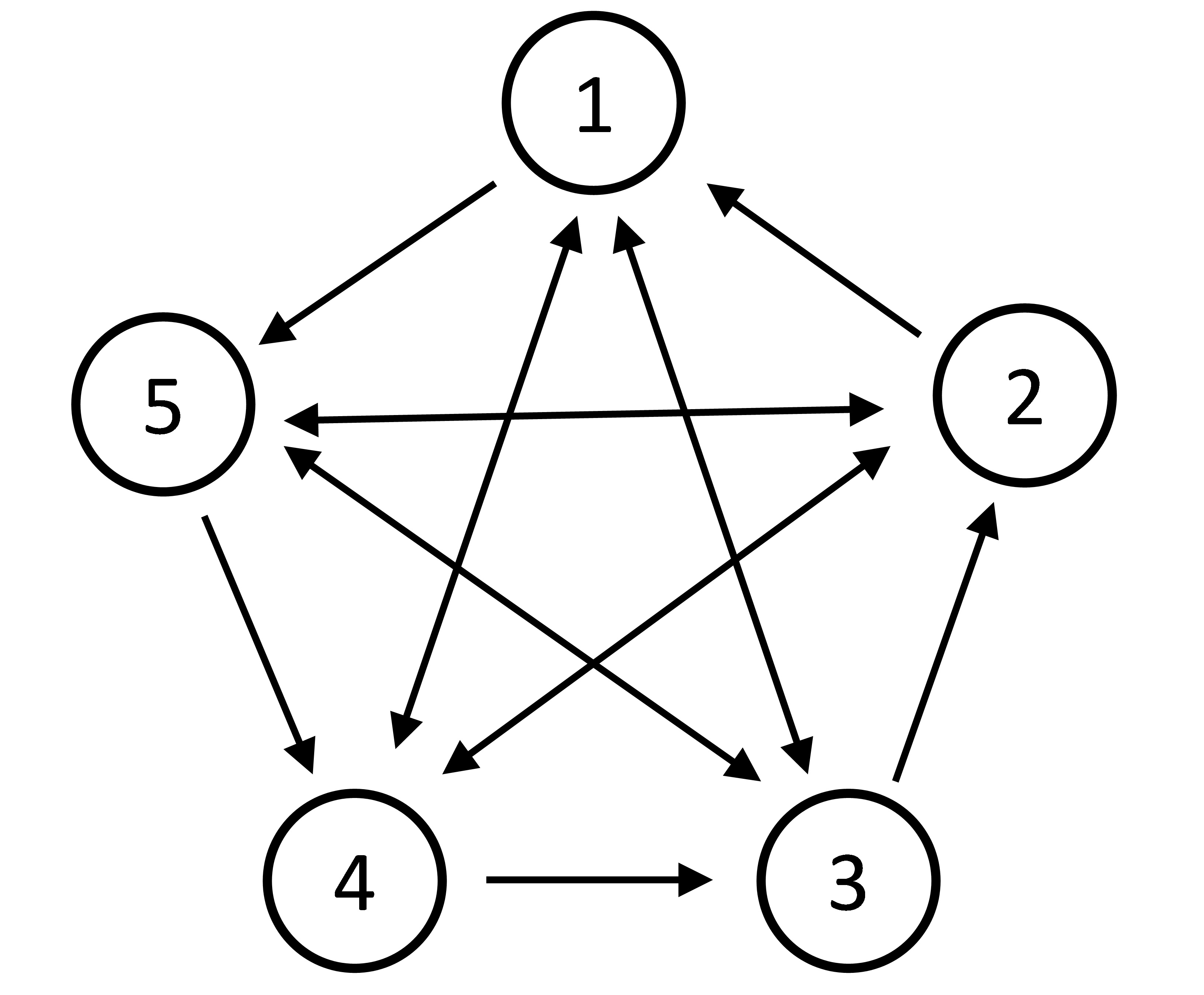}
\end{subfigure}
\caption{Cooperative navigation task in the grid world environment. Agents  navigate to their desired position (left) and communicate with one another on a directed graph (right).}
\label{fig:cooperative_navigation}
\end{figure}
\begin{figure}[t]
\centering
\begin{subfigure}{.45\textwidth}
\centering
\includegraphics[width=1\linewidth]{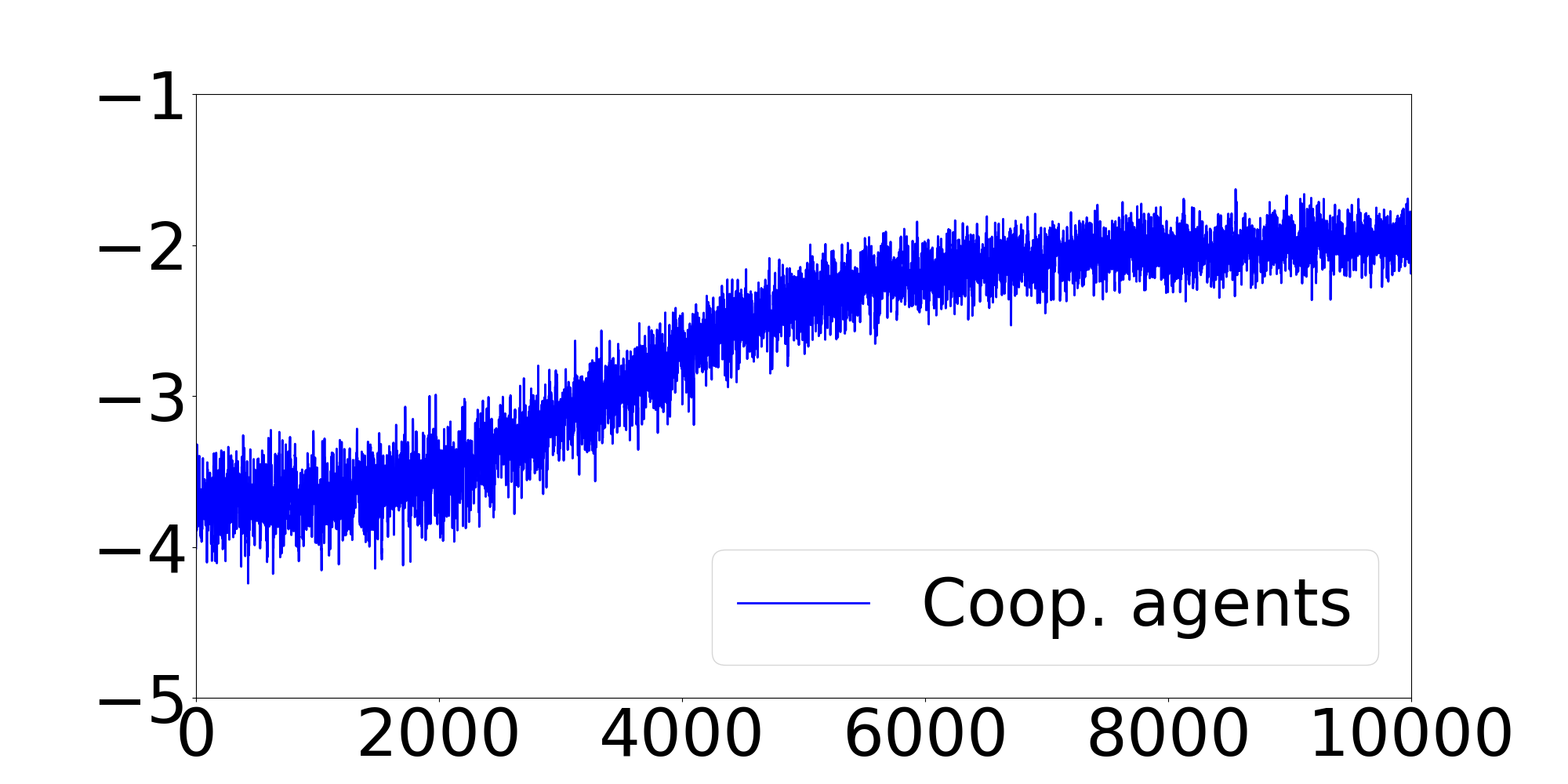}
  \caption{All-cooperative, $H=0$.}
\end{subfigure}
\begin{subfigure}{.45\textwidth}
\centering
\includegraphics[width=1\linewidth]{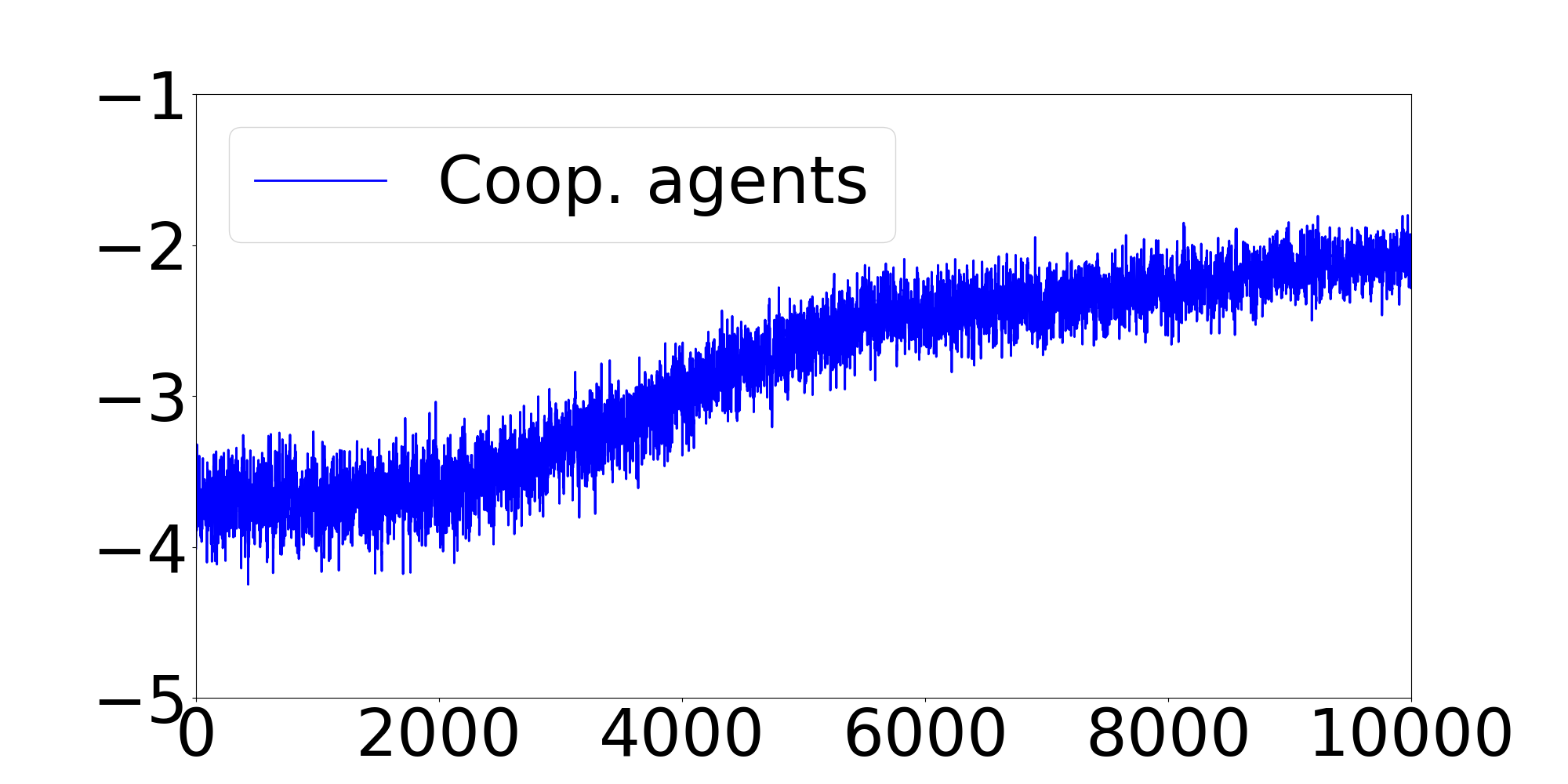}
\caption{All-cooperative, $H=1$.}
\end{subfigure}
\begin{subfigure}{.45\textwidth}
\centering
\includegraphics[width=1\linewidth]{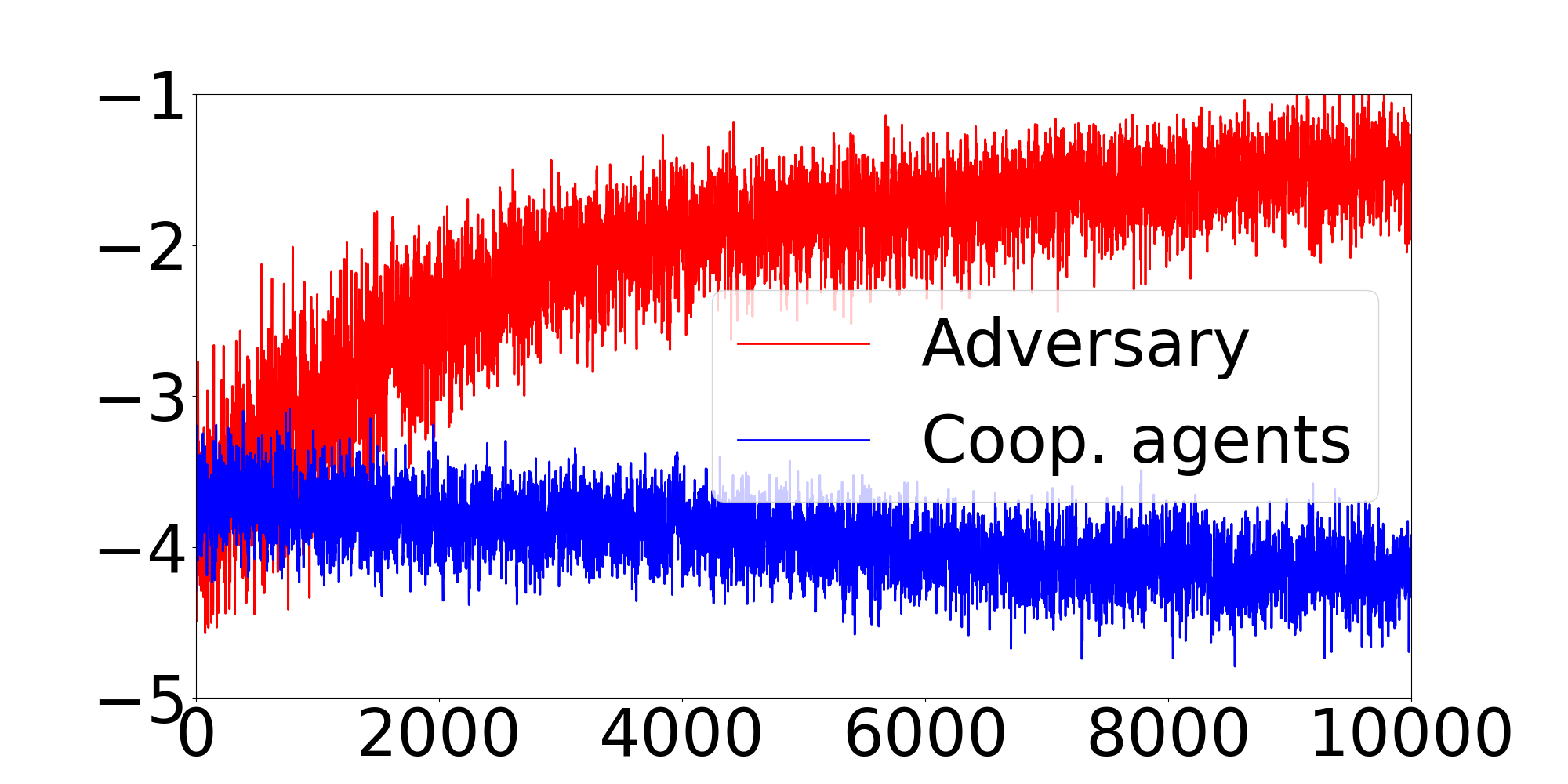}
  \caption{Greedy, $H=0$.}
\end{subfigure}
\begin{subfigure}{.45\textwidth}
\centering
\includegraphics[width=1\linewidth]{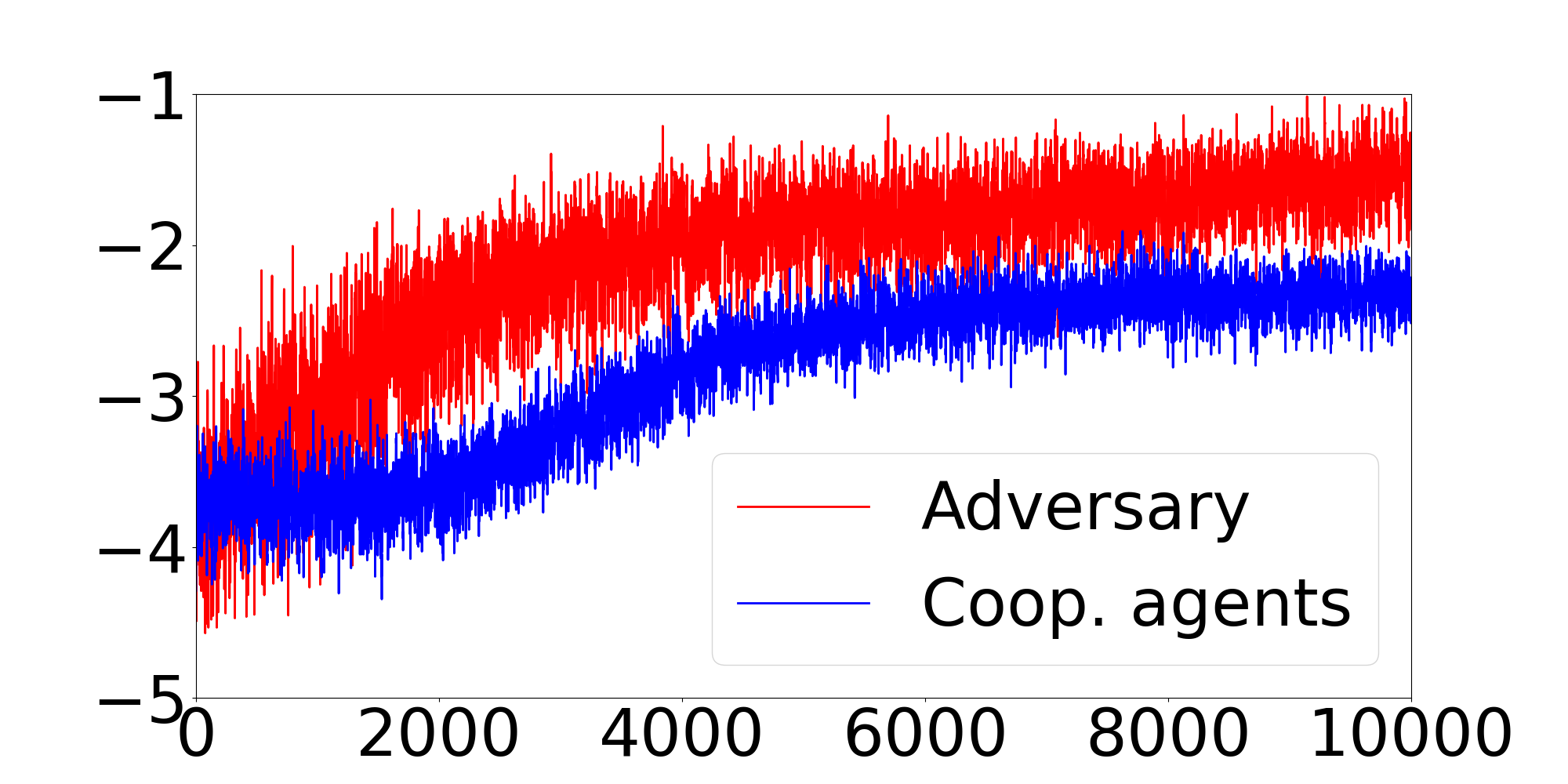}
  \caption{Greedy, $H=1$.}
\end{subfigure}

\begin{subfigure}{.45\textwidth}
\centering
\includegraphics[width=1\linewidth]{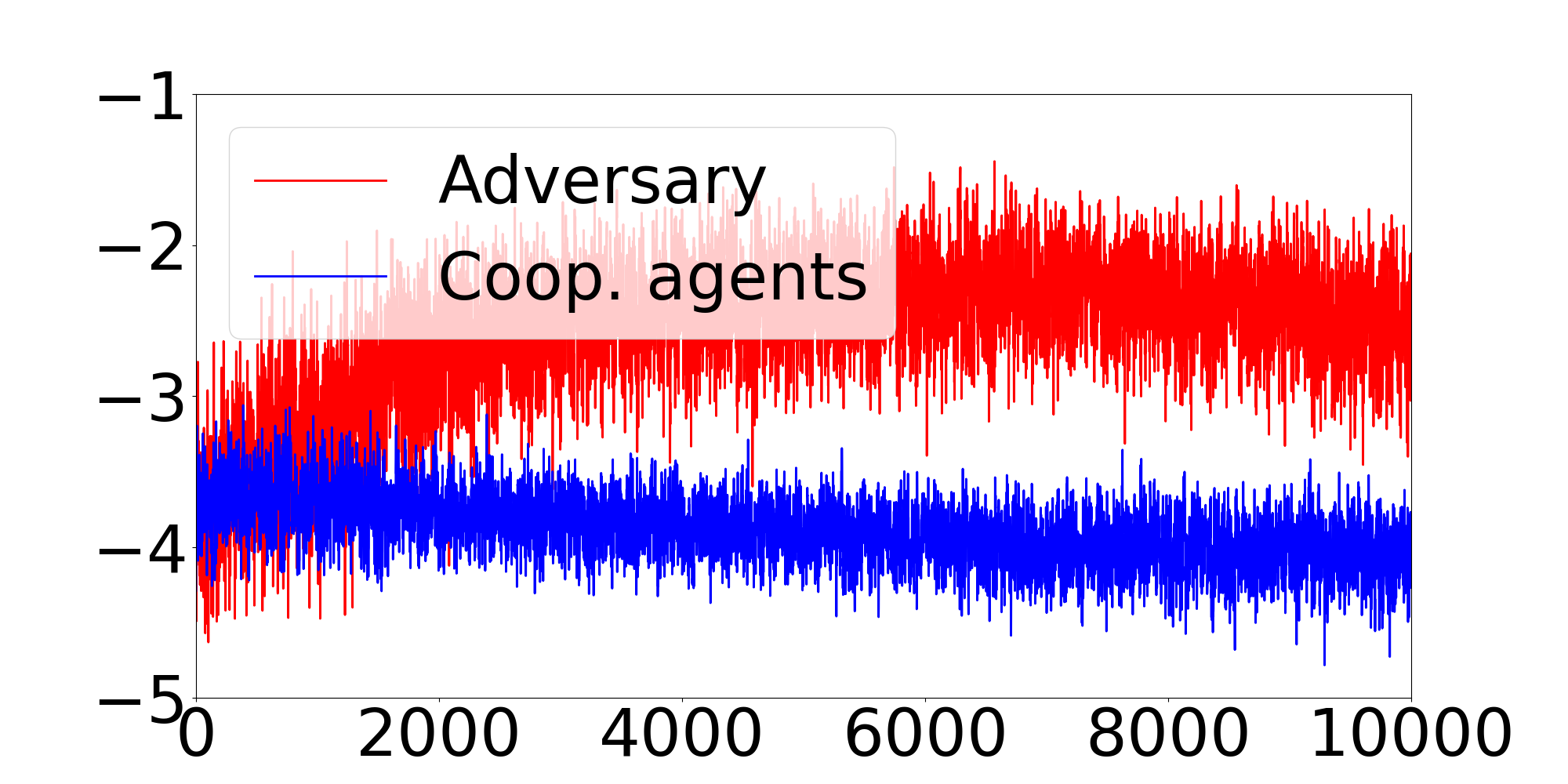}
  \caption{Faulty, $H=0$.}
\end{subfigure}
\begin{subfigure}{.45\textwidth}
\centering
\includegraphics[width=1\linewidth]{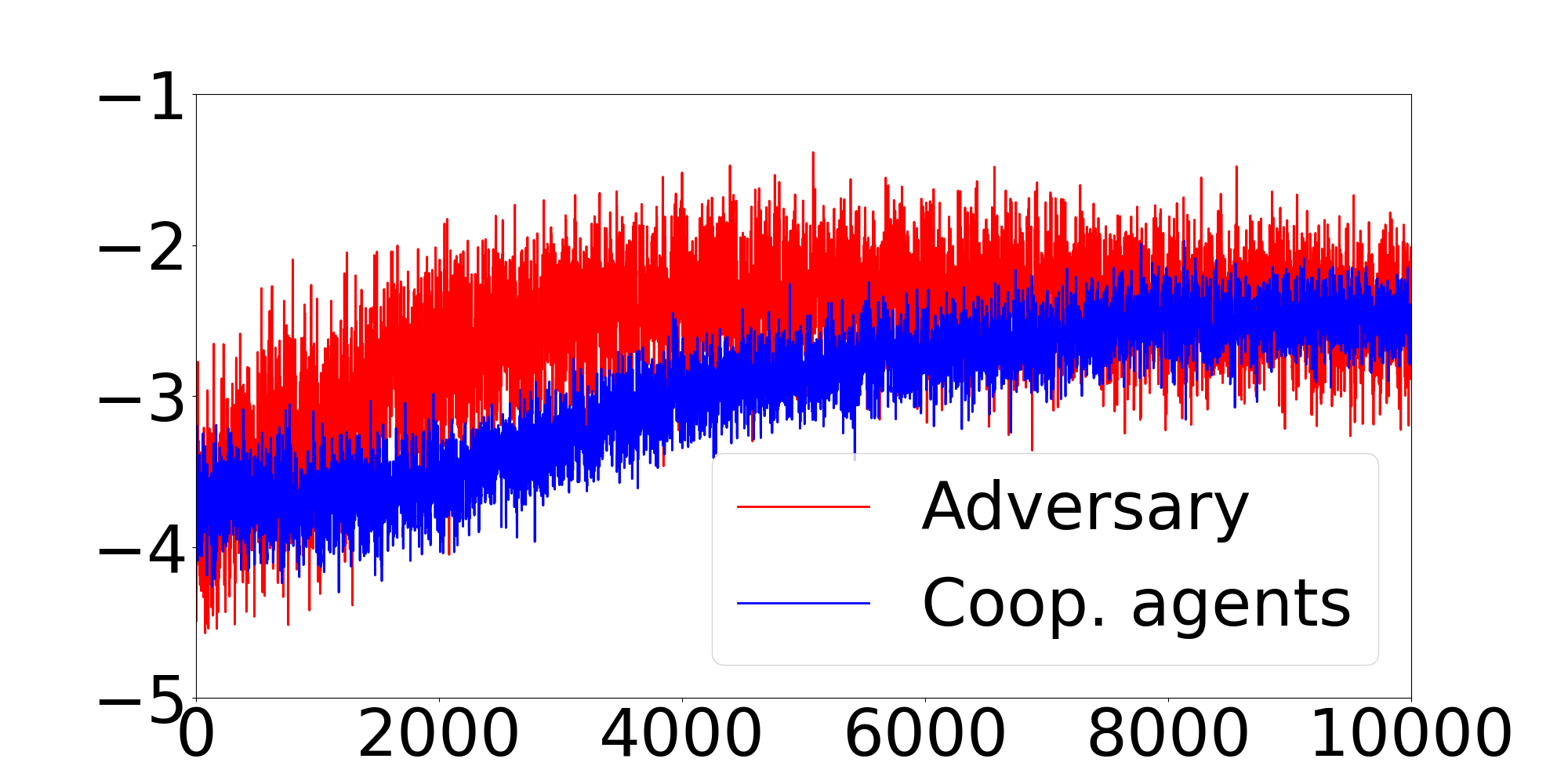}
  \caption{Faulty, $H=1$.}
\end{subfigure}

\begin{subfigure}{.45\textwidth}
\centering
\includegraphics[width=1\linewidth]{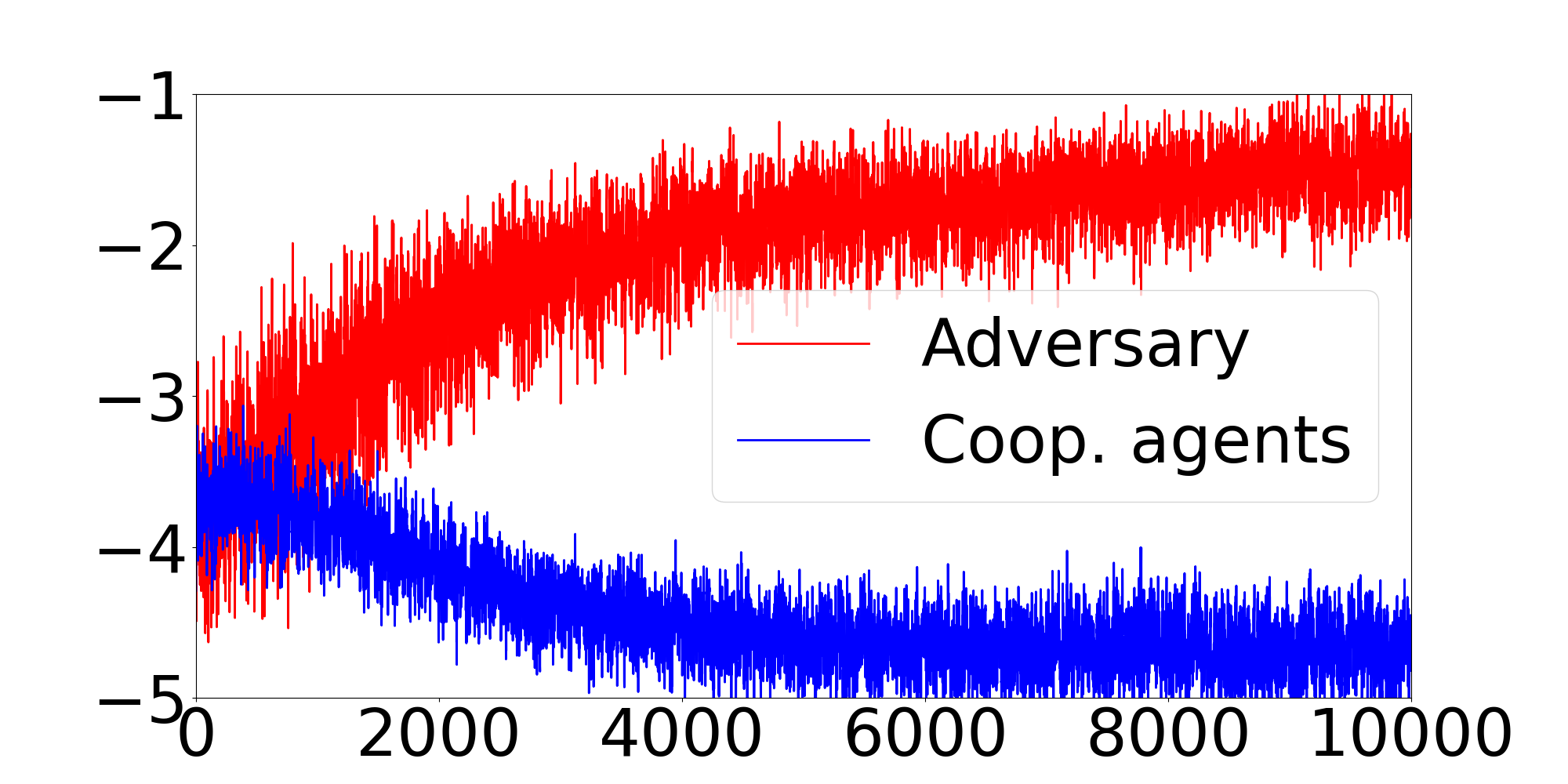}
  \caption{Strategic, $H=0$.}
\end{subfigure}
\begin{subfigure}{.45\textwidth}
\centering
\includegraphics[width=1\linewidth]{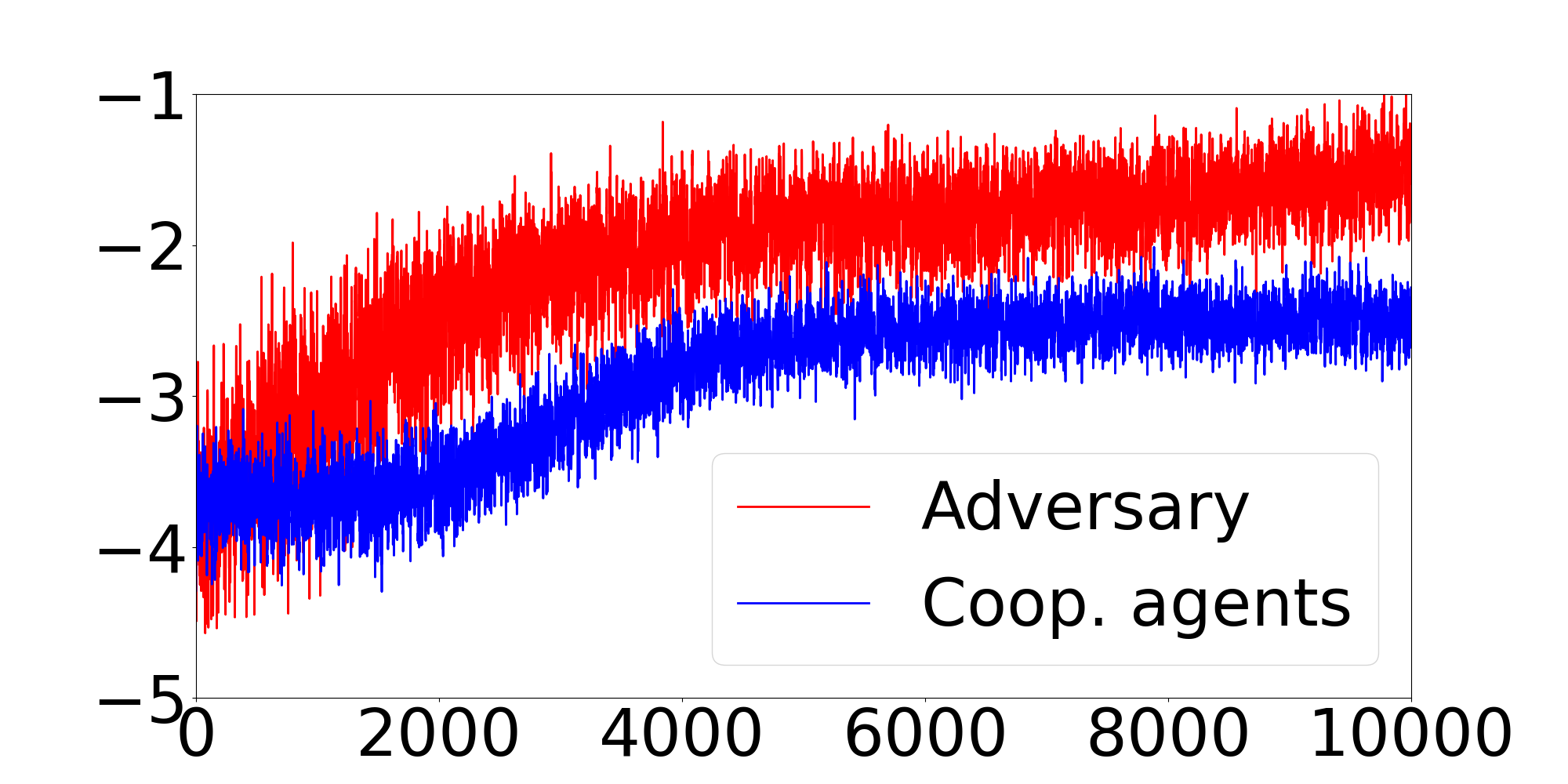}
  \caption{Strategic, $H=1$.}
\end{subfigure}
\caption{Episode returns of the adversary and the team of cooperative agents that employ Algorithm~\ref{alg:3} in four scenarios.}
\label{fig:sim_results}
\end{figure}

In the first scenario, we assume that all four agents are cooperative. Their performance is nearly identical for $H=0$ and $H=1$ as they learn near-optimal policies. In the second scenario, we assume that one agent is greedy. It participates in the team-average estimation in the sense that it sends its parameter values to the other agents but does not receive any values from them. It updates its actor network with the true local rewards as opposed to the cooperative agents. In this case, the cooperative agents end up maximizing the greedy agent's objective when $H=0$, and hence perform poorly with respect to the team-average objective. However, the cooperative agents are resilient to this attack when $H=1$ and successfully maximize the team-average objective of the cooperative network. In the third scenario, we assume that one agent is faulty. The agent sends constant parameter values to the other agents. Furthermore, it does not update the parameters internally. However, it updates its policy using the local rewards and faulty critic approximation. The cooperative agents perform poorly when $H=0$ but escape the attack with ease when $H=1$. In the last scenario, we assume that one agent is strategic in the sense that it has global knowledge of the problem and attempts to minimize the cooperative agents' objective and maximize its own objective. The adversarial agent is similar to the greedy agent in the second scenario but it trains an additional critic network for its own learning purpose. Specifically, it employs a local critic to evaluate its own performance and uses it along with local rewards to update its actor network. The team critic and the team-average reward function are trained based on the mean negative rewards of the cooperative agents. These approximations are designed to compromise the cooperative agents' performance. The adversarial agent indeed drives the cooperative agents to minimize their expected returns when $H=0$ but they manage to withstand this attack when $H=1$ and learn near-optimal policies.\par

\section{Conclusion}
In this work, we introduced novel resilient projection-based consensus actor-critic MARL algorithms with linear and nonlinear function approximation. These algorithms ensure Byzantine-resilient learning of cooperative agents in environments that are influenced by the policies of Byzantine agents. The algorithms significantly reduce the impact of Byzantine attacks on communication channels in training, and hence allow the cooperative agents to learn policies that maximize their team-average objective function. We provided a convergence analysis of the algorithm that uses linear approximation, where we proved that under a diminishing step size the policies of cooperative agents converge to a neighborhood of a local maximizer of the cooperative team-average objective function. In the simulations, we implemented the resilient algorithm that employs nonlinear approximation and demonstrated its functionality in four scenarios that included all-cooperative learning and learning with a greedy, faulty, and strategic adversarial agent present in the network, respectively. Moreover, we presented a simple motivating example that shows that the resilient projection-based method consensus method does not suffer overestimation in the function approximation unlike the method of trimmed means. In the future work, we would like to explore whether this promising projection-based aggregation method is advantageous in distributed machine learning as well.
\acks{
We thank Krishna Chaitanya Kosaraju for valuable discussion about the code implementation.}

\appendix
\section{Proofs of Theoretical Results in Section~\ref{sec:convergence}}
In this section, we provide mathematical proofs for the theoretical results stated in Section~\ref{sec:convergence}. In the analysis, we repeatedly use shorthand $\pi_t=\pi(a|s;\theta_t)$ to describe the dependence of the policy on the time-varying parameters $\theta_t$.
\subsection{Algorithm~\ref{alg:1}}

\begin{proof}(\textbf{Lemma~\ref{lem:spectral_radius}})
We consider an arbitrary vector $x$ and $x^T\mathbb{E}(C_t^T(I-\mathbf{11}^T/N)C_t|\mathcal{F}_t^x)x=\mathbb{E}(\Vert(I-\mathbf{11}^T/N)C_tx\Vert^2|\mathcal{F}_t^x)$. We provide an eigendecomposition of the term $(I-\mathbf{11}^T/N)C_tx$. We note that the matrices $(I-\mathbf{11}^T/N)$ and $C_t$ share the  eigenvector $\mathbf{1}$ and the product of the associated eigenvalues is zero. For any other vector that satisfies $x\perp\mathbf{1}$, we obtain $\Vert (I-\mathbf{11}^T/N)C_tx\Vert^2=\Vert C_tx\Vert^2\leq\Vert x\Vert^2$. By Assumption~\ref{as:consensus_matrix}, the mean consensus matrix $\mathbb{E}(C_t|\mathcal{F}_t^x)$ has all eigenvalues strictly less than one except for the eigenvalue one associated with eigenvector $\mathbf{1}$. Hence, we obtain $\mathbb{E}\big(\Vert (I-\mathbf{11}^T/N)C_tx\Vert^2|\mathcal{F}_t^x\big)\leq\rho_t\Vert x\Vert^2$, for some $\rho_t<1$.
\end{proof}

\begin{proof}(\textbf{Lemma~\ref{lem:v_consensus}})
The updates of $v_{\perp,t}$ can be compactly written as follows
\begin{align}
v_{\perp,t+1}&=\mathcal{J}_\perp\big[v_t+(C_t\otimes\Gamma_{v,t})\big(v_t+\alpha_{v,t}(A_{v,t}v_t+b_{v,t})\big)-(I\otimes\Gamma_{v,t})v_t\big]\nonumber\\
&=\mathcal{J}_\perp\big[(I\otimes\hat\Gamma_{v,t})v_t+(C_t\otimes\Gamma_{v,t})\big(v_t+\alpha_{v,t}(A_{v,t}v_t+b_{v,t})\big)\big]\nonumber\\
&=(I\otimes\hat\Gamma_{v,t})v_{\perp,t}+\mathcal{J}_\perp\big[(C_t\otimes\Gamma_{v,t})\big(v_{\perp,t}+\alpha_{v,t}(A_{v,t}v_{\perp,t}+b_{v,t})\big)\big]\label{v_perp_update},
\end{align}
where we used the fact that $(C_t\otimes \Gamma_{v,t})(\mathbf{1}\otimes\left<x\right>)=\mathbf{1}\otimes(\Gamma_{v,t}\left<x\right>)$ and $\mathcal{J}_\perp\big[\mathbf{1}\otimes(\Gamma_{v,t}\left<x\right>)\big]=0$. The updates of $v_{\perp,t}$ can be viewed in two mutually orthogonal subspaces as follows $v_{\perp,t+1}=(I\otimes\hat\Gamma_{v,t})v_{\perp,t+1}+(I\otimes\Gamma_{v,t})v_{\perp,t+1}$, where
\begin{align}
(I\otimes\hat\Gamma_{v,t})v_{\perp,t+1}&=(I\otimes\hat\Gamma_{v,t})v_{\perp,t}\label{v_perp_ort}\\
(I\otimes\Gamma_{v,t})v_{\perp,t+1}&=\big[(I-\mathbf{1}\mathbf{1}^T/N)C_t\otimes I\big](I\otimes\Gamma_{v,t})\big(v_{\perp,t}+\alpha_{v,t}(A_{v,t}v_t+b_{v,t})\big)\label{v_perp_grad}
\end{align}
We consider the expected values $\mathbb{E}(\Vert (I\otimes\hat\Gamma_{v,t})v_{\perp,t+1}\Vert^2|\mathcal{F}_t^v)$ and $\mathbb{E}(\Vert (I\otimes\Gamma_{v,t})v_{\perp,t+1}\Vert^2|\mathcal{F}_t^v)$. By \eqref{v_perp_ort}, the first term satisfies $\mathbb{E}(\Vert (I\otimes\hat\Gamma_{v,t})v_{\perp,t+1}\Vert^2|\mathcal{F}_{v,t}^v)=\mathbb{E}(\Vert (I\otimes\hat\Gamma_{v,t})v_{\perp,t}\Vert^2|\mathcal{F}_t^v)$, and thus the disagreement vector $v_{\perp,t}$ remains stable in the orthogonal subspace. To analyze stability of the term in \eqref{v_perp_grad}, we write
\begin{align}
\mathbb{E}(\Vert &(I\otimes\Gamma_{v,t})v_{\perp,t+1}\Vert^2|\mathcal{F}_t^v)\nonumber\\
&=\mathbb{E}\big(\big\Vert\big[(I-\mathbf{1}\mathbf{1}^T/N)C_t\otimes I\big](I\otimes\Gamma_{v,t})\big((I+\alpha_{v,t}A_{v,t})v_{\perp,t}+\alpha_{v,t}b_{v,t})\big)\big\Vert^2\big|\mathcal{F}_t^v\big)\nonumber\\
&\leq\rho_t\cdot\mathbb{E}\big(\big\Vert(I\otimes\Gamma_{v,t})\big((I+\alpha_{v,t}A_{v,t})v_{\perp,t}+\alpha_{v,t}b_{v,t})\big)\big\Vert^2\big|\mathcal{F}_t^v\big)\label{v_perp}
\end{align}
where $\rho_t$ is a spectral radius of $\mathbb{E}(C_t(I-\mathbf{11}^T/N)C_t|\mathcal{F}_t^v)$. We note that the inequality holds by the conditional independence of $C_t$ stated in Assumption~\ref{as:consensus_matrix}. Regarding the terms that involve $v_{\perp,t}$, we have $I+\alpha_{v,t}A_{v,t}=I\otimes(I+\alpha_{v,t}A_{v,t}^\prime)$, and hence we obtain $(I\otimes\Gamma_{v,t})(I+\alpha_{v,t}A_{v,t})=I\otimes (\Gamma_{v,t}+\alpha_{v,t}A_{v,t}^\prime)$. The eigenspace of matrix $\Gamma_{v,t}+\alpha_{v,t}A_{v,t}^\prime$ is spanned by vector $\Gamma_{v,t}\phi_t$. Hence, for $\nu_t$ that satisfies $\nu_t\leq (1+\alpha_{v,t}K_1)^2$, where $K_1=\sup_t\Vert(\gamma\phi_{t+1}-\phi_t)^T\phi_t\Vert<\infty$ by Assumption~\ref{as:linear_approx}, we have $\Vert (I\otimes \Gamma_{v,t})(I+\alpha_{v,t}A_{v,t})v_{\perp,t}\Vert^2\leq\nu_t\cdot\Vert(I\otimes\Gamma_{v,t})v_{\perp,t}\Vert^2$. We apply this inequality in the following lines, where we first premultiply both sides of \eqref{v_perp} by $\alpha_{v,t+1}^{-2}$ and apply the triangle inequality. Letting $\eta_{t+1}=\mathbb{E}(\Vert (I\otimes\Gamma_{v,t})\alpha_{v,t+1}^{-1}v_{\perp,t+1}\Vert^2|\mathcal{F}_t^v)$, we obtain the following inequality
\begin{align*}
\eta_{t+1}\leq\rho_t\cdot\nu_t\cdot\frac{\alpha_{v,t}^2}{\alpha_{v,t+1}^2}\big(\eta_t+2\nu_t^{-\frac{1}{2}}\eta_t^{\frac{1}{2}}\cdot\mathbb{E}(\Vert b_{v,t}\Vert|\mathcal{F}_t^v)+\nu_t^{-1}\cdot\mathbb{E}(\Vert b_{v,t}\Vert^2|\mathcal{F}_t^v)\big).
\end{align*}
Under Assumption \ref{as:step_size}, $\lim_{t\rightarrow\infty}\frac{\alpha_{v,t}^2}{\alpha_{v,t+1}^2}=1$ and $\lim_{t\rightarrow\infty}\nu_t=1$, and so there exists finite time $t_0$ and constant $\delta>0$ such that $\nu_t>0$ and $\rho_t\cdot\nu_t\cdot\frac{\alpha_{v,t}^2}{\alpha_{v,t+1}^2}\leq1-\delta$ for all $t>t_0$. Since $b_{v,t}$ is uniformly bounded by Assumption~\ref{as:linear_approx} and \ref{as:reward_bound}, we have $\nu_t^{-1}\cdot\mathbb{E}\big(\Vert b_{v,t}\Vert^2\big|\mathcal{F}_t^v\big)\leq K_2$ for some $K_2<\infty$. Therefore, for $t>t_0$ we can write
\begin{align*}
\eta_{t+1}&\leq(1-\delta)\big(\eta_t+2\sqrt{\eta_t}\cdot\sqrt{K_2}+K_2\big)\\
&=\big(1-\frac{\delta}{2}\big)\eta_t-\frac{\delta}{2}\big(\sqrt{\eta_t}-\frac{2}{\delta}(1-\delta)\sqrt{K_2}\big)^2+\frac{2}{\delta}(1-\delta)^2K_2+(1-\delta)K_2\\
&\leq\big(1-\delta/2\big)\eta_t+K_3,
\end{align*}
where $K_3=\frac{2}{\delta}(1-\delta)^2K_2+(1-\delta)K_2$. By induction, it follows that
$\eta_{t}\leq\big(1-\frac{\delta}{2}\big)^{t-t_0}\eta_{t_0}+\frac{2K_3}{\delta}$. Therefore, we have $\sup_t\mathbb{E}(\Vert(I\otimes\Gamma_{v,t})\alpha_{v,t}^{-1}v_{\perp,t}\Vert^2|\mathcal{F}_t^v)<K_4$ for some $K_4>0$. We note that the states are visited according to the stationary distribution $d_\pi(s)$, and hence the uniform bound holds for all $\Gamma_{v,t}$ visited in the infinite sequence. Therefore, we consider $\sum_t\mathbb{E}\big(\Vert(I\otimes\Gamma_{v,t})v_{\perp,t}\Vert^2|\mathcal{F}_t^v\big)<K_4\cdot\sum_t\alpha_{v,t}^2$ and obtain  $\lim_{t\rightarrow\infty}(I\otimes\Gamma_{v,t})v_{\perp,t}=0$ with probability one by Assumption~\ref{as:step_size}. This implies that $\lim_{t\rightarrow\infty}v_{\perp,t}=0$ with probability one.
\end{proof}

\begin{proof}(\textbf{Lemma~\ref{lem:lam_consensus}})
The proof is analogous to the proof of Lemma~\ref{lem:v_consensus}.
\end{proof}

\begin{proof}(\textbf{Theorem~\ref{thm:v_team_convergence}})
We write the iteration of $\left<v_t\right>$ as follows
\begin{align}
\left<v_{t+1}\right>&=\left<(I\otimes\hat\Gamma_{v,t})v_t+(C_t\otimes\Gamma_{v,t})\big(v_t+\alpha_{v,t}(A_{v,t}v_t+b_{v,t})\big)\right>\\
&=\left<(I\otimes\hat\Gamma_{v,t})v_t\right>\nonumber\\
&\qquad +\left<(C_t\otimes\Gamma_{v,t})\big(\mathbf{1}\otimes\left<v_t\right>+v_{\perp,t}+\alpha_{v,t}\big[(I\otimes A_{v,t}^\prime)(\mathbf{1}\otimes\left<v_t\right>+v_{\perp,t})+b_{v,t}\big]\big)\right>\nonumber\\
&=\Gamma_{v,t}\left<v_t\right>+\hat\Gamma_{v,t}\left<v_t\right>+\alpha_{v,t}\left<(C_t\otimes\Gamma_{v,t})(I\otimes A_{v,t}^\prime)(\mathbf{1}\otimes\left<v_t\right>)\right>\nonumber\\
&\quad+\alpha_{v,t}\left<(C_t\otimes\Gamma_{v,t})\big(\alpha_{v,t}^{-1}v_{\perp,t}+(I\otimes A_{v,t}^\prime)v_{\perp,t}+b_{v,t}\big)\right>\nonumber\\
&=\left<v_t\right>+\alpha_{v,t}A_{v,t}^\prime\left<v_t\right>+\alpha_{v,t}\left<(C_t\otimes\Gamma_{v,t})\big(\alpha_{v,t}^{-1}v_{\perp,t}+A_{v,t}v_{\perp,t}+b_{v,t}\big)\right>\nonumber\\
&=\left<v_t\right>+\alpha_{v,t}\big[g_t\big(\left<v_t\right>,\xi_t\big)+\delta M_t+\beta_t\big],
\end{align}
where the functions $g_t(\cdot,\cdot)$, $\delta M_t$, and $\beta_t$ are given as
\begin{align}
g_t\big(\left<v_t\right>,\xi_t\big)&=\mathbb{E}\big(A_{v,t}^\prime\left<v_t\right>+\left<(C_t\otimes \Gamma_{v,t})b_{v,t}\right>|\mathcal{F}_t^v\big)\\
\delta M_t&=A_{v,t}^\prime\left<v_t\right>+\left<(C_t\otimes I)b_{v,t}\right>+\left<(C_t\otimes\Gamma_{v,t})\big(\alpha_{v,t}^{-1}v_{\perp,t}+A_{v,t}v_{\perp,t}\big)\right>\\
&\quad-\mathbb{E}\big(A_{v,t}^\prime\left<v_t\right>+\left<(C_t\otimes \Gamma_{v,t})b_{v,t}\right>+\left<(C_t\otimes\Gamma_{v,t})\big(\alpha_{v,t}^{-1}v_{\perp,t}+A_{v,t}v_{\perp,t}\big)\right>|\mathcal{F}_t^v\big)\nonumber\\
\beta_t&=\mathbb{E}\big(\left<(C_t\otimes\Gamma_{v,t})(\alpha_{v,t}^{-1}v_{\perp,t}+A_{v,t}v_{\perp,t})\right>\big|\mathcal{F}_t^v\big).
\end{align}

To finalize the proof, we need to verify the conditions in Section~\ref{ap:unc}.
\begin{enumerate}
	
	\item We have $\Vert g_t\big(\left<x\right>,\xi_t\big)-g_t\big(\left<y\right>,\xi_t\big)\Vert=\Vert\mathbb{E}\big(A_{v,t}(\left<x\right>-\left<y\right>)|\mathcal{F}_t^v\big)\Vert\leq K_1\cdot\Vert \left<x\right>-\left<y\right>\Vert^2$ for some $K_1>0$ since $A_{v,t}^\prime$ is uniformly bounded by Assumption~\ref{as:linear_approx}. Therefore, $g_t\big(\left<v_t\right>,\xi_t\big)$ is Lipschitz continuous in $\left<v_t\right>$.

	\item The step size sequence $\{\alpha_{v,t}\}_{t\geq 0}$ satisfies $\sum_t\alpha_{v,t}=\infty$ and $\sum_t\alpha_{v,t}^2<\infty$, for $t\geq 0$.
	
	\item The martingale difference sequence $\delta M_t$ satisfies $\mathbb{E}\big(\Vert\delta M_t\Vert^2|\mathcal{F}_t^v\big)\leq K_2\cdot(1+\Vert \left<v_t\right>\Vert^2)$, since $A_{v,t}$, $b_{v,t}$, and $\alpha_{v,t}^{-1}v_{\perp,t}$ are uniformly bounded by Assumption~\ref{as:linear_approx} and \ref{as:reward_bound}, and Lemma~\ref{lem:v_consensus}.

	\item By Assumption~\ref{as:consensus_matrix}, we can write	$\beta_t=\mathbb{E}\big(\left<(C_t\otimes\Gamma_{v,t})(\alpha_{v,t}^{-1}v_{\perp,t}+A_{v,t}v_{\perp,t})\right>\big|\mathcal{F}_t^v\big)=\frac{1}{N}(\mathbf{1}^T\otimes I)\mathbb{E}\big(C_t\otimes I\big|\mathcal{F}_t^v\big)\mathbb{E}\big((I\otimes\Gamma_{v,t})(\alpha_{v,t}^{-1}v_{\perp,t}+A_{v,t}v_{\perp,t})\big|\mathcal{F}_t^v\big)=\frac{1}{N}(\mathbf{1}^T\otimes I)\mathbb{E}\big((I\otimes\Gamma_{v,t})(\alpha_{v,t}^{-1}v_{\perp,t}+A_{v,t}v_{\perp,t})\big|\mathcal{F}_t^v\big)$. The last term is uniformly bounded by Lemma~\ref{lem:v_consensus} and Assumption~\ref{as:linear_approx}. This and the fact that $(\mathbf{1}^T\otimes I)(I\otimes\Gamma_{v,t})(\alpha_{v,t}^{-1}v_{\perp,t}+A_{v,t}v_{\perp,t})=0$ imply that $\beta_t=0$.

	\item The Markov chain $\{\xi_t\}_{t\geq 0}$ is irreducible and has a stationary distribution $\eta$.
		
\end{enumerate}

Applying Theorem~\ref{thm:ODE_unc1}, it follows that the asymptotic behavior is described by the ODE $$\left<\dot{v}\right>=\bar{g}\big(\left<v\right>\big)=\mathbb{E}_{d_\pi}\big[g_t(\left<v_t\right>,\xi_t)\big]=\Phi^TD_\pi^s(\gamma P_\pi-I)\Phi\left<v\right>+\frac{1}{N}\sum_{i\in\mathcal{N}}\Phi^TD_\pi^sR_\pi^i.$$
We let $\lim_{c\rightarrow\infty}\bar{g}(cx)\cdot c^{-1}=g_\infty(x)=\Phi^TD_\pi^s(\gamma P_\pi-I)\Phi x$. Since $\Phi$ is full column rank by Assumption~\ref{as:linear_approx}, we let $\zeta$ and $\Phi y$ denote an arbitrary eigenvalue-eigenvector pair of the matrix product $D_\pi(\gamma P_\pi-I)$. We obtain $D_\pi(\gamma P_\pi-I)\Phi y=\zeta \Phi y$, which can be further manipulated to yield $y^T\Phi^T(\gamma P_\pi-I)D_\pi(\gamma P_\pi-I)\Phi y=\zeta y^T\Phi^T(\gamma P_\pi-I)\Phi y$, which implies that $$\zeta=\frac{y^T\Phi^T(\gamma P_\pi-I)^TD_\pi(\gamma P_\pi-I)\Phi y}{y^T\Phi^T(\gamma P_\pi-I)\Phi y}<0 \text{ with probability one},$$ since the numerator is positive definite and the denominator is negative definite with probability one. Therefore, the system $\dot{x}=g_\infty(x)$ has a unique globally asymptotically stable (GAS) equilibrium, and we can directly apply Theorem~\ref{thm:ODE_unc2} to obtain the desired boundedness of the iterates, i.e., $\sup_t\Vert v_t\Vert<\infty$ with probability one. Finally, we apply Theorem~\ref{thm:ODE_unc3} to establish convergence with probability one to the GAS of the ODE $\left<\dot{v}\right>=\bar{g}\big(\left<v\right>\big)=\Phi^TD_\pi^s(\gamma P_\pi-I)\Phi\left<v\right>+\frac{1}{N}\sum_{i\in\mathcal{N}}\Phi^TD_\pi^s R_\pi^i$. 

\end{proof}

\begin{proof}(\textbf{Theorem~\ref{thm:lam_team_convergence}})
The proof is analogous to the proof of Theorem~\ref{thm:v_team_convergence}.
\end{proof}

\begin{proof}(\textbf{Theorem~\ref{thm:actor}})
We define a filtration $\mathcal{F}_t^\theta=\sigma(\theta_0,Y_{\tau-1},\tau\leq t)$, where $Y_t$ are the actor updates. The recursion of agent~$i$, $i\in\mathcal{N}$, is given as
\begin{align}
\theta_{t+1}^i&=\Psi_\Theta^i\big(\theta_t^i+\alpha_{\theta,t}\delta_t^i\psi_t^i\big)\\
&=\Psi_\Theta^i\big(\theta_t^i+\alpha_{\theta,t}\big[g_t^i(\theta_t^i)+\delta M_t+\beta_t\big]\big),
\end{align}
where $\delta_t^i=f_t^T\lambda_t^i+\gamma\phi_{t+1}^Tv_t^i-\phi_t^Tv_t^i$, $\psi_t^i=\nabla_{\theta^i}\log\pi^i(a_t^i|s_t^i;\theta_t^i)$, and the functions $g_t(\cdot)$, $\delta M_t$, $\beta_t$ are given as
\begin{align}
g_t^i(\theta_t^i)&=\mathbb{E}_{\pi_t,d_{\pi_t},p}\big(\delta_{t,\pi_t}\psi_t^i|\mathcal{F}_t^\theta\big)\\
\delta M_t&=\delta_t^i\psi_t^i-\mathbb{E}_{\pi_t,d_{\pi_t},p}\big(\delta_t^i\psi_t^i|\mathcal{F}_t^\theta\big)\\
\beta_t&=\mathbb{E}_{\pi_t,d_{\pi_t},p}\big((\delta_{t}^i-\delta_{t,\pi_t})\psi_t^i|\mathcal{F}_t^\theta\big).
\end{align}
The signal $\delta_{t,\pi_t}$ is the approximated team-average TD error upon convergence of the parameters $v_t$ and $\lambda_t$ under the current network policy $\pi(a|s;\theta_t)$, i.e., $\delta_{t,\pi_t}=f_t^T\lambda_{\pi_t}+\gamma\phi_{t+1}v_{\pi_t}-\phi_t^Tv_{\pi_t}$. To complete the convergence proof, we verify conditions given in Appendix~\ref{ap:con}.
\begin{enumerate}

	\item The function $\delta_t^i$ is bounded by Assumption~\ref{as:linear_approx}, and Theorem~\ref{thm:v_team_convergence} and \ref{thm:lam_team_convergence}. The function $\psi_t^i$ is bounded by Assumption~\ref{as:policy_updates}. Therefore, we obtain $\sup_t\mathbb{E}\big(\Vert\delta_t^i\psi_t^i\Vert\big|\mathcal{F}_t^\theta\big)<\infty$.

	\item The step size sequence $\alpha_{\theta,t}$ satisfies $\sum_t\alpha_{\theta,t}^2<\infty$ and $\lim_t\frac{\alpha_{\theta,t+1}}{\alpha_{\theta,t}}=1$.

	\item The bias term satisfies $\beta_t\rightarrow 0$ with probability one since $v_t\rightarrow v_{\pi_t}$ and $\lambda_t\rightarrow \lambda_{\pi_t}$ on the faster time scale by Assumption~\ref{as:step_size}.
	
	\item The admissible set $\Theta$ is a hyperrectangle by Assumption~\ref{as:policy_updates}.		
	
	\item The function $g_t^i(\cdot)$ is continuous in $\theta_t^i$ uniformly in $t$. Furthermore, $g_t^i(\cdot):=\bar{g}^i(\cdot)$ since it is independent of $t$.	

	\end{enumerate}
Applying Theorem~\ref{thm:ODE_con1}, the asymptotic behavior of the actor updates is given by the ODE $\dot\theta^i=\Psi_{\Theta^i}\big[\bar{g}^i(\theta^i)\big]$. We define the approximated team-average objective function as $\tilde{J}(\theta)=\mathbb{E}_{\pi,d_{\pi},p}\big[\bar{r}(s,a;\lambda_\pi)+\gamma V(s^\prime;v_\pi)\big]$. It can be easily verified that $v_\pi$ and $\lambda_\pi$ are continuously differentiable in $\theta^i$. Furthermore, Assumption~\ref{as:policy_coop} ensures differentiability of $\nabla_{\theta^i}\log\pi^i(a_t^i|s_t;\theta)$. Therefore, $\tilde{J}(\theta)$ is continuously differentiable in $\theta^i$. The local AC policy gradient is given as $\nabla_{\theta^i}\tilde{J}(\theta)=\mathbb{E}_{\pi,d_\pi,p}\big[\big(\bar{r}(s,a;\lambda_\pi)+\gamma V(s^\prime;v_\pi)-V(s;v_\pi)\big)\nabla_{\theta^i}\log\pi^i(a^i|s;\theta^i)\big]$. We note that  $\bar{g}^i(\theta^i)=\nabla_{\theta^i}\tilde{J}(\theta)$ for $i\in\mathcal{N}$. The rate of change of $\tilde{J}(\theta)$ is given as $\dot{\tilde{J}}(\theta)=\nabla_\theta \tilde{J}(\theta)^T\big(\nabla_\theta \tilde{J}(\theta)+z\big)$, where $z$ is the reflection term that projects the actor parameters back into the admissible set $\Theta$, i.e., $z=-\nabla_\theta \tilde{J}(\theta)$ whenever a constraint is active and $z=0$ otherwise (elementwise). Therefore, we obtain $\dot{\tilde{J}}(\theta)>0$ if $\nabla_\theta \tilde{J}(\theta)+z\neq0$ and $\dot{\tilde{J}}(\theta)=0$ otherwise. By Theorem~\ref{thm:ODE_con2}, the solution of the ODE $\dot\theta=\Psi_\Theta\big[\bar{g}(\theta)\big]=\begin{bmatrix} \Psi_\Theta^1\big[\bar{g}^1(\theta^1)\big]^T & \dots & \Psi_\Theta^N\big[\bar{g}^N(\theta^N)\big]^T\end{bmatrix}^T$ converges to a set of stationary points $\nabla_\theta \tilde{J}(\theta)+z=0$ that correspond to the local maxima of the approximated objective function $\tilde{J}(\theta)$.
\end{proof}

\subsection{Algorithm~\ref{alg:2}}

\begin{proof}(\textbf{Lemma~\ref{lem:spectral_radius2}})
The proof is analogous to the proof of Lemma~\ref{lem:spectral_radius}. The difference here is that we assume connectivity of the consensus matrix $C_{x,t}^+$ for $t>0$ instead of the mean connectivity and that the positive lower-bound of the consensus weights is implied to Lemma~\ref{lem:equivalent_consensus}. Using the same reasoning about the eigenvalues, we obtain  $\Vert (I-\mathbf{11}^T/N^+)C_{x,t}^+x\Vert^2\leq\rho_t^+\Vert x\Vert^2$ for all $x$ and some $\rho_t^+<1$.
\end{proof}

\begin{proof}(\textbf{Lemma~\ref{lem:v_consensus2}})
We write the updates of $v_{\perp,t}^+$ in the same form as in \eqref{v_perp_update}. The updates are given as follows
\begin{align*}
v_{\perp,t+1}^+&=(I\otimes\hat\Gamma_{v,t})v_{\perp,t}^++\mathcal{J}_\perp\big[(C_{v,t}^+\otimes\Gamma_{v,t})\big(v_{\perp,t}^++\alpha_{v,t}(A_{v,t}^+v_{\perp,t}^++b_{v,t}^+)\big)\big].
\end{align*}
Spliting the updates into two mutually orthogonal subspaces, we obtain
\begin{align*}
(I\otimes\hat\Gamma_{v,t})v_{\perp,t+1}^+&=(I\otimes\hat\Gamma_{v,t})v_{\perp,t}^+\\
(I\otimes\Gamma_{v,t})v_{\perp,t+1}^+&=\big[(I-\mathbf{1}\mathbf{1}^T/N^+)C_{v,t}^+\otimes I\big](I\otimes\Gamma_{v,t})\big(v_{\perp,t}^++\alpha_{v,t}(A_{v,t}^+v_{\perp,t}^++b_{v,t}^+)\big)
\end{align*}
The first term equation implies $\Vert(I\otimes\hat\Gamma_{v,t})v_{\perp,t+1}^+\Vert^2=\Vert(I\otimes\hat\Gamma_{v,t})v_{\perp,t}^+\Vert^2$ and for the second equation we write
\begin{align*}
\Vert (I\otimes\Gamma_{v,t})v_{\perp,t+1}^+\Vert^2&=
\big\Vert\big[(I-\mathbf{1}\mathbf{1}^T/N^+)C_{v,t}^+\otimes I\big](I\otimes\Gamma_{v,t})\big(v_{\perp,t}^++\alpha_{v,t}(A_{v,t}^+v_{\perp,t}^++b_{v,t}^+)\big)\big\Vert^2\\
&\leq\rho_t^+\cdot\big\Vert(I\otimes\Gamma_{v,t})\big((I+\alpha_{v,t}A_{v,t}^+)v_{\perp,t}^++\alpha_{v,t}b_{v,t}^+)\big)\big\Vert^2,
\end{align*}
where $\rho_t^+<1$ by Lemma~\ref{lem:spectral_radius2}. Following the proof of Lemma~\ref{lem:v_consensus}, we obtain $\sup_t\Vert(I\otimes\Gamma_{v,t})\alpha_{v,t}^{-1}v_{\perp,t}^+\Vert^2<\infty$ and $\lim\limits_{t\rightarrow\infty}v_{\perp,t}^+=0$ with probability one.
\end{proof}
\begin{proof}(\textbf{Lemma~\ref{lem:lam_consensus2}})
The proof is analogous to the proof of Lemma~\ref{lem:v_consensus2}.
\end{proof}

\begin{proof}(\textbf{Theorem~\ref{thm:v_team_convergence2}})
We let $\mathcal{F}_t^{v}=\sigma(v_0^+,Y_{\tau-1},\xi_\tau,\tau\leq t)$ denote a filtration, where $Y_\tau$ is a critic update and $\xi_\tau=(r_\tau^+,s_\tau,a_\tau,C_{v,\tau-1}^+)$ is a collection of random variables. We write the iteration of $\left<v_t^+\right>$ as follows
\begin{align}
\left<v_{t+1}^+\right>&=\left<(I\otimes\hat\Gamma_{v,t})v_t^++(C_{v,t}^+\otimes\Gamma_{v,t})\big(v_t^++\alpha_{v,t}(A_{v,t}^+v_t^++b_{v,t}^+)\big)\right>\\
&=\left<v_t^+\right>+\alpha_{v,t}A_{v,t}^\prime\left<v_t^+\right>+\alpha_{v,t}\left<(C_{v,t}^+\otimes\Gamma_{v,t})\big(\alpha_{v,t}^{-1}v_{\perp,t}^++A_{v,t}^+v_{\perp,t}^++b_{v,t}^+\big)\right>\\
&=\left<v_t^+\right>+\alpha_{v,t}\big[g_t\big(\left<v_t^+\right>,\xi_t\big)+\delta M_t+\beta_t\big],
\end{align}
where the functions $g_t(\cdot,\cdot)$, $\delta M_t$, and $\beta_t$ are given as
\begin{align}
g_t\big(\left<v_t^+\right>,\xi_t\big)&=\mathbb{E}\big(A_{v,t}^\prime\left<v_t^+\right>|\mathcal{F}_t^v\big)+\left<(C_{v,t}^+\otimes \Gamma_{v,t})b_{v,t}^+\right>+\left<(C_{v,t}^+\otimes\Gamma_{v,t})\alpha_{v,t}^{-1}v_{\perp,t}\right>\\
\delta M_t&=A_{v,t}^\prime\left<v_t^+\right>-\mathbb{E}\big(A_{v,t}^\prime\left<v_t^+\right>|\mathcal{F}_t^v\big)\\
\beta_t&=\left<(C_{v,t}^+\otimes\Gamma_{v,t})A_{v,t}^+v_{\perp,t}^+\right>.
\end{align}
To finalize the proof, we need to verify the conditions in Appendix~\ref{ap:unc}.
\begin{enumerate}
	
	\item We have $\Vert g_t\big(\left<x\right>,\xi_t\big)-g_t\big(\left<y\right>,\xi_t\big)\Vert=\Vert\mathbb{E}\big(A_{v,t}(\left<x\right>-\left<y\right>)|\mathcal{F}_t^v\big)\Vert\leq K_1\cdot\Vert \left<x\right>-\left<y\right>\Vert^2$ for some $K_1>0$ since $A_{v,t}^\prime$ is uniformly bounded by Assumption~\ref{as:linear_approx}. Thus, $g_t\big(\left<v_t^+\right>,\xi_t\big)$ is Lipschitz continuous in $\left<v_t^+\right>$.
	
	\item The step size sequence $\{\alpha_{v,t}\}_{t\geq 0}$ satisfies $\sum_t\alpha_{v,t}=\infty$ and $\sum_t\alpha_{v,t}^2<\infty$, for $t\geq 0$.
	
	\item The martingale difference sequence $\delta M_t$ satisfies $\mathbb{E}\big(\Vert\delta M_t\Vert^2|\mathcal{F}_t^v\big)\leq K_2\cdot(1+\Vert \left<v_t^+\right>\Vert^2)$, since $A_{v,t}^\prime$ is uniformly bounded by Assumption~\ref{as:linear_approx}.
	
	\item By Lemma~\ref{lem:v_consensus2} and Assumption~\ref{as:linear_approx}, the bias term $\beta_t$ is uniformly bounded and $\beta_t\rightarrow 0$ with probability one.
	
	\item We let $D(\left<v\right>)$ denote a set of all occupation measures of the Markov chain $\{\xi_t\}_{t\geq 0}$ for a fixed $\left<v\right>$. The Markov chain $\{\xi_t\}_{t\geq 0}$ is uniformly bounded since $r_t$, $s_t$, $a_t$, and $C_{v,t}^+$ are uniformly bounded.
	
\end{enumerate}
By Theorem~\ref{thm:DI_unc2}, the asymptotic behavior is described by the differential inclusion $\left<\dot{v}^+\right>\in\mathbb{E}_{\pi,d_\pi}\big[g_t(\left<v^+\right>,\xi_t)\big]=\Phi^TD_\pi^s(\gamma P_\pi-I)\Phi\left<v^+\right>+\frac{1}{N^+}\sum_{i\in\mathcal{N^+}}\Phi^TD_\pi^s R_\pi^i+\Delta_v$, where
\begin{align*}
\Vert\Delta_v\Vert&\leq\lim_{t,m}\sup_{\xi_t}\bigg\Vert\frac{1}{m}\sum_{k=t}^{t+m-1}\big[(\mathbf{1}^TC_{v,k}^+r_{k+1}^+\otimes \phi_k)\nonumber\\
&\quad+(\mathbf{1}^TC_{v,k}^+\otimes\frac{\phi_k\phi_k^T}{\Vert \phi_k\Vert^2})\alpha_{v,k}^{-1}v_{\perp,k}^+\big]- \frac{1}{N^+}\sum_{i\in\mathcal{N^+}}\Phi^TD_\pi^{s} R_\pi^i\bigg\Vert.
\end{align*}
Since the terms $b_{v,t}^+$ and $\alpha_{v,t}^{-1}v_{\perp,t}^+$, and consequently $\Delta_v$, are uniformly bounded by Assumption~\ref{as:linear_approx} and \ref{as:reward_bound} and Lemma~\ref{lem:v_consensus2}, we can apply Theorem~\ref{thm:ODE_unc2} to establish boundedness of the critic updates. Finally, we apply Theorem~\ref{thm:DI_unc3} to establish that the team-average critic value $\left< v^+\right>$ converges with probability one to a bounded neighborhood around the cooperative-team-average true minimizer $v_\pi^+$ that satisfies $\Phi^TD_\pi^s(\gamma P_\pi-I)\Phi v_\pi^++\frac{1}{N^+}\sum_{i\in\mathcal{N^+}}\Phi^TD_\pi^sR_\pi^i=0$.
\end{proof}

\begin{proof}(\textbf{Lemma~\ref{thm:lam_team_convergence2}})
The proof is nearly identical to the proof of Theorem~\ref{thm:v_team_convergence2}. We let $\mathcal{F}_t^{\lambda}=\sigma(\lambda_0,\xi_\tau,\tau\leq t)$ denote a filtration, where $\xi_\tau=(r_\tau,s_\tau,a_\tau,C_{\lambda,\tau-1}^+)$ is a collection of signals. We write the updates in the form
\begin{align}
\left<\lambda_{t+1}^+\right>&=\left<(I\otimes\hat\Gamma_{\lambda,t})\lambda_t^++(C_{\lambda,t}^+\otimes\Gamma_{\lambda,t})\big(\lambda_t^++\alpha_{\lambda,t}(A_{\lambda,t}^+\lambda_t^++b_{\lambda,t}^+)\big)\right>\\
&=\left<\lambda_t^+\right>+\alpha_{\lambda,t}\big[g_t\big(\left<\lambda_t^+\right>,\xi_t\big)+\delta M_t+\beta_t\big],
\end{align}
where the functions $g_t(\cdot,\cdot)$, $\delta M_t$, and $\beta_t$ are given as
\begin{align}
g_t\big(\left<\lambda_t^+\right>,\xi_t\big)&=\mathbb{E}\big(A_{\lambda,t}^\prime\left<\lambda_t^+\right>|\mathcal{F}_t^\lambda\big)+\left<(C_{\lambda,t}^+\otimes \Gamma_{\lambda,t})b_{\lambda,t}^+\right>+\left<(C_{\lambda,t}^+\otimes\Gamma_{\lambda,t})\alpha_{\lambda,t}^{-1}\lambda_{\perp,t}\right>\\
\delta M_t&=A_{\lambda,t}^\prime\left<\lambda_t^+\right>-\mathbb{E}\big(A_{\lambda,t}^\prime\left<\lambda_t^+\right>|\mathcal{F}_t^\lambda\big)\\
\beta_t&=\left<(C_{\lambda,t}^+\otimes\Gamma_{\lambda,t})A_{v,t}^+\lambda_{\perp,t}^+\right>.
\end{align}
The conditions can be verified as in the proof of Theorem~\ref{thm:v_team_convergence2}, which leads to the convergence of $\left<\lambda_t^+\right>$ to a limit set of the differential inclusion $\left<\dot{\lambda}^+\right>\in-F^TD_\pi^{s,a}F\left<\lambda^+\right>+\frac{1}{N^+}\sum_{i\in\mathcal{N^+}}F^TD_\pi^{s,a}R^i+\Delta_\lambda$, where 
\begin{align*}
\Vert\Delta_\lambda\Vert&=\lim_{t,m\rightarrow\infty}\sup_{\xi_t}\bigg\Vert\frac{1}{m}\sum_{k=t}^{t+m-1}\frac{1}{N^+}\bigg((\mathbf{1}^TC_{\lambda,k}^+r_{k+1}^+\otimes f_k)\nonumber\\
&\qquad+(\mathbf{1}^TC_{\lambda,k}^+\otimes\frac{f_kf_k^T}{\Vert f_k\Vert^2})\alpha_{\lambda,t}^{-1}\lambda_{\perp,k}^+\bigg) - \frac{1}{N^+}\sum_{i\in\mathcal{N^+}}F^TD_\pi^{s,a}R^i\bigg\Vert.
\end{align*}
Hence, the team-average value of the team-average reward function parameter, $\left<\lambda^+\right>$, converges with probability one to a bounded neighborhood around the desired minimizer $\lambda_\pi^+$ that satisfies $F^TD_\pi^{s,a}\big(\frac{1}{N^+}\sum_{i\in\mathcal{N}^+}R^i-F\lambda_\pi^+\big)=0$.
\end{proof}

\begin{proof}(\textbf{Theorem~\ref{thm:actor2}})
We define a filtration $\mathcal{F}_t^\theta=\sigma(\theta_\tau^i,\tau\leq t)$. The actor updates of agent~$i$, $i\in\mathcal{N}^+$, are given as
\begin{align}
\theta_{t+1}^i&=\Psi_{\Theta^i}\big(\theta_t^i+\alpha_{\theta,t}\delta_t^i\psi_t^i\big)\\
&=\Psi_{\Theta^i}\big(\theta_t^i+\alpha_{\theta,t}\big[g_t(\theta_t^i)+\delta M_t\big]\big),
\end{align}
where $\delta_t^i=f_t^T\lambda_t^i+\gamma\phi_{t+1}^Tv_t^i-\phi_t^Tv_t^i$, $\psi_t^i=\nabla_{\theta^i}\log\pi^i(a_t^i|s_t^i;\theta_t^i)$, and the function $g_t(\cdot)$ and martingale difference $\delta M_t$ are given as
\begin{align}
g_t(\theta_t^i)&=\mathbb{E}_{\pi_t,d_{\pi_t},p}\big(\delta_{t,\pi_t}\psi_t^i|\mathcal{F}_t^\theta\big)+\mathbb{E}_{\pi_t,d_{\pi_t},p}\big((\delta_{t}^i-\delta_{t,\pi_t})\psi_t^i|\mathcal{F}_t^\theta\big)\\
\delta M_t&=\delta_t^i\psi_t^i-\mathbb{E}_{\pi_t,d_{\pi_t},p}\big(\delta_t^i\psi_t^i|\mathcal{F}_t^\theta\big).
\end{align}
The signal $\delta_{t,\pi_t}$ is the approximated network TD error under the current network policy $\pi(a|s;\theta_t)$ evaluated at $v_{\pi_t}^+$ and $\lambda_{\pi_t}^+$, i.e., $\delta_{t,\pi_t}=f_t^T\lambda_{\pi_t}^++\gamma\phi_{t+1}v_{\pi_t}^+-\phi_t^Tv_{\pi_t}^+$. To complete the convergence proof, we verify conditions given in Section~\ref{ap:con}.
\begin{enumerate}

	\item The function $\delta_t^i$ is bounded by Assumption~\ref{as:linear_approx} and Theorem~\ref{thm:v_team_convergence2} and \ref{thm:lam_team_convergence2}. The function $\psi_t^i$ is bounded by Assumption~\ref{as:policy_updates}. Therefore, we obtain $\sup_t\mathbb{E}\big(\Vert\delta_t^i\psi_t^i\Vert\big|\mathcal{F}_t^\theta\big)<\infty$.
	
	\item The step size sequence $\alpha_{\theta,t}$ satisfies $\sum_t\alpha_{\theta,t}^2<\infty$ and $\lim_{t\rightarrow\infty}\frac{\alpha_{\theta,t+1}}{\alpha_{\theta,t}}=1$.

	\item The bias term satisfies $\beta_t=0$ with probability one.
	
	\item The admissible set $\Theta$ is a hyperrectangle by Assumption~\ref{as:policy_updates}.	
	
	\item The function $g_t^i(\theta^i)$ is continuous in $\theta^i$ uniformly in $t$, which further implies that a set-valued function $G(\theta^i)=\big\{\lim_{n,m\rightarrow\infty}\frac{1}{m}\sum_{t=n}^{n+m-1}g_t^i(\theta^i)\big\}$ is upper semicontinuous.

	\end{enumerate}
Applying Theorem~\ref{thm:DI_con}, the asymptotic behavior of the actor updates is given by the differential inclusion $$\dot\theta^i=\Psi_{\Theta^i}\big[G^i(\theta^i)\big].$$ 
We let $\tilde{J}^+(\theta^+,\pi^-)=\mathbb{E}_{\pi,d_{\pi},p}\big[\bar{r}(s,a;\lambda_\pi^+)+\gamma V(s_{t+1};v_\pi^+)\big]$ denote the approximated team-average objective function. It can be easily verified that $v_\pi^+$ and $\lambda_\pi^+$ are continuously differentiable in $\theta^i$. Furthermore, Assumption~\ref{as:policy_coop} ensures differentiability of $\nabla_{\theta^i}\log\pi^i(a_t^i|s_t;\theta^i)$. Therefore, $\tilde{J}^+(\theta^+,\pi^-)$ is continuously differentiable in $\theta^i$ and the associated local AC policy gradient is given as
\begin{align*}
\nabla_{\theta^i}\tilde{J}^+(\theta^+,\pi^-)=\mathbb{E}_{\pi,d_{\pi},p}\big[\big(\bar{r}(s,a;\lambda_\pi^+)+\gamma V(s^\prime;v_\pi^+)-V(s;v_\pi^+)\big)\nabla_{\theta^i}\log\pi^i(a^i|s;\theta^i)\big].
\end{align*}
We note that $G(\theta^i)=\nabla_{\theta^i}\tilde{J}^+(\theta^+,\pi^-)+\varepsilon_t^i$, where $\varepsilon_t^i$ is a set-valued error due to the discrepancy $\mathbb{E}_{\pi,d_{\pi},p}\big((\delta_{t}^i-\delta_{t,\pi})\psi_t^i|\mathcal{F}_t^\theta\big)$. Using Assumption~\ref{as:policy_adv}, the rate of change of $\tilde{J}^+(\theta^+,\pi^-)$ is given in terms of the cooperative agents as follows $$\dot{\tilde{J}}^+(\theta^+,\pi^-)=\sum_{i\in\mathcal{N}^+}\big[\nabla_{\theta^i} \tilde{J}^+(\theta^+,\pi^-)^T\big(\nabla_{\theta^i}\tilde{J}^+(\theta^+,\pi^-)+\varepsilon_t^i+z_t^i\big)\big].$$ Here, $z_t^i$ is the reflection term that projects the actor parameters back into the admissible set $\Theta^i$, i.e., $z_t^i=-\nabla_{\theta^i} \tilde{J}(\theta^+,\pi^-)-\varepsilon_t^i$ whenever a constraint is active and $z_t^i=0$ otherwise (elementwise). Suppose that $\sum_{i\in\mathcal{N}^+}\Vert z_t^i+\varepsilon_t^i\Vert^2\leq\sum_{i\in\mathcal{N}^+}\Vert\nabla_{\theta^i}\tilde{J}(\theta^+,\pi^-)\Vert^2$ on a compact subset. By Cauchy-Schwartz inequality, we then obtain $\dot{\tilde{J}}^+(\theta)\geq 0$. Therefore, the policies converge to a neighborhood of a local maximizer of the cooperative team-average objective function provided that the errors $\varepsilon^i$ are small.
\end{proof}

\section{Stochastic Approximation}
In this section, we state important theoretical results on stochastic approximation. These results closely follow those in \citep{zhang2018} and are adopted from \citep{borkar2009book} and \citep{kushner2003book}. The latter provides the most general statements about the convergence of iterates.  The theoretical results presented here include convergence theorems for unconstrained and constrained stochastic approximation.

\subsection{Unconstrained Stochastic Approximation with Correlated Noise}\label{ap:unc}

We let $\theta_n$, $Y_n$ and $\xi_n$ denote the estimated parameter, observation, and state of a Markov chain, respectively. We define the filtration $\mathcal{F}_n=\sigma(\theta_0,Y_{i-1},\xi_i,i\leq n)$. The unconstrained stochastic updates are given as follows
\begin{align}
\theta_{n+1}&=\theta_n+\epsilon_n\big[\mathbb{E}(Y_n|\mathcal{F}_n)+\delta M_n\big]\\
&=\theta_n+\epsilon_n\big[g_n(\theta_n,\xi_n)+\delta M_n+\beta_n\big]\label{alg_unc}
\end{align}
where $\epsilon_n>0$ and  $\delta M_n=Y_n-E(Y_n|\mathcal{F}_n)$ is a martingale difference. We introduce assumptions for the updates above in the following lines.

\begin{enumerate}
	
	\item The function $g_n(\theta_n,\xi_n)$ is Lipschitz continuous in the first argument.
		
	\item The step size sequence $\{\epsilon_n\}_{n\geq 0}$ satisfies $\sum_n\epsilon_n=\infty$ and $\sum_n\epsilon_n^2<\infty$, for $n\geq 0$.
	
	\item The martingale difference sequence $\{\delta M_n\}_{n\geq 0}$ satisfies $\mathbb{E}(\Vert \delta M_{n+1}\Vert^2 | \mathcal{F}_n)\leq K\cdot(1+\Vert\theta_n\Vert^2)$ for all $n\geq0$ and some $K>0$.
	
	\item The random sequence $\{\beta_n\}_{n\geq 0}$ is bounded and satisfies $\beta_n\rightarrow 0$ with probability one.

	\item
	\begin{enumerate}
	\item $\{\xi_n\}_{n\geq 0}$ is an irreducible Markov chain with stationary distribution $\eta$.
	\item The Markov chain $\{\xi_n\}_{n\geq 0}$ is uniformly bounded and has a set of occupation measures $\mathcal{D}(\theta)$ for any $\theta$.
	\end{enumerate}
	
\end{enumerate}

\begin{theorem}\label{thm:ODE_unc1}
Under assumptions 1-4 and 5(a), the asymptotic behavior of the algorithm \eqref{alg_unc} is described by the ODE
\begin{align}
\dot\theta=\bar{g}(\theta):=\mathbb{E}_{i\in\eta}\big[g(\theta,i)\big].\label{ODE}
\end{align}
\end{theorem}
\begin{theorem}\label{thm:ODE_unc2}
Suppose that $\lim_{c\rightarrow\infty}\bar{g}(c\theta)\cdot c^{-1}=g_\infty(\theta)$ exists uniformly on compact sets for some $g_\infty\in C(\mathbb{R}^n)$. If the ODE $\dot\theta=g_\infty(\theta)$ has the origin as the unique globally asymptotically stable equilibrium, then $\sup_n\Vert\theta_n\Vert<\infty$ with probability one.
\end{theorem}
\begin{theorem}\label{thm:ODE_unc3}
If the ODE \eqref{ODE} has a unique globally asymptotically stable equilibrium $\theta^*$ and $\sup_n\Vert\theta_n\Vert<\infty$ with probability one, then $\theta_n\rightarrow\theta^*$ as $n\rightarrow\infty$ with probability one.
\end{theorem}
\begin{theorem}\label{thm:DI_unc1}
Under assumptions 1-4 and 5(b), the asymptotic behavior of the algorithm \eqref{alg_unc} is described by the differential inclusion
\begin{align}
\dot\theta\in G(\theta):=\bigg\{\lim_{n,m\rightarrow\infty}\frac{1}{m}\sum_{i=n}^{n+m-1}g_i(\theta,j),\,j\in\mathcal{D}\bigg\}\label{DI}.
\end{align}
\end{theorem}
\begin{theorem}\label{thm:DI_unc2}
Under assumptions 1-4 and 5(b), suppose that $\lim_{c\rightarrow\infty}G(c\theta)\cdot c^{-1}=g_\infty(\theta)$ exists uniformly on compact sets for all $i\in\mathcal{D}$ and some $g_\infty\in C(\mathbb{R}^n)$. If the ODE $\dot\theta=g_\infty(\theta,i)$ has the origin as the unique globally asymptotically stable equilibrium, then $\sup_n\Vert\theta_n\Vert<\infty$ with probability one.
\end{theorem}
\begin{theorem}\label{thm:DI_unc3}
If  $\sup_n\Vert\theta_n\Vert<\infty$, then the trajectories converge to the limit set of the differential inclusion $\dot\theta\in G(\theta)$.
\end{theorem}

\subsection{Constrained Stochastic Approximation with Martingale Difference Noise}\label{ap:con}

We let $\theta_n$ and $Y_n$ denote the estimated parameter and observation, respectively. We define the filtration $\mathcal{F}_n=\sigma(\theta_0,Y_{i-1},i\leq n)$.
The unconstrained stochastic updates are given as follows
\begin{align}
\theta_{n+1}&=\Psi_\Theta\big(\theta_n+\epsilon_n\big[\mathbb{E}(Y_n|\mathcal{F}_n)+\delta M_n\big]\big)\\
&=\Psi_\Theta\big(\theta_n+\epsilon_n\big[g_n(\theta_n)+\delta M_n+\beta_n\big]\big),\label{alg_con}
\end{align}
where $\Psi_\Theta(\cdot)$ is a projection operator that maps the stochastic updates into a compact admissible set $\Theta$, and $\delta M_n=Y_n-E(Y_n|\mathcal{F}_n)$ is a martingale difference. We introduce assumptions for the algorithm updates.
\begin{enumerate}
	\item $\sup_n\mathbb{E}(\Vert Y_n\Vert|\mathcal{F}_n)<\infty$.
	
	\item The step size sequence $\epsilon_n$ satisfies $\sum_n\epsilon^2_n<\infty$ and $\lim_{n\rightarrow\infty}\frac{\epsilon_{n+1}}{\epsilon_{n}}=1$.
	
	\item The random sequence $\{\beta_n\}_{n\geq 0}$ satisfies $\beta_n\rightarrow 0$ with probability one.
	
	\item The admissible set $\Theta$ is a hyperrectangle, i.e., there exist $a$ and $b$ such that $a<b$ and $\Theta = \{\theta_n:a\leq \theta_n\leq b\}$.
	
	\item
	\begin{enumerate}
	\item The function $g_n(\cdot)$ is continuous uniformly in $n$. Furthermore, there exists a function $\bar{g}(\theta)$ such that for all $m>0$, we have $\lim_{n\rightarrow\infty}\big\Vert\sum_{i=n}^{n+m-1}\epsilon_i\big[g_i(\theta)-\bar{g}(\theta)\big]\big\Vert=0$.
	\item The function $g_n(\cdot)$ is continuous uniformly in $n$. Moreover, there exists an upper semicontinuous set-valued function $G(\theta)$ such that $\lim_{n,m\rightarrow\infty}\frac{1}{m}\sum_{i=n}^{n+m-1}g_i(\theta)\in G(\theta)$.
	\end{enumerate}

\end{enumerate}
\begin{theorem}\label{thm:ODE_con1}
Under Assumption 1-4 and 5(a), the asymptotic behavior of the algorithm \eqref{alg_con} is described by the ODE
\begin{align}
\dot\theta=\Psi_\Theta\big[\bar{g}(\theta)\big].\label{ODE2}
\end{align}
\end{theorem}

\begin{theorem}\label{thm:ODE_con2}
Under Assumption 1-5, if there exists a continuously differentiable function $\omega(\theta)$ such that $\bar{g}(\theta)=\frac{d\omega}{d\theta}$ and $\omega(\theta)$ is constant on disjoint compact sets $\mathcal{L}_i$, $i=1,\dots,M$, then the parameters $\theta_n$ converge with probability one to $\mathcal{L}_i$ for some $i\in\{1,\dots,M\}$ as $n\rightarrow\infty$.
\end{theorem}

\begin{theorem}\label{thm:DI_con}
Under Assumption 1-4 and 5(b), the limit points are contained in an invariant set of the differential inclusion
\begin{align}
\dot\theta\in\Psi_\Theta\big[G(\theta)\big].
\end{align}

\end{theorem}

\bibliography{references}

\end{document}